\documentclass{winnower}
\usepackage[utf8]{inputenc} 
\usepackage[T1]{fontenc}    
\usepackage{booktabs}       
\usepackage{amsfonts}       
\usepackage{nicefrac}       
\usepackage{microtype}      
\usepackage[parfill]{parskip}
\usepackage{times}
\usepackage{algorithm}
\usepackage{algorithmic}
\usepackage{float}
\usepackage{graphicx}
\usepackage{subfigure}
\usepackage{caption}
\usepackage{epstopdf}
\usepackage{color}
\usepackage{amsmath}
\usepackage{amssymb}
\usepackage{amsthm}
\usepackage[makeroom]{cancel}
\usepackage{mathtools}
\usepackage[justification=centering]{caption}
\usepackage[inline]{enumitem}
\usepackage{wrapfig}

\newtheorem{theorem}{Theorem}
\newtheorem{lemma}[theorem]{Lemma}
\newtheorem{assumption}{Assumption}
\DeclarePairedDelimiter{\ceil}{\lceil}{\rceil}
\graphicspath {{figs/}}
\begin{document}

\title{Distributed Multi-Task Relationship Learning}

\author{Sulin Liu$^\dag$, Sinno Jialin Pan$^\dag$ and Qirong Ho$^\ddag$ \\
$^\dag$Nanyang Technological University, Singapore, $^\ddag$Petuum, Inc. \\
$^\dag$\{liusl,sinnopan\}@ntu.edu.sg, $^\ddag$hoqirong@gmail.com
}
\date{}
\maketitle

\begin{abstract}
Multi-task learning aims to learn multiple tasks jointly by exploiting their relatedness to improve the generalization performance for each task. Traditionally, to perform multi-task learning, one needs to centralize data from all the tasks to a single machine. However, in many real-world applications, data of different tasks may be geo-distributed over different local machines. Due to heavy communication caused by transmitting the data and the issue of data privacy and security, it is impossible to send data of different task to a master machine to perform multi-task learning. Therefore, in this paper, we propose a distributed multi-task learning framework that simultaneously learns predictive models for each task as well as task relationships between tasks alternatingly in the parameter server paradigm. In our framework, we first offer a general dual form for a family of regularized multi-task relationship learning methods. Subsequently, we propose a communication-efficient primal-dual distributed optimization algorithm to solve the dual problem by carefully designing local subproblems to make the dual problem decomposable. Moreover, we provide a theoretical convergence analysis for the proposed algorithm, which is specific for distributed multi-task relationship learning. We conduct extensive experiments on both synthetic and real-world datasets to evaluate our proposed framework in terms of effectiveness and convergence.
\end{abstract}

\section{Introduction}
In the era of big data, developing distributed machine learning algorithms for big data analytics has become increasingly important yet challenging. Most of the recent developments of distributed machine learning optimization techniques focus on designing algorithms on learning a single predictive model under a setting where data of a certain task is distributed over different worker machines. Besides this setting, there is another natural setting of distributed machine learning where data of different sources (e.g. users, organizations) is geo-stored in local machines, and the goal is to learn a specific predictive model for each source.

Under this setting, a traditional approach is to regard learning on each source's data as an independent task and solve it locally for each source. This approach fails to fully exploit the commonality or relatedness among all the available data to learn a more precisely predictive model for each source. Another approach is through multi-task learning (MTL), where multiple tasks are learned jointly with the help of related tasks~\citep{caruana1997multitask,PanY09TKDE}. The aim is to explore the shared relevant information between the tasks to achieve better generalization performance than learning the tasks independently. Out of data privacy or security issue and communication cost of transmitting the data, it is not feasible to centralize the data of different tasks to perform MTL. And even if the data can be centralized, the size of the total data could easily exceed the physical memory of the machine. However, most existing MTL methods that have been developed could not be implemented directly in a distributed manner. Although there have been developments in data-parallel distributed algorithms for single task learning~\citep{yang2013nips,ShamirS014,MaSJJRT15}, distributed MTL under the aforementioned setting remains challenging as these algorithms do not suit the MTL formulation as MTL requires joint optimization of parameters of different tasks.

To address the problem mentioned above, we propose a distributed multi-task relationship learning algorithmic framework, denoted by DMTRL, which allows multi-task learning to be done in a distributed manner when tasks are geo-distributed over different places and data is stored locally over different machines. In general, existing MTL methods can be categorized into two main categories: learning with feature covariance~\citep{ando2005framework,argyriou2008convex,obozinski2010joint} and learning with task relations~\citep{evgeniou2004regularized,evgeniou2005kernelMTL,jacob2009clustered,zhang2010convex}. Different from prior solutions to distributed MTL, which are focused on the former category~\citep{wang2015distributedmtl,wang2016distributed,baytas16async}, our proposed DMTRL falls into the latter category. In our proposed framework, a communication-efficient primal-dual distributed optimization technique is utilized to simultaneously learn multiple tasks as well as the task relatedness in the parameter server paradigm, with a theoretical convergence guarantee.

Specifically, to make learning multiple tasks with unknown task relationships in a distributed computing environment possible, we first derive a general dual form for a family of regularized multi-task relationship learning methods. With the general dual form, we design our distributed learning algorithm by leveraging the primal-dual structure of the optimization under the parameter server paradigm. In each round of the distributed learning procedure, each local worker solves a local subproblem approximately over the data of each local task, and sends updates back to the server. Then the server aggregates the updates and calculates updated task weight vectors, which are sent back to corresponding workers, by updating and exploiting the task relatedness. Moreover, we provide theoretical analysis on the convergence rate of the proposed framework and analyze how task relationships affect the convergence rate.

The major contributions of our work are three folds:
\begin{itemize}
	\item Our proposed framework DMTRL is general for a family of regularized MTL methods, which simultaneously learn task relationships and task-specific predictive models from geo-distributed task data. Furthermore, DMTRL is communication-efficient. As a by-product, DMTRL provides a scalable solution to MTL in large scale when the total data is of massive due to either large number of tasks or large amount of data per task.\footnote{Though MTL is originally proposed for the problem where each task only has a small size of labeled training data, it has been shown by other researchers that when some tasks have relatively large amount of data, MTL can still help improve generalization across tasks by jointly exploiting information from related tasks~\citep{AhmedDS14}.}
	
	\item We provide theoretical analysis on primal-dual convergence rate for the proposed distributed MTL optimization for both smooth and non-smooth convex losses. Different from previous distributed optimization convergence analysis for a single task, ours is specific for distributed MTL which takes task relationships into consideration.
	
	\item We implement the framework on a distributed machine learning platform Petuum~\citep{xing2015petuum}, and conduct extensive experiments on both synthetic and real-world datasets to demonstrate its effectiveness in terms of prediction accuracy and convergence. Note that our framework can be fitted to any distributed machine learning platform under the parameter server paradigm, such as~\citep{LiAPSAJLSS14}.
\end{itemize}

\section{Related Work}

\noindent{\bf Distributed machine learning} has attracted more and more interests recently~\citep{BalcanBFM12}. There have been tremendous efforts done on different machine learning problems~\citep{newman2009distributed,forero2010consensus,BalcanBFM12}. At the same time, developing distributed optimization methods for large-scale machine learning has been receiving much research interest~\citep{boyd2011distributed,richtarik2013distributed,yang2013nips,ShamirS014,MaSJJRT15}. These methods allow for local optimization procedure to be taken at each communication round. However, their algorithms focus on single-task learning problems, while our work aims at developing a distributed optimization algorithm for MTL problems where single-task learning algorithms cannot be directly applied.

\noindent{\bf Online Multi-task Learning} assumes instances from different tasks arrive in a sequence and adversarially chooses task to learn. \citet{cavallanti2010online} exploited online MTL with a given task relationship encoded in a matrix, which is known beforehand.~\citet{saha2011online} exploited online learning of task weight vectors and relationship together. They formulated the problem of online learning the task relationship matrix as a Bregman divergence minimization problem. After the task relationship matrix is learned, it is exploited to help actively select informative instances for online learning.

\noindent{\bf Parallel Multi-task Learning} aims to develop parallel computing algorithms for MTL in a shared-memory computing environment. Recently,~\citet{zhangyuPMTL} proposed a parallel MTL algorithm named PMTL. In PMTL, dual forms of three losses are presented and accelerated proximal gradient (APG) method is applied to make the problem decomposable, and thus possible to be solved in parallel. By comparison, firstly, we induce a more general dual form, where any convex loss function can be applied. In addition, our algorithm can solve the same type of problem as PMTL under the distributed machine learning setting, while PMTL cannot be applied directly when data of different tasks are stored on different local machines.

\noindent{\bf Distributed Multi-task Learning} is an area that has not been much exploited.~\citet{wang2015distributedmtl} proposed a distributed algorithm for MTL by assuming that different tasks are related through shared sparsity. In another work~\citep{baytas16async}, asynchronous distributed MTL method is proposed for MTL with shared subspace learning or shared feature subset learning.
Different from the above mentioned approaches, our method aims at solving MTL by learning task relationships from data, which can be positive, negative, or unrelated, via a task-covariance matrix.~\citet{AhmedDS14} proposed a hierarchical MTL model motivated from the application of advertising. Their method assumes a hierarchical structure among tasks be given in advance. Proximal subgradient method is utilized such that partial subgradients can be distributively calculated. Our setting is different from theirs as we do not assume tasks lie in a hierarchical structure. Moreover, our method distributes the dual problem while theirs focuses on distributively solving the primal with the given task hierarchy. Another work that exploits distributed MTL~\citep{dinuzzo2011client} considers a client-server setting, where clients send their own data to the server and the server sends back helpful information for each client to solve the task independently. By sending data from clients to the server, it is very communication-heavy and thus not feasible under our problem setting.

\textbf{Notation.} Scalars, vectors and matrices are denoted by lowercase, boldface lowercase and boldface uppercase letters respectively. For any $k \in \mathbb{N}^+$, we define $[k] = \{1,\cdots, k\}$. For a vector $\mathbf{a} \in \mathbb{R}^n$ that is split into $m$ coordinate blocks, i.e. $\mathbf{a} = [\tilde{\mathbf{a}}_{[1]}; \cdots;\tilde{\mathbf{a}}_{[m]}]$, we define $\mathbf{a}_{[i]} \in \mathbb{R}^n \,(i \in [m])$ that takes same value of $\mathbf{a}$ if the coordinate belongs to $i$-th coordinate block and takes $0$ elsewhere.
\section{Problem Statement}
For simplicity in description, we consider a setting with $m$ tasks $\{T_i\} _{i=1}^m$ that are distributed over $m$ workers, i.e., one machine is for one task. In practice, our framework is flexible to put several tasks together in one worker or further distribute data of one task over several local workers with straightforward modification of the algorithm. Each task $T_i$ on a worker $i$ follows a distribution $\mathcal{D}_i$ and has a training set of size $\mathit{n_i}$ with $\mathbf{x}_j^i\in\mathbb{R}^{d}$ being the $j$-th data point and $\mathit{y}_j^i$ as its label. The value of label $\mathit{y}_j^i$ can be continuous for a regression problem or discrete for a classification problem. Here, we consider a general family of regularized MTL methods introduced in ~\citep{zhang2010convex}, which is a general multi-task relationship learning framework that includes many existing popular MTL methods as its special cases~\citep{evgeniou2004regularized,evgeniou2005kernelMTL,kato2008multi,jacob2009clustered}.

The formulation is defined as follows:
\begin{eqnarray}
\label{formulation}
& \underset{\mathbf{W,\Omega}}{\mbox{min}} &
\sum_{i=1}^{m} \frac{1}{n_i}\sum_{j=1}^{n_i}l_j^i(\mathbf{w}_i^T \phi(\mathbf{x}_j^i),y_j^i )+ \frac{\lambda}{2} \mbox{tr}(\mathbf{W}\boldsymbol{\Omega}\mathbf{W}^T) \\
&\mbox{s.t.} & \boldsymbol{\Omega}^{-1} \succeq 0, \mbox{ and }\; \mbox{tr}(\boldsymbol{\Omega}^{-1}) = 1, \nonumber 
\end{eqnarray}
where $l_j^i(\cdot)$ is an arbitrary convex real-valued loss function of the $i$-th task on the $j$-th data point, $\phi(\cdot)$ is a feature mapping that can be linear or nonlinear, $\mathbf{W} = (\mathbf{w}_1,\mathbf{w}_2,\cdots,\mathbf{w}_m)\in\mathbb{R}^{d\times m}$, and $\lambda > 0$ is the regularization parameter. The first term of the objective measures the empirical loss of all tasks with the term $1/n_i$ to balance different sample sizes of different tasks. The second term serves as a task-relationship regularizer with $\boldsymbol{\Omega}$ being the precision matrix (inverse of the covariance matrix as shown in~ \cite{zhang2010convex}). The covariance matrix $\boldsymbol{\Omega}^{-1}$ is flexible enough to describe positive, negative and unrelated task relationships. The regularization term on each task's weight vector is embedded in $\boldsymbol{\Omega}$ as well. The constraints serve to enforce some prior assumptions on $\boldsymbol{\Omega}^{-1}$, which can be replaced by some other convex constraints on $\boldsymbol{\Omega}^{-1}$.

According to~\cite{zhang2010convex}, \eqref{formulation} is jointly convex w.r.t. $\mathbf{W}$ and $\boldsymbol{\Omega}^{-1}$, which can be resorted to an alternating optimization procedure. Our proposed DMTRL aims at distributing the learning of multiple tasks for finding $\mathbf{W}$ with precision matrix $\boldsymbol{\Omega}$ fixed when data of different tasks are stored in local workers, and centralizing parameters $\mathbf{W}$ to a server to update $\boldsymbol{\Omega}$ in the alternating step. In the following section, we first derive a general dual form for (\ref{formulation}) with $\boldsymbol{\Omega}$ fixed, which will be used to design our distributed learning algorithm.

\section{General Dual Form with \texorpdfstring{${\boldsymbol{\Omega}}$}{Omega} Fixed and Primal-Dual Certificates}
Motivated by the recent advances in distributed optimization using stochastic dual coordinate ascent (SDCA) for single task learning~\citep{yang2013nips,MaSJJRT15}, we turn to deriving the dual form to facilitate distributed optimization for MTL and arrive at the following theorem.
\begin{theorem}\label{dual_form_theorem}
	The general dual problem of \eqref{formulation} with ${\mathbf{\Omega}}$ fixed is given by:
	\begin{equation}\label{generaldual}
	\underset{\boldsymbol{\alpha} \in \mathbb{R}^n}{\textnormal{max}}D(\boldsymbol{\alpha})=
	-\frac{1}{2\lambda}\boldsymbol{\alpha}^T \mathbf{K} \boldsymbol{\alpha} - \sum_{i=1}^{m}\frac{1}{n_i}\sum_{j=1}^{n_i}{l_j^i}^*(-\alpha_j^i),
	\end{equation}
	where ${l_j^i}^*(\cdot)$ is the conjugate function of ${l_j^i}(\cdot)$, $\boldsymbol{\alpha} =(\tilde{\boldsymbol{\alpha}}_{[1]};\cdots;\tilde{\boldsymbol{\alpha}}_{[m]})$ with $\tilde{\boldsymbol{\alpha}}_{[i]} = (\alpha_1^i,\cdots,\alpha_{n_i}^i)^T$,
	$\mathbf{K}$ is an $n \times n$ matrix, where $n\!=\!\sum_{t=1}^mn_t$, with its $(I_j^i,I_{j^\prime}^{i^\prime})$-th element being $\frac{\sigma_{i i^\prime}}{n_in_{i'}} \langle \phi(\mathbf{x}_j^i),\phi(\mathbf{x}_{j^\prime}^{i^\prime})\rangle$, $I_j^i = j+\sum_{t=1}^{i-1}n_t$ is the global index for $\mathbf{x}_j^i$ among all training data from all tasks, and $\sigma_{i i^\prime}$ is the $(i,i')$-th element of $\,\boldsymbol{\Sigma} = \boldsymbol{\Omega}^{-1}$, which represents the correlation between task $i$ and $i'$. The primal-dual optimal point correspondence is given by
	$\mathbf{w}^*_i = \frac{1}{\lambda} \sum_{i'=1}^{m}\sum_{j'=1}^{n_{i'}}\frac{{\alpha_{j'}^{i'}}^*}{n_{i'}}{\phi(\mathbf{x}_{j'}^{i'})} \sigma_{ii'}$.
\end{theorem}
Here, $\mathbf{K}$ could be regarded as a multi-task similarity matrix with each element scaled by inter-task covariance and the number of instances per task. If $\phi(\cdot)$ maps an instance to a Hilbert space, then $\mathbf{K}$ is a kernel matrix. However, in this way, we have to compute the kernel matrix of $n \!\times\! n$ size using all instances from all tasks, which is infeasible in our distributed setting. Therefore, we propose to use explicit feature mapping function $\phi(\cdot)$ instead. For example, we could approximate infinite kernel expansions by using randomly drawn features in an unbiased manner~\citep{rahimi2007random}. Note that in PMTL~\citep{zhangyuPMTL}, dual forms of (\ref{formulation}) for special cases such as hinge loss, $\epsilon$-sensitive loss and squared loss are derived for a similar problem. The difference lies in that our theorem is more general, which applies to all kinds of convex losses.

Following the primal-dual optimal point correspondence of Theorem \ref{dual_form_theorem}, it is natural to define a feasible $\mathbf{W}(\boldsymbol{\alpha})$ that corresponds to $\boldsymbol{\alpha}$ as follows,
\begin{equation} \label{primal_dual_feasible}
{{\mathbf{w}}}_i(\boldsymbol{\alpha}) = \frac{1}{\lambda} \sum_{i'=1}^{m}\sum_{j'=1}^{n_{i'}}\frac{\alpha_{j'}^{i'}}{n_{i'}}{\phi(\mathbf{x}_{j'}^{i'})} \sigma_{ii'}.
\end{equation}
By defining the objective in \eqref{formulation} with $\mathbf{\Omega}$ fixed as $P(\mathbf{W})$, we have the duality gap function defined as
$G(\boldsymbol{\alpha}) = P(\mathbf{W}(\boldsymbol{\alpha})) - D(\boldsymbol{\alpha})$.
From weak duality, $P(\mathbf{W}(\boldsymbol{\alpha}))$ is always greater or equal to $D(\boldsymbol{\alpha})$. Therefore, duality gap $G(\boldsymbol{\alpha})$ could provide a certificate on the approximation to the optimum.

With the derived dual problem, we could carefully design local dual subproblems that allow for distributed primal-dual optimization, which will be presented in details in Section~\ref{sec:DMTRL_algo}.
The primal-dual optimization method we introduce later has several advantages over the gradient-based primal methods: 1) it does not need to determine any step-size, and 2) the duality gap provides a measure of approximation quality during training. Next, we introduce two common classes of functions.

\textbf{Definition 1} (L-Lipschitz continuous function) A function $l$: $\mathbb{R} \!\rightarrow\! \mathbb{R}$ is L-Lipschitz continuous if $\forall a,b \in \mathbb{R}$, we have
\begin{equation*}
|l(a) - l(b)| \leq L|a-b|.
\end{equation*}

\textbf{Definition 2} ($(1/\mu)$-smooth function) A function $l$: $\mathbb{R} \!\rightarrow\! \mathbb{R}$ is $(1/\mu)$-smooth if it is differentiable and its derivative is $(1/\mu)$-Lipschitz, where $\mu > 0$. Or equivalently, $\forall a,b \in \mathbb{R}$, we have
\begin{equation*}
l(a) \leq l(b) + l'(b)(a-b) + \frac{1}{2\mu}(a-b)^2.
\end{equation*}
Note that most commonly used loss functions fall into the above two classes. For instance, hinge loss falls under the first category, squared loss falls under the second category, and logistic loss falls under both categories. In Section~\ref{sec:convergence}, we provide convergence analysis when the loss function falls into either of the above two categories. Based on the definition of smooth function above, we have the following well-known lemma:
\begin{lemma}
	Function $l(\cdot)$ is $(1/\mu)$-smooth if and only if its conjugate function $l^*(\cdot)$ is $\mu$ strongly convex.
\end{lemma}

\section{The Proposed Methodology}\label{sec:DMTRL_algo}
\subsection{The Overall Framework}\label{sec:overall_framework}
Our proposed overall algorithm presented in Algorithm~\ref{DDCAMTL} is mainly based on an alternating optimization procedure that comprises two steps: solving $\mathbf{W}$ with a fixed $\boldsymbol{\Omega}$ in a distributed manner between the server and workers ($\mathbf{W}$-step: Steps 4-10), and solving $\boldsymbol{\Omega}$ with aggregated $\mathbf{W}$ from all workers on the server ($\boldsymbol{\Omega}$-step: Step 11).

Specifically, during the $\mathbf{W}$-step, distributed optimization is conducted on the dual problem~\eqref{generaldual} iteratively. Each worker is assigned a local subproblem that only requires accessing local data.
Note that there are two types of updates in $\mathbf{W}$-step: global update and local update. In local update, every worker $i$ solves the local dual subproblem through a Local SDCA algorithm approximately over the local dual coordinate block $\boldsymbol{\alpha}_{[i]}$. Moreover, by defining $\mathbf{b}_i \!=\! \frac{1}{n_{i}}\sum_{j'=1}^{n_{i}} {\alpha_{j'}^{i}} \phi(x_{j'}^{i})$, each worker computes the updates on $\mathbf{b}_i$, i.e., $\Delta\mathbf{b}_i$, locally. When the local update ends, in global update, each worker sends the corresponding $\Delta\mathbf{b}_i$  to the server. We know from \eqref{primal_dual_feasible} that ${{\mathbf{w}}}_i(\boldsymbol{\alpha}) \!=\! \frac{1}{\lambda}\sum_{i'=1}^{m}\mathbf{b}_i \sigma_{ii'}$. Therefore, the server aggregates the local updates on $\{\mathbf{b}_i\}_{i=1}^m$ from all local workers to calculate updated task weight vectors $\{\mathbf{w}_i\}_{i=1}^m$ and send them back to the corresponding local workers. This procedure repeats until desired duality gap is arrived to establish convergence. Figure \ref{fig:distributed_W_step} provides an illustration of the procedure in $\mathbf{W}$-step.
\begin{figure}[h!]
	\begin{center}
		\includegraphics[width=0.6\columnwidth]{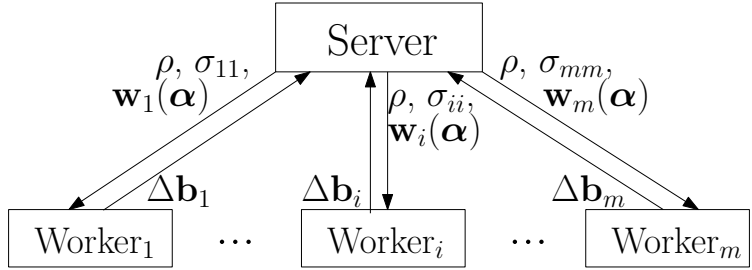}
		\caption{Distributed learning in $\mathbf{W}$-step}\label{fig:distributed_W_step}
	\end{center}
\end{figure}

After $\mathbf{W}$-step is finished, $\boldsymbol{\Omega}$-step is conducted on the server by solving problem \eqref{formulation} with fixed $\mathbf{W}$. Then the server will send each updated $\sigma_{ii}$ to the corresponding local worker $i$. The computational cost of this step is reasonable for computing centrally, since it only involves optimizing $\mbox{tr}(\mathbf{W}\boldsymbol{\Omega}\mathbf{W}^T)$ given the constraints on $\boldsymbol{\Omega}^{-1}$. The optimization involves eigen-decomposition of $\mathbf{W}^T\mathbf{W}$, which is computationally expensive when number of tasks is large. In that case, existing distributed SVD algorithms~\citep{meng2015mllib} can be leveraged to solve it more efficiently, which is beyond the focus of this work.

Note that in $\boldsymbol{\Omega}$-step, the communication cost is just to send an updated scalar $\sigma_{ii}$ to the corresponding worker $i$. Therefore, the main communication cost is caused by the global updates in $\mathbf{W}$-step (Row 9 in Algorithm~\ref{DDCAMTL}), i.e., the number of iterations $T$.
\begin{algorithm}[!t]
	\caption{DMTRL Algorithm}
	\label{DDCAMTL}
	\begin{algorithmic}[1]
		\STATE {\bfseries Input:} data $\{\mathbf{x}_j^i,y_j^i \}$ with $i\!=\!1,...m$ and $j\!=\!1,..,n_j$ distributed over $m$ machines, aggregation parameter $\frac{1}{m} \leq \eta \leq 1$, maximum number of alternating iterations $P$, and maximum number of global update iterations, $T$, in the $\mathbf{W}$-step
		\STATE {\bfseries Initialize:}  $\boldsymbol{\alpha}_{[i]}^{(0)}  \gets \mathbf{0}$ for all machines $i$ and $ {\mathbf{w}_i}(\boldsymbol{\alpha}) \gets  \mathbf{0}$, where $i$ is the task that the data in machine $i$ belong to, $\boldsymbol{\Omega} \gets m\,\mathbf{I},\boldsymbol{\Sigma} \gets \frac{1}{m}\,\mathbf{I}$
		\FOR{$p = 1$ {\textbf{to}} $P$ }
		\FOR{$t = 1$ {\textbf{to}} $T$}
		\FOR{\textbf{all machines (local update):} $i=1,2,\cdots,m$ \textbf{in parallel}}
		\STATE solving local subproblem: \\$\Delta\boldsymbol{\alpha}_{[i]} \gets \mbox{Local SDCA}\left(\boldsymbol{\alpha}_{[i]}^{(t-1)},{{\mathbf{w}}}_i(\boldsymbol{\alpha})^{(t-1)},\sigma_{ii}^{(p-1)}\right)$
		\STATE local updates: \\
		${\boldsymbol{\alpha}_{[i]}}^{(t)} \gets{\boldsymbol{\alpha}_{[i]}}^{(t-1)} + \eta {\Delta\boldsymbol{\alpha}_{[i]}}$ \\
		$\Delta\mathbf{b}_{i}^{(t)} \gets\Delta\mathbf{b}_{i}^{(t-1)} + \frac{1}{n_{i}}\sum_{j'=1}^{n_{i}} {\eta\Delta\alpha_{j'}^{i}} \phi(\mathbf{x}_{j'}^{i})$
		\ENDFOR	
		\STATE {\bfseries Reduce (global update):} server aggregates $\Delta\mathbf{b}_i$'s from all workers to compute $\mathbf{w}_i(\boldsymbol{\alpha})^{(t)} \!=\! {{\mathbf{w}}}_i(\boldsymbol{\alpha})^{(t-1)} \!+\! \frac{1}{\lambda}\sum_{i'=1}^{m}\Delta\mathbf{b}_i \sigma_{ii'}$, and sends updated $\{\mathbf{w}_i\}$'s back to the corresponding local workers.
		\ENDFOR
		\STATE $\boldsymbol{\Omega}^{(p)}$ $\gets$ Solve problem (\ref{formulation}) on server for fixed $\mathbf{W}^{(p)}$, whose $i$-th column corresponds to $\mathbf{w}_i(\boldsymbol{\alpha})$, and update $\boldsymbol{\Sigma}^{(p)} = {\boldsymbol{\Omega}^{(p)}}^{-1}$. Server sends updated $\sigma_{ii}$ to each worker $i$
		\ENDFOR
		\STATE {\bfseries Output:} $\mathbf{W},\boldsymbol{\Sigma}$
	\end{algorithmic}
\end{algorithm}

\subsection{Local SDCA}\label{sec:localSDCA}
In this section, we describe the local dual subproblem to be solved for local updates in $\mathbf{W}$-step in details.
For each worker, a local subproblem is defined and only local data is needed for solving it.
During the local update step of $\mathbf{W}$-step, each worker approximately solves the local subproblem (Rows 5-8 in Algorithm~\ref{DDCAMTL}). The local subproblem solution does not need to be near-optimal. It only needs to achieve some improvement of the local subproblem objective towards the optimum, which will be explained more clearly in Section~\ref{sec:convergence}. The subproblem for each worker is defined as,
\begin{equation}\label{localsubproblem}
\underset{\Delta\boldsymbol{\alpha}_{[i]} \in \mathbb{R}^{n_i}}{\text{max}} \mathcal{D}_i^\rho(\Delta\boldsymbol{\alpha}_{[i]};\mathbf{w}_i(\boldsymbol{\alpha}),{\boldsymbol{\alpha}}_{[i]}),
\end{equation}
where
\begin{eqnarray}
\lefteqn{\mathcal{D}_i^\rho(\Delta\boldsymbol{\alpha}_{[i]};\mathbf{w}_i(\boldsymbol{\alpha}),{\boldsymbol{\alpha}}_{[i]})} \nonumber \\
& = & - \frac{1}{n_i}\sum_{j=1}^{n_i}{l_j^i}^{*}(-\alpha_j^i - \Delta\alpha_j^i)  - \frac{1}{n_i}\sum_{j=1}^{n_i} \Delta\alpha_j^i {\mathbf{w}_i(\boldsymbol{\alpha})}^T \phi(\mathbf{x}_j^i) \nonumber - \frac{1}{2\lambda m}{\boldsymbol{\alpha}}^T \mathbf{K} {\boldsymbol{\alpha}} - \frac{\rho}{2\lambda}  \Delta\boldsymbol{\alpha}_{[i]}^T \mathbf{K} \Delta\boldsymbol{\alpha}_{[i]}. \nonumber
\end{eqnarray}
By defining the local subproblem in this way, when the local variables $\Delta\boldsymbol{\alpha}_{[i]}$ vary during the local subproblem optimization, the local objectives well approximate the global objective in (\ref{generaldual}) as shown in the following Lemma \ref{dual_lemma}.
\begin{lemma}\label{dual_lemma}
	For any dual variable $\boldsymbol{\alpha}\in \mathbb{R}^n$, change in dual variable $\Delta\boldsymbol{\alpha} \in \mathbb{R}^n$, primal variable $\mathbf{w}_i = \mathbf{w}_i(\boldsymbol{\alpha})$, aggregation parameter $\eta \in [0,1],$ and $\rho$, when
	\begin{equation}\label{rhodef}
	\rho \geq \rho_{\text{min}} = \eta \underset{\boldsymbol{\alpha} \in \mathbb{R}^n}{\text{max}} \frac{{\boldsymbol{\alpha}}^T \mathbf{K} \boldsymbol{\alpha}} {\sum_{i=1}^{m} {\boldsymbol{\alpha}}_{[i]}^T \mathbf{K} \boldsymbol{\alpha}_{[i]}},
	\end{equation}
	it holds that
	\begin{equation*}
	D(\boldsymbol{\alpha} + \eta \sum_{i=1}^{m}\Delta\boldsymbol{\alpha}_{[i]}) \geq (1-\eta) D(\boldsymbol{\alpha})
	+ \eta \sum_{i=1}^{m} \mathcal{D}_i^\rho(\Delta\boldsymbol{\alpha}_{[i]};\mathbf{w}_i(\boldsymbol{\alpha}),{\boldsymbol{\alpha}}_{[i]}).
	\end{equation*}
	
\end{lemma}

In Algorithm~\ref{DDCAMTL}, each worker implements the local dual stochastic coordinate ascent (SDCA) method on the local subproblem \eqref{localsubproblem} to reach an approximate solution. The detailed algorithm of the local SDCA method is presented in Algorithm~\ref{LSDCA}.
In each iteration, a coordinate $\alpha_j^i$ in $\boldsymbol{\alpha}_{[i]}$ is randomly selected and set to the update that maximizes the local subproblem $\mathcal{D}_i^\rho(\Delta\boldsymbol{\alpha}_{[i]};\mathbf{w}_i(\boldsymbol{\alpha}),{\boldsymbol{\alpha}}_{[i]})$ with other local coordinates fixed. $\mathbf{e}_j^i \in \mathbb{R}^{n}$ in Algorithm \ref{LSDCA} is defined as a basis vector with $\mathbf{e}_j^i(I_j^i) = 1$ and $0$ elsewhere.
\begin{algorithm}
	
	\caption{Local SDCA}
	\label{LSDCA}
	\begin{algorithmic}
		\STATE\textbf{Input:} $H \geq 1,\boldsymbol{\alpha}_{[i]},{\mathbf{w}_i}(\boldsymbol{\alpha}), \sigma_{ii}$
		\STATE\textbf{Data:} Local data $\{x_j^i,y_j^i \}_{j=1}^{n_i}$
		\STATE\textbf{Initialize:} $\boldsymbol{\Delta\alpha}_{[i]} \gets 0$
		\FOR {$h=1,2,\cdots,H$}
		\STATE choose $j \in {1,2,\cdots,n_i}$ uniformly at random
		\STATE $\quad {\delta}_j^i := \underset{{\delta}_j^i\in\mathbb{R} }{\text{argmax}} \mathcal{D}_i^\rho(\Delta\boldsymbol{\alpha}_{[i]}^{(h-1)}+\delta_j^i\mathbf{e}_j^i;\mathbf{w}_i(\boldsymbol{\alpha}),{\boldsymbol{\alpha}}_{[i]})$
		\STATE $\quad{\Delta\alpha_j^i}^{(h)} \gets {\Delta\alpha_j^i}^{(h-1)} + {\delta}_j^i$
		\ENDFOR
		\STATE \textbf{Output:} $ \Delta\boldsymbol{\alpha}_{[i]} $
		
	\end{algorithmic}
\end{algorithm}

\section{Convergence Analysis}\label{sec:convergence}
Since the optimization problem \eqref{formulation} is jointly convex with $\mathbf{W}$ and $\boldsymbol{\Omega}^{-1}$, the alternating optimization procedure is guaranteed to converge to the global optimal solution. Our analysis focuses on the distributed optimization in $\mathbf{W}$-step. Ideas of our convergence analysis come from distributed or stochastic primal-dual optimization methods for single task learning~\cite{MaSJJRT15,shalev2013stochastic}. However, our analysis is specific for multi-task learning and provides insights on how task relationships affect the convergence (section \ref{effect_task_relationship}).
Before conducting convergence analysis, we define Assumption~\ref{local_approx_assumption} that characterizes how well the local solution approximates the local optimal solution. In section~\ref{sec:local_convergence}, we analyze the local convergence of the local SDCA method in each worker, i.e. when is Assumption~\ref{local_approx_assumption} satisfied.
In section~\ref{sec:primal_dual_convergence}, we show the primal-dual convergence rate for the global update of $\mathbf{W}$-step when Assumption~\ref{local_approx_assumption} is satisfied.
\begin{assumption}\label{local_approx_assumption}
	($\Theta$-approximate solution). $\forall i \!\in\! [m]$, the local solver at any iteration $t \!\in\! [T]$ reaches an approximate update $\Delta \boldsymbol{\alpha}_{[i]}$ such that there exists a $\Theta \!\in\! [0,1)$, and the following inequality holds:
	\begin{eqnarray}
	\lefteqn{\mathbb{E} [\mathcal{D}_i^\rho(\Delta\boldsymbol{\alpha}_{[i]}^*;\mathbf{w}_i(\boldsymbol{\alpha}),{\boldsymbol{\alpha}}_{[i]}) - \mathcal{D}_i^\rho(\Delta\boldsymbol{\alpha}_{[i]};\mathbf{w}_i(\boldsymbol{\alpha}),{\boldsymbol{\alpha}}_{[i]})]} \nonumber \\
	& & \leq \Theta \bigg( \mathcal{D}_i^\rho(\Delta\boldsymbol{\alpha}_{[i]}^*;\mathbf{w}_i(\boldsymbol{\alpha}),{\boldsymbol{\alpha}}_{[i]}) - \mathcal{D}_i^\rho(\mathbf{0};\mathbf{w}_i(\boldsymbol{\alpha}),{\boldsymbol{\alpha}}_{[i]}) \bigg), \nonumber
	\end{eqnarray}
	where
	$\Delta\boldsymbol{\alpha}_{[i]}^* \in \underset{\Delta\boldsymbol{\alpha}_{[i]} \in \mathbb{R}^{n_i}}{\mathrm{argmax}} \mathcal{D}_i^\rho(\Delta\boldsymbol{\alpha}_{[i]};\mathbf{w}_i(\boldsymbol{\alpha}),{\boldsymbol{\alpha}}_{[i]})$, i.e. the optimal solution to the local subproblem.
\end{assumption}
The assumption characterizes how well the local subproblem is solved. The smaller $\Theta$ is, the better the local subproblem is  solved.

Note that in the following sections, due to the limit in space, for most theorems and Lemmas, proofs are deferred to the Appendix.

\subsection{Local Subproblem Convergence}\label{sec:local_convergence}
The following two theorems show the local subproblem convergence using SDCA as the local solver. In particular, by removing the negative sign from \eqref{localsubproblem}, the original maximization local subproblem can be written as the following minimization problem,
\begin{equation*}
\min\limits_{\Delta\boldsymbol{\alpha}_{[i]}}g(\Delta\boldsymbol{\alpha}_{[i]})+f(\Delta\boldsymbol{\alpha}_{[i]}),
\end{equation*}
where
\[g(\Delta\boldsymbol{\alpha}_{[i]}) = \frac{1}{n_i}\sum_{j=1}^{n_i}{l_j^i}^*(-\alpha_j^i - \Delta\alpha_j^i),\]
and
\[f(\Delta\boldsymbol{\alpha}_{[i]}) =  \frac{1}{2\lambda m}{\boldsymbol{\alpha}}^T \mathbf{K} {\boldsymbol{\alpha}} + \frac{1}{n_i}\sum_{j=1}^{n_i} \Delta\alpha_j^i {\mathbf{w}_i(\boldsymbol{\alpha})}^T \phi(\mathbf{x}_j^i) +
\frac{\rho}{2\lambda}  \Delta\boldsymbol{\alpha}_{[i]}^T \mathbf{K} \Delta\boldsymbol{\alpha}_{[i]},\]
whose gradient is coordinate-wise Lipschitz continuous. This type of objective function has been studied in Block Coordinate Descent~\citep{Richtarik2015,tappenden2015complexity}. We can show that the Local SDCA algorithm achieves the following convergence rate when applying it to the local subproblem of our algorithm. In the theorems, $q_{\text{max}} \!=\! \text{max}_j  \|\phi(\mathbf{x}_j^i)\|^2$.
\begin{theorem}\label{local_convergence_smooth}
	When functions $l_j^i(\cdot)$ are $(1/\mu)$-smooth for all $(i,j)$: Assumption~\ref{local_approx_assumption} holds for Local SDCA if the number of iterations $H$ satisfies
	\begin{equation*}
	H \geq \text{log}(\frac{1}{\Theta}) \frac{\rho\sigma_{ii} q_{\text{max}}+\mu\lambda n_i}{\mu\lambda }.
	\end{equation*}
\end{theorem}

\begin{theorem}\label{local_convergence_lipschitz}
	When functions $l_j^i(\cdot)$ are L-Lipschitz for all $(i,j)$: Assumption~\ref{local_approx_assumption} holds for Local SDCA if the number of iterations $H$ satisfies
	\begin{equation*}
	H \geq n_i\bigg( \frac{1-\Theta}{\Theta} + \frac{\rho \sigma_{ii}q_{\text{max}} \| \Delta\boldsymbol{\alpha}_{[i]}^* \|^2}{2\Theta \lambda {n_i}^2 \big(\mathcal{D}_i^\rho(\Delta\boldsymbol{\alpha}_{[i]}^*;.)-\mathcal{D}_i^\rho(\mathbf{0};.)\big)} \bigg).
	\end{equation*}
\end{theorem}

\subsection{Primal-Dual Convergence Analysis}\label{sec:primal_dual_convergence}
Next, we show the primal-dual convergence of the global update step when solving $\mathbf{W}$. Before introducing the main theorems, we first introduce the following lemmas that describe the relationship between increase in dual objective and the duality gap.

\begin{lemma}\label{conv_lemma}
	$\forall i,j$, if ${l_j^i}^*(\cdot)$ is $\mu$ strongly convex (i.e., ${l_j^i}(\cdot)$ is $(1/\mu)$-smooth) and Assumption~\ref{local_approx_assumption} is fulfilled, then for all iterations $t\in[T]$ within $\mathbf{W}$-step of Algorithm~\ref{DDCAMTL} and $ \forall s \in [0,1]$,
	\begin{equation} \label{convergence_lemma_ineq}
	\mathbb{E}\left(D(\boldsymbol{\alpha}^{(t+1)})-D(\boldsymbol{\alpha}^{(t)}) \right) \geq \eta(1-\Theta) \left(sG(\boldsymbol{\alpha}^{(t)})-\frac{\rho}{2\lambda}s^2 Q^{(t)} \right), \nonumber
	\end{equation}
	where
	\begin{equation*}
	\begin{aligned}
	Q^{(t)} = & - \frac{\lambda\mu(1-s)}{\rho s}\sum_{i=1}^{m}\frac{1}{n_i}\left\|\mathbf{u}_{[i]}^{(t)} - \boldsymbol{\alpha}_{[i]}^{(t)} \right\|^2 + \sum_{i=1}^{m} \left(\mathbf{u}_{[i]}^{(t)} - \boldsymbol{\alpha}_{[i]}^{(t)}\right)^T \mathbf{K} \left(\mathbf{u}_{[i]}^{(t)} - \boldsymbol{\alpha}_{[i]}^{(t)} \right),
	\end{aligned}
	\end{equation*}
	with \\
	\begin{equation*}-{u_j^i}^{(t)} \in \partial l_j^i\left({\mathbf{w}_i(\boldsymbol{\alpha}^{(t)})}^T \mathbf{x}_j^i\right),\end{equation*} where $\partial l_j^i(z)$ denotes the set of subgradients of $l_j^i(\cdot)$ at $z$.
\end{lemma}

\begin{lemma}\label{Q_upp_bound}
	$\forall i,j$, if ${l_j^i}(\cdot)$ is L-Lipschitz continuous , then for all $t$,
	$Q^{(t)} \leq 4L^2 \pi$, where
	\[\pi=\sum_{i=1}^{m}\pi_i n_i, \;\mbox{ and }\;
	\pi_i = \underset{\boldsymbol{\alpha}_{[i]} \in \mathbb{R}^{n_i}}{\text{max}} \frac{\|\boldsymbol{\alpha}_{[i]}^T \mathbf{K} \boldsymbol{\alpha}_{[i]}\|^2}{\|\boldsymbol{\alpha}_{[i]}\|^2}.\]
\end{lemma}

When all $\phi(\mathbf{x}_j^i)$ are normalized to $\|\phi(\mathbf{x}_j^i)\|^2 \!\leq\! 1$, we have
$\pi_i \!\leq\! \frac{\sigma_{ii}}{n_i}$, and therefore
$Q^{(t)} \!\leq\! 4L^2\sum_{i=1}^{m}\sigma_{ii}$.

Now, we are ready to present the convergence theorems for smooth loss functions and non-smooth general convex loss functions in Theorem~\ref{convergence_smooth} and Theorem~\ref{convergence_general}, respectively.

\begin{theorem}\label{convergence_smooth}
	Consider $\mathbf{W}$-step in Algorithm \ref{DDCAMTL} with $\boldsymbol{\alpha}^{(0)} \!=\! 0$. Assume that $l_j^i(\cdot)$ are ($1/\mu$)-smooth for all $(i,j)$. Let $ i^* \!\!=\! \underset{i}{\text{argmax}}\frac{-\lambda\mu(1-s)}{\rho s n_i} + \pi_i$. To obtain $\mathbb{E}[D(\boldsymbol{\alpha}^*) \!-\! D(\boldsymbol{\alpha}^{(t)})] \leq \epsilon_D$, it suffices to have $t$ number of iterations with
	\begin{equation*}t \geq \frac{1}{\eta(1-\Theta)}\frac{\lambda \mu + \rho n_{i^*} \pi_{i^*}}{\lambda\mu} \textnormal{log}\frac{m}{\epsilon_D}.\end{equation*}
	
	To obtain expected duality gap \begin{small}$ \mathbb{E}[P(\mathbf{W}(\boldsymbol{\alpha}^{(t)})) \!-\! D(\boldsymbol{\alpha}^{(t)})] \!\leq\! \epsilon_G$\end{small}, it suffices to have $t$ number of iterations with
	\begin{equation*} 
	t \geq \frac{1}{\eta(1\!-\!\Theta)}\frac{\lambda \mu \!+\! \rho n_{i^*} \pi_{i^*}}{\lambda\mu} \textnormal{log}\bigg( \frac{m }{\eta(1\!-\!\Theta)} \frac{\lambda \mu \!+\! \rho n_{i^*} \pi_{i^*}}{\lambda\mu} \frac{1}{\epsilon_G} \bigg).\end{equation*}
\end{theorem}

\begin{theorem}\label{convergence_general}
	Let ${l_j^i}(\cdot)$ be L-Lipschitz continuous and $\epsilon_G > 0$ be the duality gap. Then after $T$ iterations in $\mathbf{W}$-step of Algorithm \ref{DDCAMTL}, when
	\begin{eqnarray}
	T   & \geq & T_0 + \textnormal{max} \left\{\ceil[\Big]{\frac{1}{\eta(1-\Theta)}},\frac{4L^2 \pi \rho}{\lambda\epsilon_G\eta(1-\Theta) }  \right\}, \nonumber \\
	T_0 & \geq & t_0 + \bigg(\frac{2}{\eta(1-\Theta)}\left(\frac{8L^2\pi\rho}{\lambda \epsilon_G} -1\right) \bigg)_{\!\!+}, \nonumber \\
	t_0 & \geq & \textnormal{max}\bigg(0,\ceil[\Big] {\frac{1}{\eta(1-\Theta)}\textnormal{log} \bigg( \frac{2\lambda m}{4L^2\pi\rho} \bigg)} \bigg), \nonumber
	\end{eqnarray}
	we have $\mathbb{E}[P(\mathbf{w}(\bar{\boldsymbol{\alpha}})) \!-\! D(\bar{\boldsymbol{\alpha}})] \!\leq\! \epsilon_G$, where $\bar{\boldsymbol{\alpha}}$ is the average $ \boldsymbol{\alpha}$ over $T_0\!+\!1$ to $T$ iterations,
	$\bar{\boldsymbol{\alpha}} \!=\! \frac{1}{T-T_0}\sum_{t=T_0}^{T-1} \boldsymbol{\alpha}^{(t)}$.
\end{theorem}

Note that regarding the primal-dual convergence analysis, our framework is not restricted to use the SDCA method as the local solver. Any other local optimization methods that achieve a $\Theta$-approximate solution could be used to achieve primal-dual convergence for the global problem. Our analysis shows that the outer iteration $T$ depends on $\Theta$, i.e., how local solution approximates the optimal local solution. This implies the trade-off between local computation ($\Theta$) and rounds of communication ($T$). We will discuss it in details in the experiments section.

\subsection{Effect of Task Relationships on Primal-Dual Convergence Rate}\label{effect_task_relationship}
Finally, in this section, we analyze how task relationships affect the primal-dual convergence rate in our algorithm. Previously from~\eqref{rhodef}, we know that the parameter $\rho$ must be not smaller than $\rho_{\text{min}}$.
We have the upper bound for $\rho_{\text{min}}$ given by Lemma~\ref{rho_upp_bound}.
\begin{lemma}\label{rho_upp_bound}
	$\rho_{\text{min}}$ is upper bounded by $\eta \times \underset{i}{\text{max}} \sum_{i'=1}^{m} \frac{|\sigma_{ii'}|}{\sigma_{ii}}$.
\end{lemma}
\begin{proof}
	\begin{equation*}\label{alpha_K_alpha}
	\begin{aligned}
	{\boldsymbol{\alpha}}^T \mathbf{K} \boldsymbol{\alpha}
	& = \sum_{i=1}^{m} {{\boldsymbol{\alpha}}_{[i]}^T \mathbf{K} \sum_{i'=1}^{m} \boldsymbol{\alpha}_{[i']}}\\
	& = \sum_{i=1}^{m} \sum_{i'=1}^{m} \boldsymbol{\alpha}_{[i]}^T \mathbf{K}\boldsymbol{\alpha}_{[i']}\\
	& = \sum_{i=1}^{m} \sum_{i'=1}^{m} \sigma_{ii'}  \langle \frac{1}{n_i} \sum_{i=1}^{n_i} \alpha_j^i \mathbf{x}_j^i, \frac{1}{n_{i'}}  \sum_{i'=1}^{n_{i'}} \alpha_j^{i'} \mathbf{x}_j^{i'} \rangle \\
	& \leq \sum_{i=1}^{m} \sum_{i'=1}^{m} \frac{1}{2} |\sigma_{ii'}| \left( \frac{1}{{n_i}^2} \left\| \sum_{i=1}^{n_i} \alpha_j^i \mathbf{x}_j^i\right\|^2   \!+\! \frac{1}{{n_{i'}}^2} \left\| \sum_{i'=1}^{n_{i'}} \alpha_j^{i'} \mathbf{x}_j^{i'}\right\|^2      \right) \\
	& = \sum_{i=1}^{m} \sum_{i'=1}^{m} \frac{1}{2} \bigg( \frac{|\sigma_{ii'}|}{\sigma_{ii}} \boldsymbol{\alpha}_{[i]}^T \mathbf{K}\boldsymbol{\alpha}_{[i]} + \frac{|\sigma_{ii'}|}{\sigma_{i'i'}} \boldsymbol{\alpha}_{[i']}^T \mathbf{K}\boldsymbol{\alpha}_{[i']} \bigg)\\
	& = \sum_{i=1}^{m} \sum_{i'=1}^{m} \frac{|\sigma_{ii'}|}{\sigma_{ii}} \boldsymbol{\alpha}_{[i]}^T \mathbf{K}\boldsymbol{\alpha}_{[i]}.
	\end{aligned}
	\end{equation*}
	It follows that
	\begin{equation*}
	\begin{small}
	{\text{max}} \frac{{\boldsymbol{\alpha}}^T \mathbf{K} \boldsymbol{\alpha}} {\sum_{i=1}^{m} {\boldsymbol{\alpha}}_{[i]}^T \mathbf{K} \boldsymbol{\alpha}_{[i]}} \leq \underset{i}{\text{max}} \sum_{i'=1}^{m} \frac{|\sigma_{ii'}|}{\sigma_{ii}}.
	\end{small}
	\end{equation*}
\end{proof}

This upper bound on $\rho_{\text{min}}$ can be interpreted as the maximum sum of relative task covariance between task $i$ and all other tasks. Consider two extreme conditions of the upper bound:
\begin{itemize}
	\item Every task is equally correlated. In this case, the precision matrix $\boldsymbol{\Omega}$ is a Laplacian matrix defined on a fully connected graph with $0/1$ weight. Then the task covariance matrix $\boldsymbol{\Sigma}=\boldsymbol{\Omega}^{-1}$ has equal elements. Therefore, the upper bound of $\rho_{\text{min}}$ becomes $\eta \times m$, where $m$ is the number of tasks.
	\item Every task has no correlation with each other. Under this condition, as $\boldsymbol{\Sigma}=\boldsymbol{\Omega}^{-1}$ is learned from the uncorrelated \{$\mathbf{w}_i$\}'s, the absolute values of its diagonal elements dominate others. Therefore, the upper bound of $\rho_{\text{min}}$ becomes close to $\eta$.
\end{itemize}

From Theorems \ref{convergence_smooth} and \ref{convergence_general}, the smaller $\rho_{\text{min}}$ is, the faster the primal-dual convergence rate is.
This is coherent with the MTL intuition. When the tasks have no or very weak correlation with each other, there is no or very little interaction between the updates from each task. Each task's weight vector could be updated almost independently. And thus the convergence rate will be faster in this case. On the contrary, when the tasks have strong correlation with each other, there will be relatively strong interaction between the updates from all tasks. As a result, the interaction between each other's updates will impact the convergence rate to become slower compared to the former situation. Looking from another angle, $\rho_{\text{min}}$ could be interpreted as a measure of the separability of the objective function in~\eqref{generaldual}. Smaller $\rho_{\text{min}}$ means that the objective function is easier to be separated and distributed. Therefore, the primal-dual convergence rate will become faster.

\section{Experiments}\label{sec:expriment}

\subsection{Implementation Details and Setup}
We implement DMTRL on Petuum~\citep{xing2015petuum}, which is a distributed machine learning platform.\footnote{Note that our proposed method could be implemented on other distributed machine learning platforms as well.} And we run it on a local cluster consisting of 4 machines with 16 worker cores each. For datasets whose number of tasks is less than the total number of cores, we assign each task to one worker core. Otherwise, we equally distribute tasks over the available cores and each core run the local subproblem update sequentially task by task. Due to limitation of resources, we are not able to distribute each task on one machine. However, the experimental results presented later show good convergence performance and promise for distributing over more machines. In all the experiments, we set the aggregation scaling parameter $\eta\!=\!1$, and $\rho$ is set to $\underset{i}{\text{max}} \sum_{i'\!=1}^{m} \!\!\frac{|\sigma_{ii'}|}{\sigma_{ii}}$ in each $\boldsymbol{\Omega}$-step according to Lemma~\ref{rho_upp_bound}. Regarding $\phi(\cdot)$ in DMTRL, we use a linear mapping.
We compare our method with three baselines:

\begin{itemize}
	\item Single Task Learning (\textbf{STL}): each task is solved independently as a single empirical risk minimization problem.
	
	\item \textbf{Centralized MTRL}: all tasks are gathered in one machine and MTRL is implemented centrally as described in~\citep{zhang2010convex}. This baseline can be considered as a gold standard solution for learning task relationships for MTL, but it fails to work in a distributed computation manner.
	
	\item Single-machine SDCA (\textbf{SSDCA}): all tasks are centralized in one machine where SDCA is performed over all coordinates of $\boldsymbol{\alpha}$. This method could handle the case when there is too much centralized data that \textbf{Centralized MTRL} could not handle. It could be regarded as a scalable single machine solution to MTRL.
\end{itemize}

We conduct extensive experiments on the following synthetic and benchmark datasets.
\begin{itemize}
	\item \textbf{Synthetic 1}: we generate a synthetic dataset for binary classification with 16 tasks with feature dimension 100. Weight vectors of three ``parent'' tasks $\{\mathbf{w_1},\mathbf{w_6},\mathbf{w_{11}}\}$ are first randomly initiated. Then other tasks' weight vectors (``child'' tasks) are initialized by choosing one randomly from $\{\mathbf{w_1},-\mathbf{w_1},\\\mathbf{w_6},-\mathbf{w_6},\mathbf{w_{11}}, -\mathbf{w_{11}}\}$ and adding some random noise to the parameters. The negative sign is to simulate tasks with negative relationships. The instances for each task are randomly generated. Labels are generated using the logistic regression model. The averaged number of training instances per task is 1,894, while the averaged number of test instances per task is 811. The total number of instances equals 43,280.
	
	\item \textbf{Synthetic 2}: another synthetic dataset is generated using the same data as the first one but with different task weight parameters such that there are more task correlations than the first one. For \textbf{Synthetic 2}, $\rho$ (\eqref{rhodef} in Lemma~\ref{dual_lemma}) equals to $12.9457$ while $\rho = 6.2418$ for \textbf{Synthetic 1}. We generate this dataset to compare the primal-dual convergence rates of $\mathbf{W}$-step under different situations of task correlations.
	
	\item \textbf{School}: this is a regression dataset which contains examination scores of 15,362 students from 139 schools. Each school corresponds to a task. By adding 1 to the end of all data to account for the bias term, each data point has 28 features. The averaged number of training instances per task is 83 and the averaged number of testing instances per task is 28. For training and testing samples, we use the splits given by~\citep{argyriou2008convex}.
	
	\item \textbf{MNIST}: this is a large hand-written digits dataset with 10 classes. It contains 60,000 training and 10,000 testing instances. The data points have a feature dimension of 784.  We treat each task as an one v.s. all binary classification task in our experimental setting and draw equal number of instances from other classes randomly and assign negative labels. Thus, we arrive at training instances of 120,000 and testing instances 20,000 in total for 10 tasks.
	
	\item \textbf{MDS}~\citep{blitzer2007biographies}: this is a dataset of product reviews on 25 domains (apparel, books, DVD, etc.) crawled from Amazon.com. We delete three domains with less than 100 instances and make it a multi-task learning problem with 22 tasks. Each task is a sentiment classification task that classifies a review as negative or positive. The number of instances per task varies from 314 to 20,751, with average data size of 4,150. Training and testing samples are obtained using a 70\%-30\% split.
\end{itemize}

The statistics about the above 5 datasets is summarized in Table~\ref{dataset-table}. For all datasets, we perform experiments on 10 random splits and report the averaged results. We use hinge loss for classification problems and squared loss for regression problems.

\begin{table}[h!]
	\caption{Statistics of the datasets}
	\label{dataset-table}
	\begin{center}
		\begin{tabular}{lcccc}
			\hline
			\hline
			Dataset & \# Tasks & \# Instances & Dims & Sparsity(\%) \\
			\hline
			Synthetic (1 \& 2) & 16 & 43,280 & 100 & 100 \\
			School & 139 & 15,362  & 28     & 32.14 \\
			MNIST  & 10  & 140,000 & 784    & 19.14 \\
			MDS    & 22  & 91,290  & 10,000 & 0.9   \\
			\hline
		\end{tabular}
	\end{center}
\end{table}

\subsection{Results on Synthetic Datasets}
Our first experiment is designed to test whether task relationships can be well recovered by our proposed DMTRL in a distributed computation manner. Figures~\ref{fig:Visualization of Learnt Correlation} and Figures~\ref{fig:Visualization of True Correlation} show the comparison between the learned task correlation matrix and the ground-truth on Synthetic 1. We can see that DMTRL is able to capture the correlation between tasks accurately, and the discrepancy between the learned correlation and the ground-truth is within reasonable amount.
\begin{figure}[h!]
	\begin{center}
		\subfigure[Task correlations learned by DMTRL]{\label{fig:Visualization of Learnt Correlation}\includegraphics[trim = 12cm 0cm 9cm 0cm, clip,width=0.315\columnwidth]{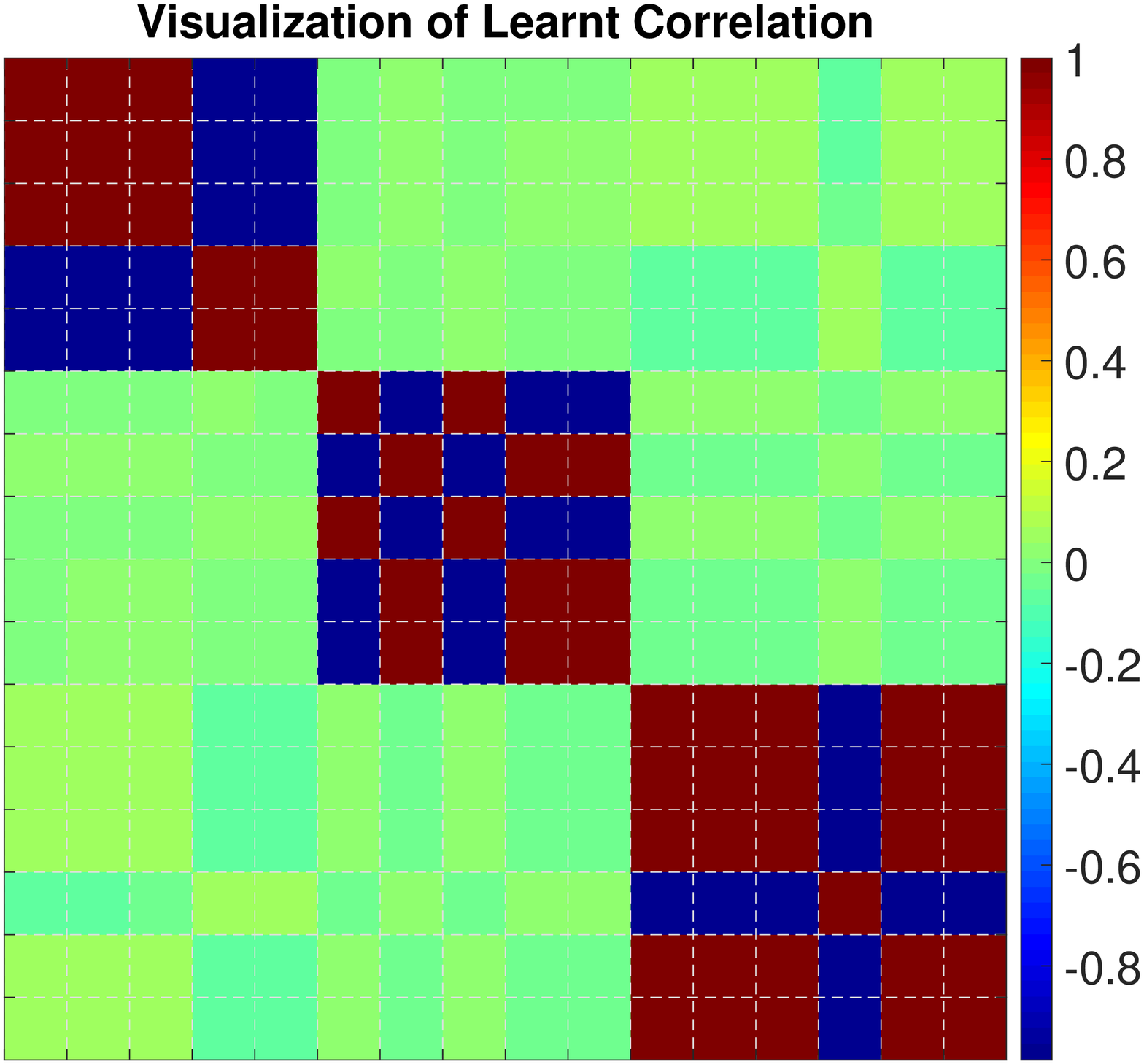}}\hspace{2mm}
		\subfigure[Ground truth task correlations]{\label{fig:Visualization of True Correlation}\includegraphics[trim =  12cm 0cm 9cm 0cm, clip,width=0.315\columnwidth]{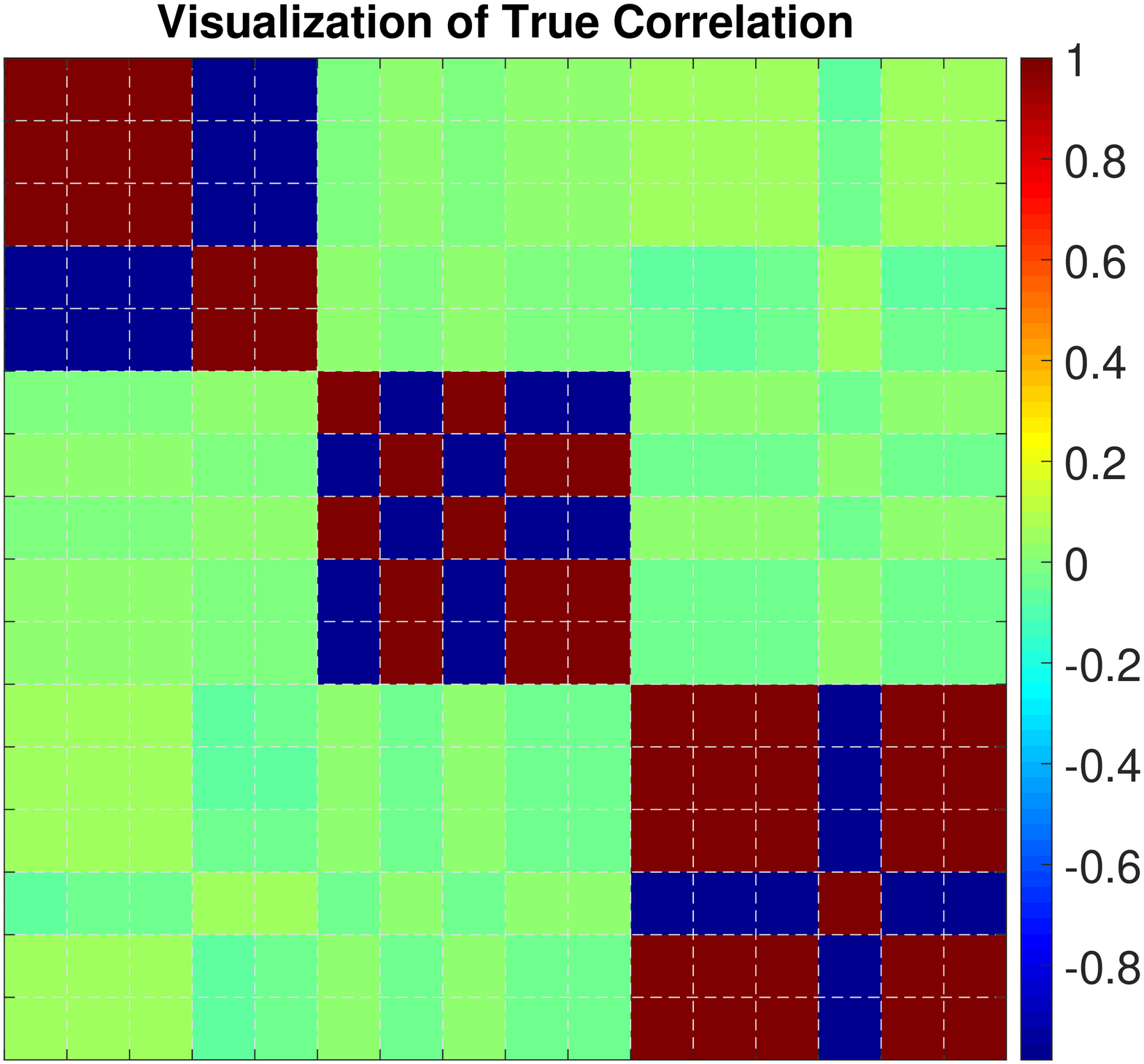}}
		\caption{Learned correlation v.s. ground-truth correlation.}\label{fig:synthetic_results1}
	\end{center}
\end{figure}

Our second experiment is designed to test under different situation of tasks correlations, i.e., different values of $\rho$, how our proposed algorithm converges. Figure~\ref{fig:commu_syn_1_2_compare} shows the comparison results of primal-dual convergence rate on Synthetic 1 and Synthetic 2. Convergence rate is slower when there are more task correlations (Synthetic 2) given same data. This verifies our discussion of the impact of task relationships on primal-dual convergence rate.

\begin{figure}[h!]
	\centering
	\includegraphics[trim = 11cm 0cm 12cm 0cm, clip, width=0.4\columnwidth]{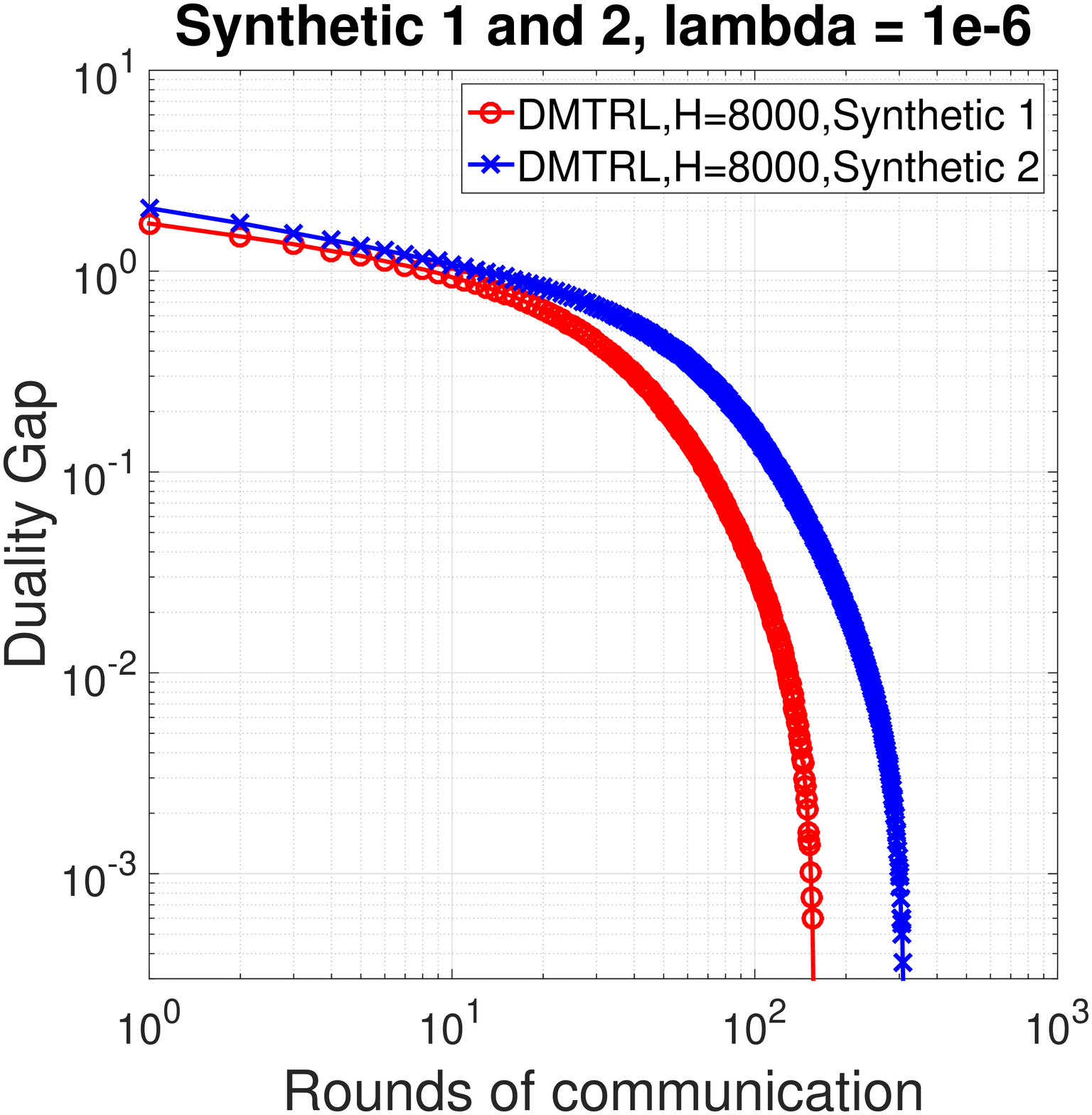}
	\caption{Convergence on different task correlations}\label{fig:commu_syn_1_2_compare}
\end{figure}
\begin{figure}[h!]
	\begin{center}
		\subfigure[Duality gap v.s. Time]{\label{fig:dualGap_syn_1e-6}\includegraphics[trim = 10cm 0cm 12cm 0cm,clip,width=0.315\columnwidth]{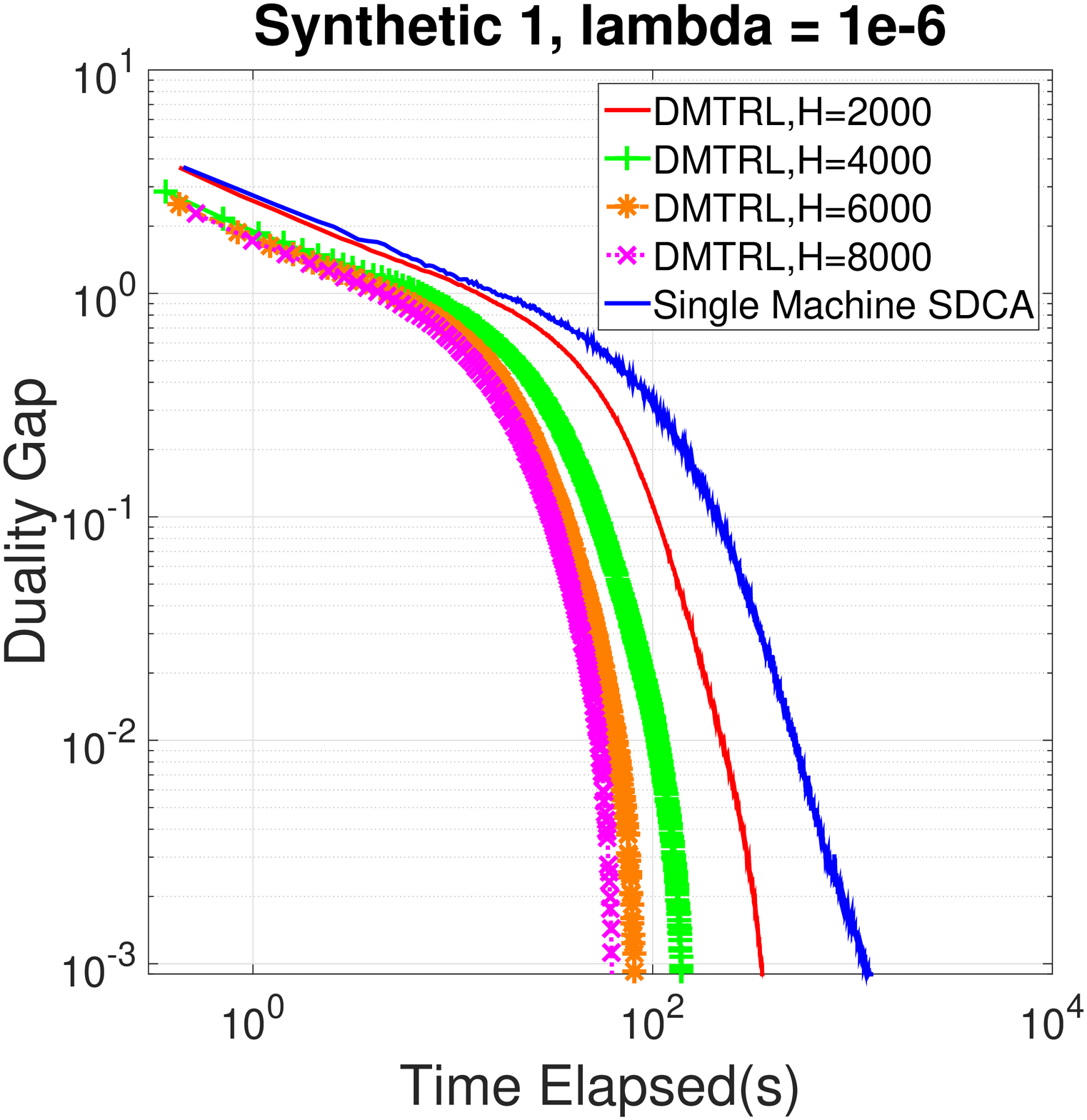}}
		\subfigure[Duality gap v.s. Comm.]{\label{fig:commu_syn_1e-6}\includegraphics[trim = 10cm 0cm 12cm 0cm,clip,width=0.315\columnwidth]{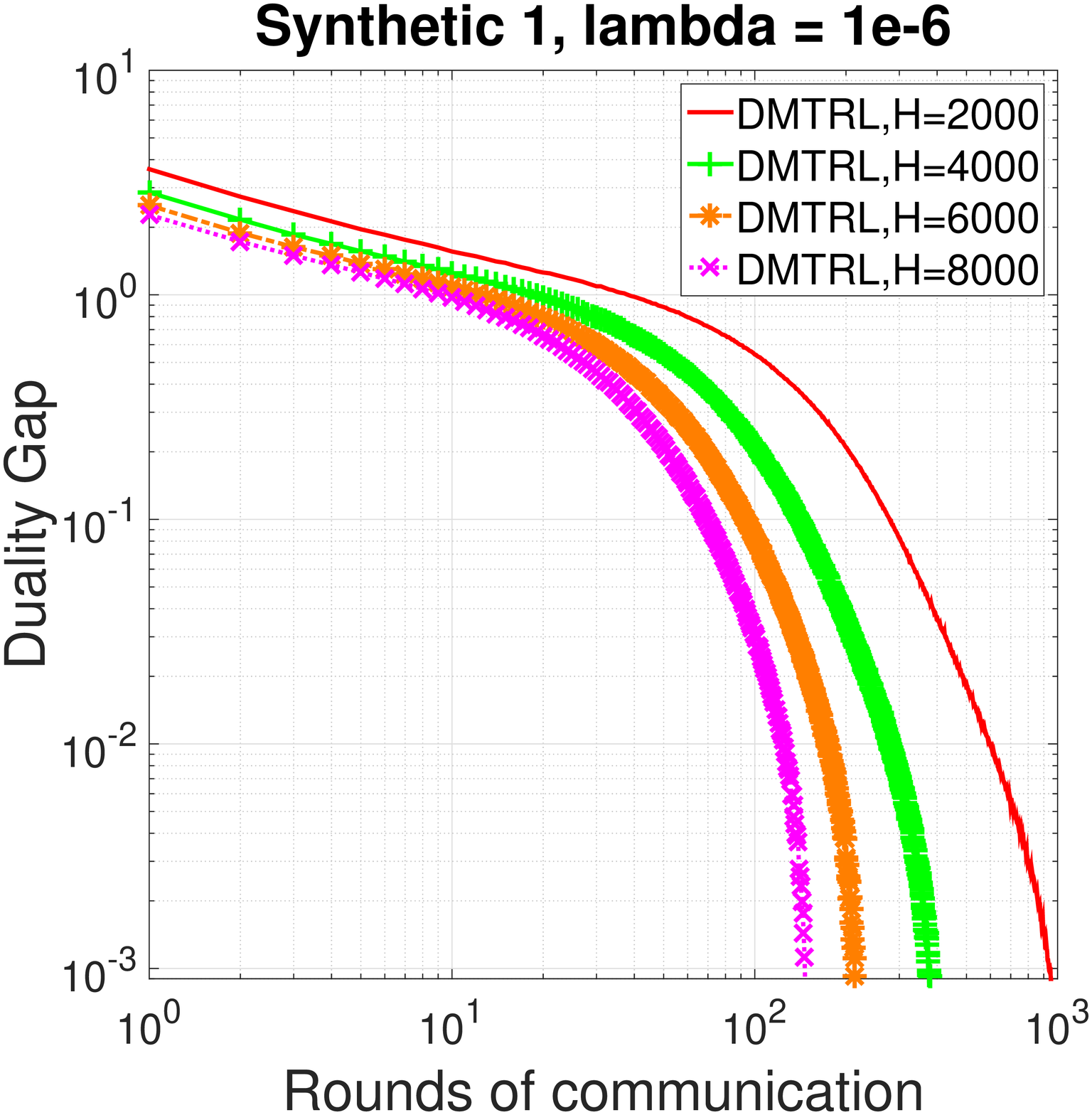}}
		\subfigure[Prediction v.s. Comm.]{\label{fig:pred_syn_1e-6}\includegraphics[trim = 10cm 0cm 12cm 0cm,clip,width=0.315\columnwidth]{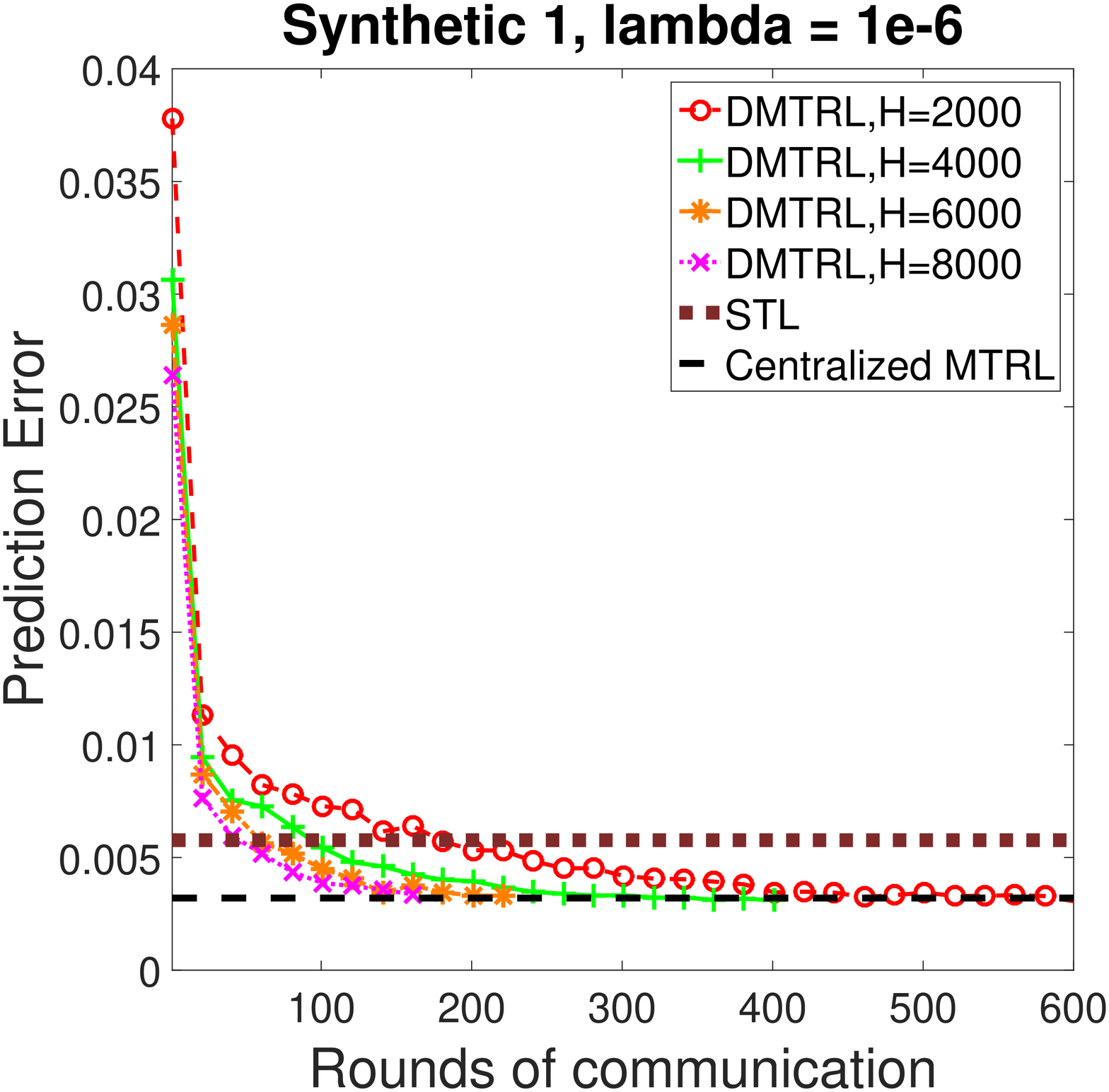}}
		\\
		\subfigure[Duality gap v.s. Time]{\label{fig:dualGap_syn_1e-5}\includegraphics[trim = 10cm 0cm 12cm 0cm,clip,width=0.315\columnwidth]{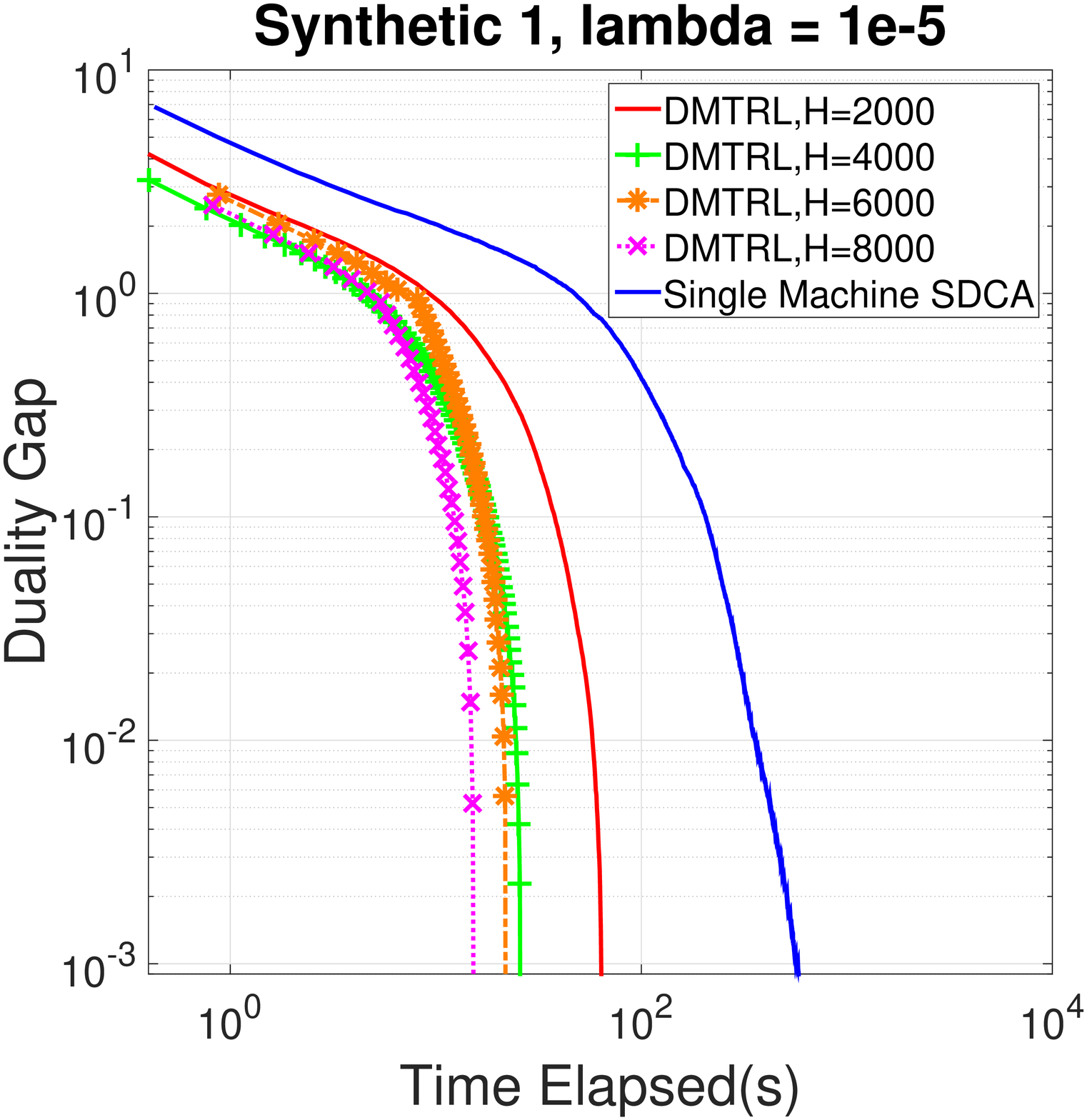}}
		\subfigure[Duality gap v.s. Comm.]{\label{fig:commu_syn_1e-5}\includegraphics[trim = 10cm 0cm 12cm 0cm,clip,width=0.315\columnwidth]{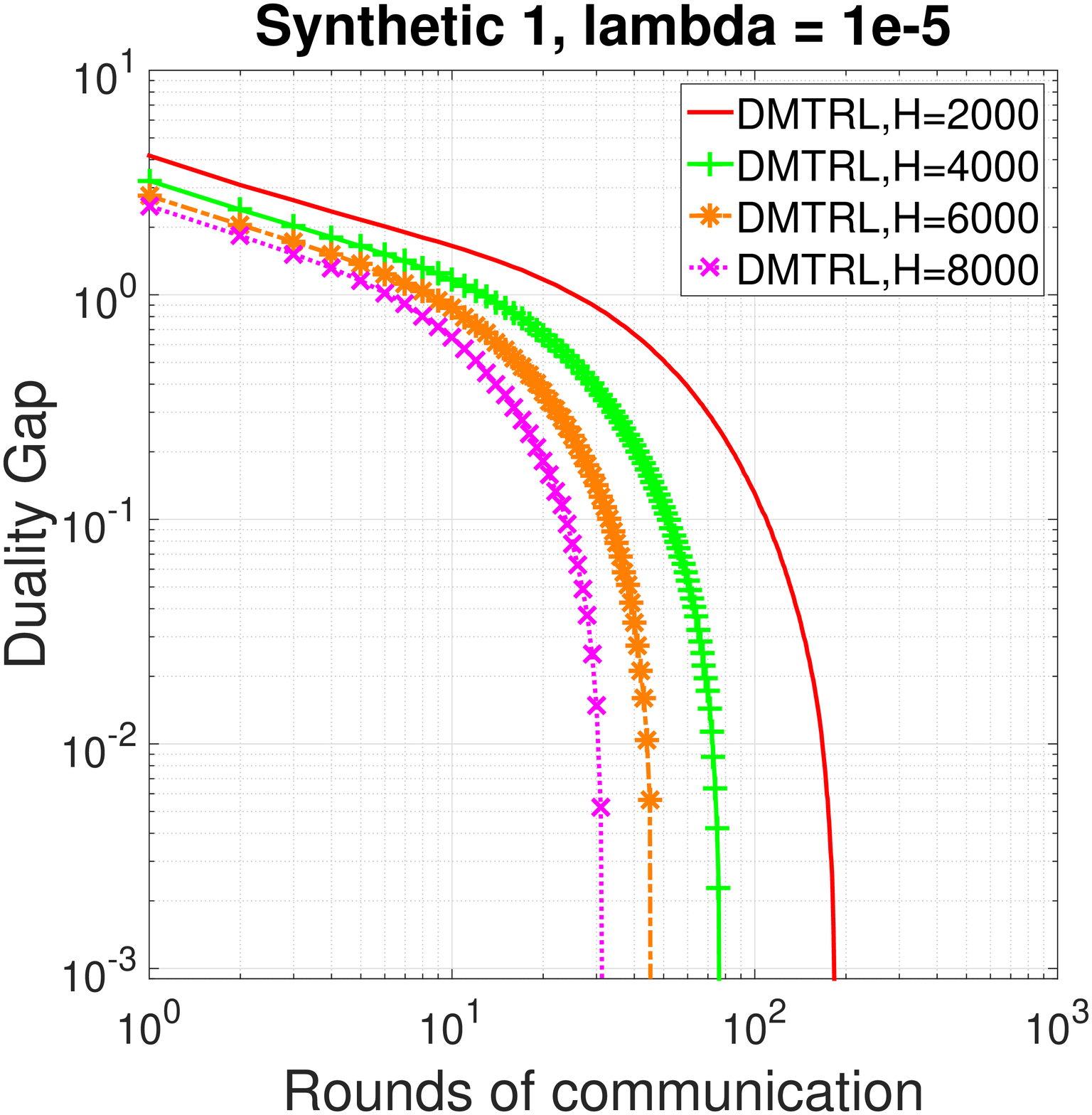}}
		\subfigure[Prediction v.s. Comm.]{\label{fig:pred_syn_1e-5}\includegraphics[trim = 10cm 0cm 12cm 0cm,clip,width=0.315\columnwidth]{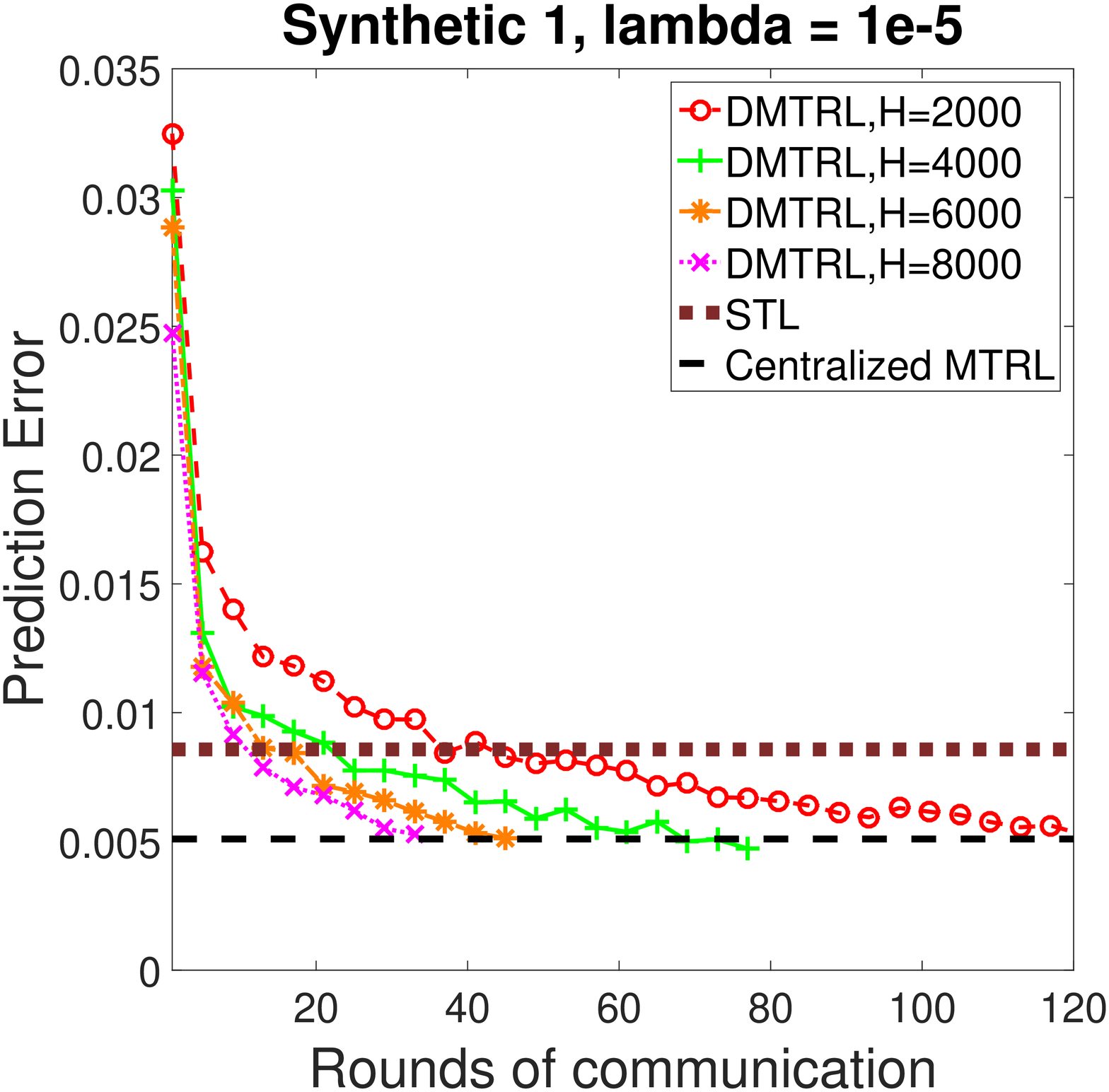}}
		\caption{Experimental results on Synthetic 1}\label{fig:syn_experiment}
	\end{center}
\end{figure}
Our third experiment is to test the performance of our proposed algorithm in terms of convergence time, communication cost, and classification accuracy. Figures~\ref{fig:dualGap_syn_1e-6}-\ref{fig:pred_syn_1e-6} show the experimental results of duality gap v.s. elapsed time, duality gap v.s. rounds of communication, and prediction error v.s. rounds of communication on Synthetic 1, where $\lambda$ in \eqref{formulation} in the MTL formulation is set to be $\lambda=10^{-6}$. Figure~\ref{fig:dualGap_syn_1e-6} shows comparison results in terms of convergence performance in $\mathbf{W}$-step between DMTRL and Single-machine SDCA, where $H$ denotes the number of iteration in local SDCA. Superior performance of DMTRL verifies that besides distributed, speed-up comes as a by-product of our algorithm in the case of solving multi-task learning in large-scale. From Figure~\ref{fig:commu_syn_1e-6}, we observe that with the increase of local computation, the number of communication rounds reduces. This is inline with the theoretical convergence analysis. When number of local SDCA iterations increases, each subproblem arrives at a better approximate solution with smaller $\Theta$ and thus the iterations of global update needed to reach convergence is reduced. From Figure~\ref{fig:pred_syn_1e-6}, we observe that if $H$ is larger, then fewer communication rounds is needed for DMTRL to converge to the optimal solution, which is as the same as the solution obtained by Centralized MTRL. We also conduct experiments with a different value of $\lambda$ ($\lambda=10^{-5}$), with results shown in Figures~\ref{fig:dualGap_syn_1e-5}-\ref{fig:pred_syn_1e-5}, where similar results are observed.

\subsection{Results on Real-world Datasets}
On the three real-world datasets, we focus on testing the prediction performance of DMTRL compared with the baselines. Table~\ref{tab:performance-school} and~\ref{tab:performance_realworld} report prediction performance of DMTRL with comparison to STL and Centralized MTRL. We use RMSE and explained variance as used in~\citep{argyriou2008convex} to measure the performance on School, and use averaged prediction error rate to measure the performance on MDS and MNIST. Note that Centralized MTRL fails to generate results on MDS and MNIST because of the out-of-memory issue when calculating the kernel matrix (each machine is of 16GB RAM). We also report the prediction performance of DMTRL against the number of rounds of communications with comparison with STL, Single-machine SDCA, and Centralized MTRL in Figure~\ref{fig:pred_error_realworld}.

\begin{table}[h!]
	\caption{Comparison performance in terms of RMSE and explained variance on \textbf{School}}
	\label{tab:performance-school}
	\begin{center}
		\begin{tabular}{lcccr}
			\hline
			\hline
			Method & RMSE & Explained Variance \\
			\hline
			DMTRL  & 10.23 $\pm$ 0.21  & 26.9 $\pm$  1.6\%\\
			Centralized MTRL & 10.23 $\pm$ 0.21  & 26.9 $\pm$  1.7\%\\
			STL & 11.10 $\pm$ 0.21 & 23.5 $\pm$  1.9\%\\
			\hline
		\end{tabular}
	\end{center}
\end{table}

\begin{table}[h!]
	\caption{Comparison performance in terms of error rate on \textbf{MNIST} and \textbf{MDS}}
	\label{tab:performance_realworld}
	\begin{center}
		\begin{tabular}{lccc}
			\hline
			\hline
			Data set & DMTL              & Centralized MTRL & STL \\
			\hline
			MNIST    & 5.2 $\pm$ 0.12\%  & Nil   & 5.2 $\pm$ 0.11\%\\
			\hline
			MDS      & 12.6 $\pm$ 0.09\% & Nil   & 16.0 $\pm$ 0.1\% \\
			\hline
		\end{tabular}
	\end{center}
\end{table}

\begin{figure}[h!]
	\begin{center}
		\subfigure[Prediction Error v.s. Communication] {\label{fig:pred_error_realworld}\includegraphics[trim =  10.5cm 0cm 12cm 0cm, clip, width=0.315\columnwidth]{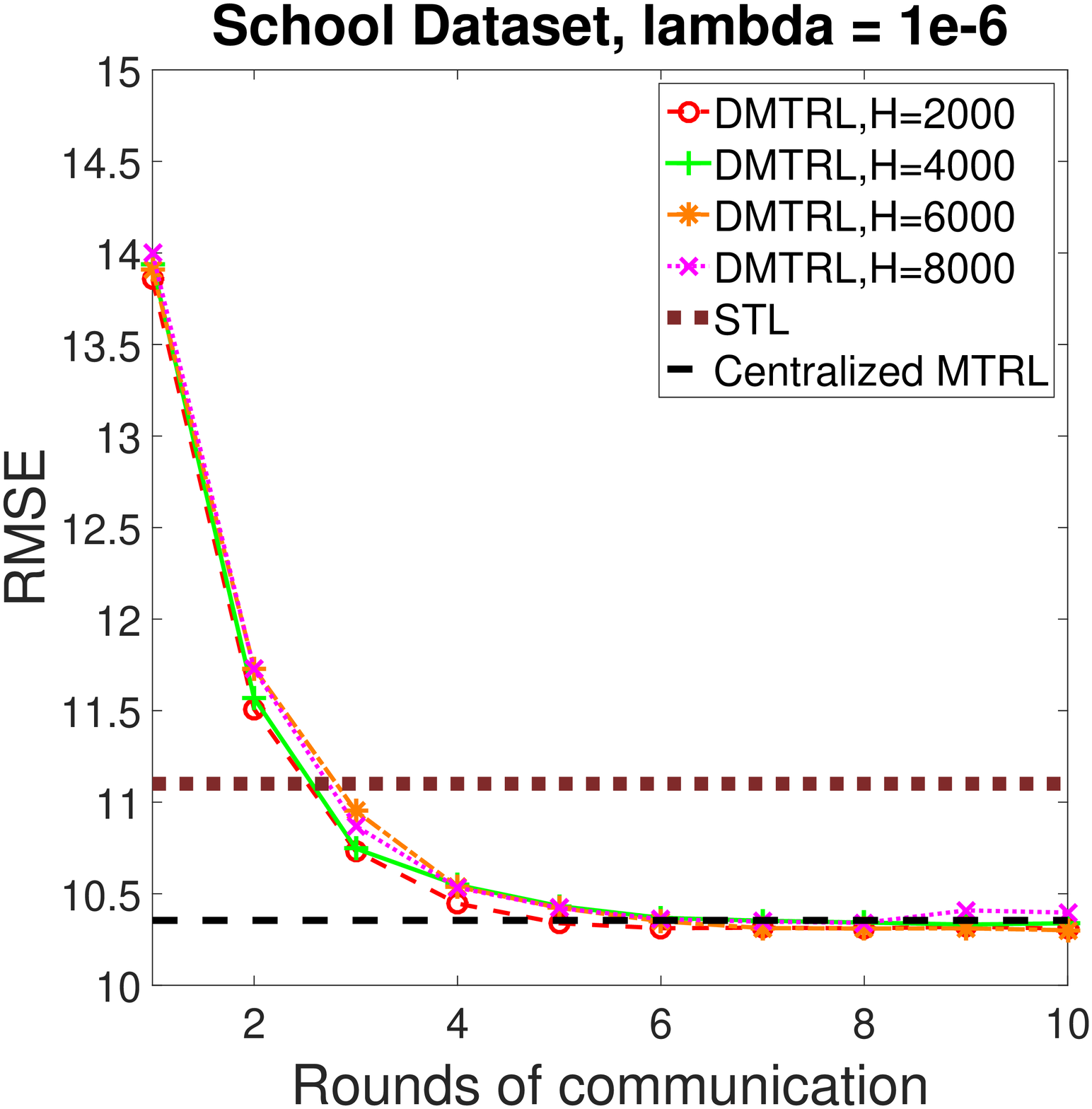}
			\includegraphics[trim = 10.5cm 0cm 12cm 0cm, clip, width=0.315\columnwidth]{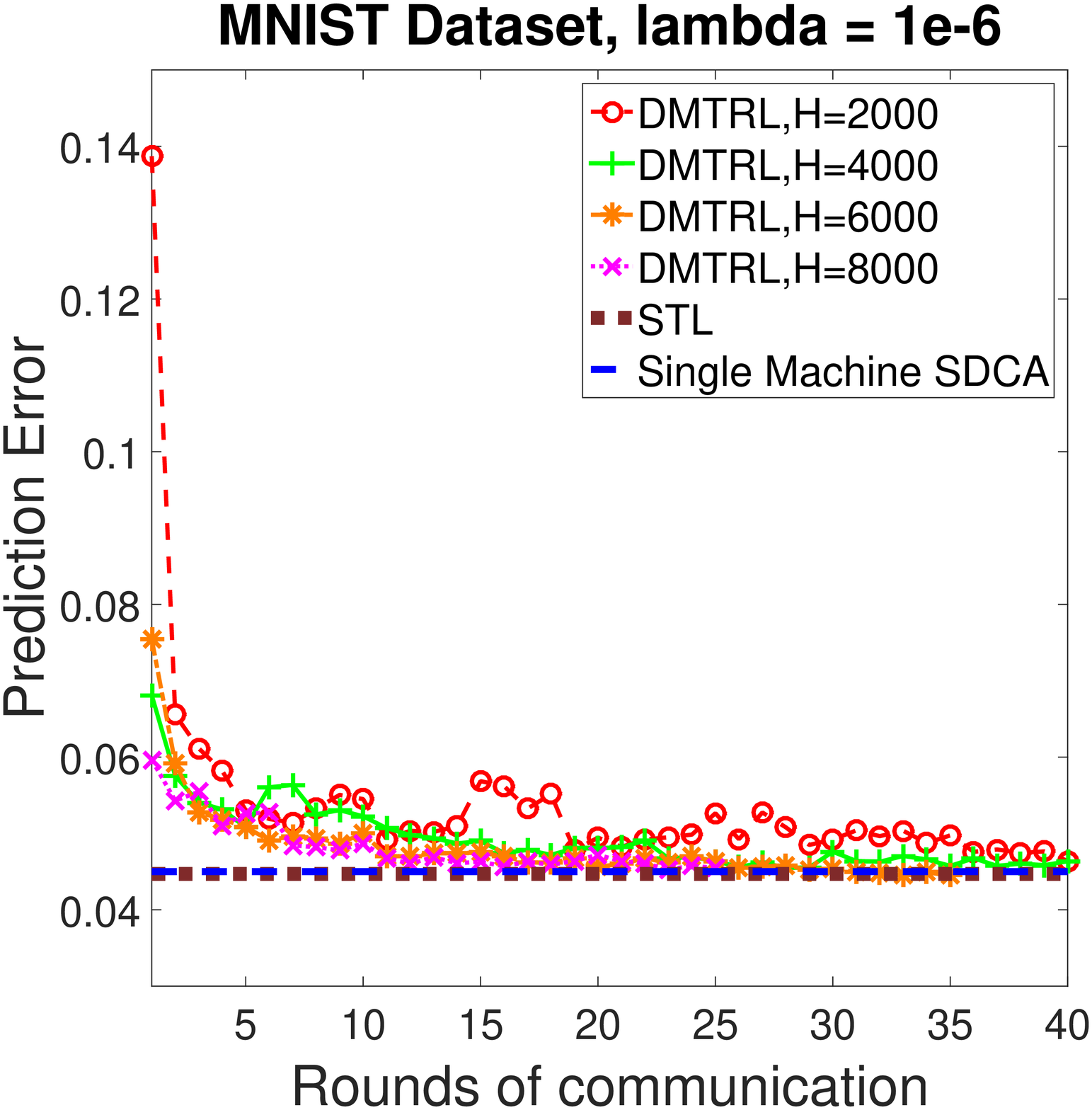}
			\includegraphics[trim = 10.5cm 0cm 12cm 0cm, clip, width=0.315\columnwidth]{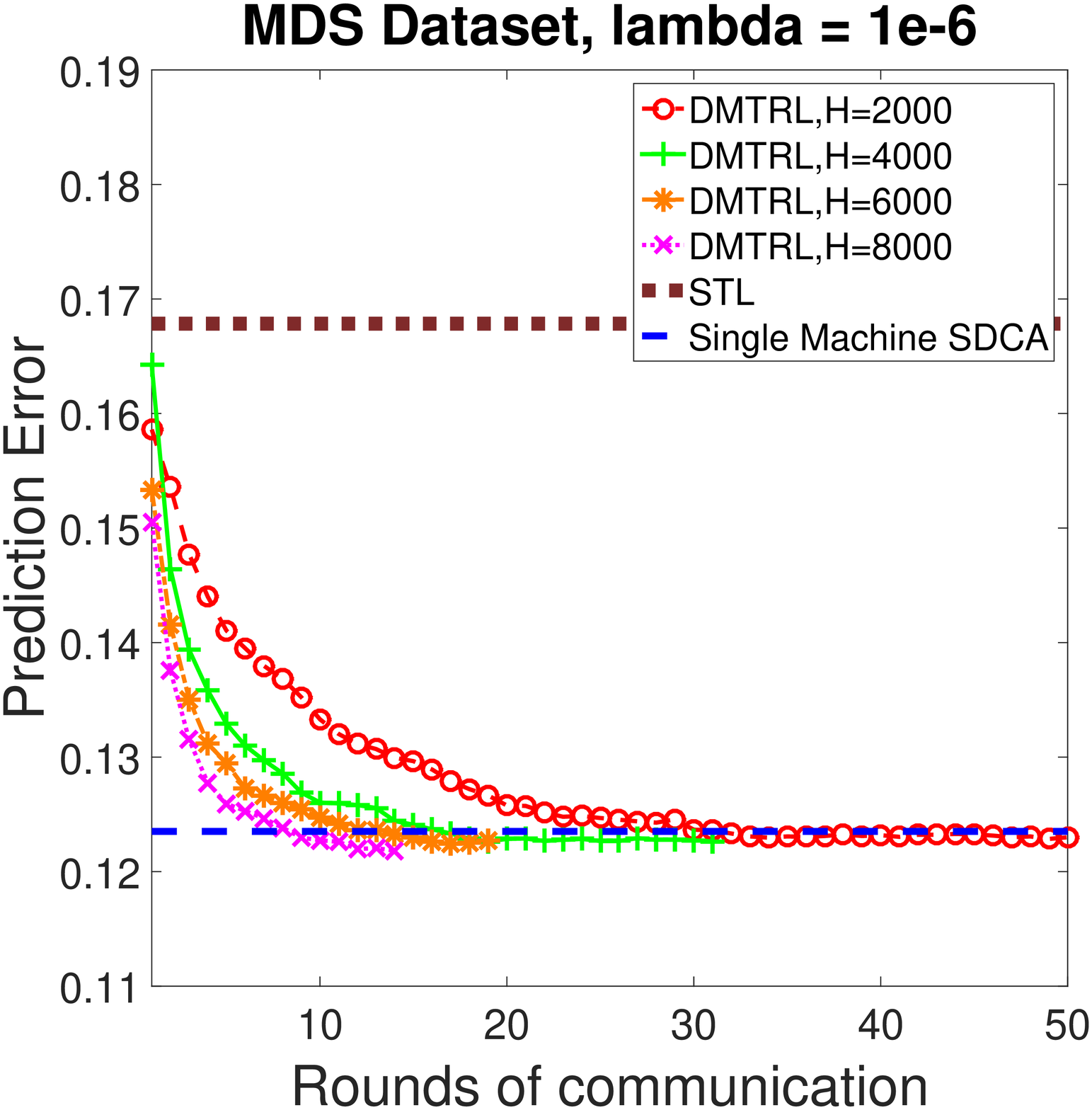}}
		\\
		\subfigure[Duality Gap v.s. Time] {\label{fig:dualGap_realworld}\includegraphics[trim = 10cm 0cm 12cm 0cm, clip, width=0.315\columnwidth]{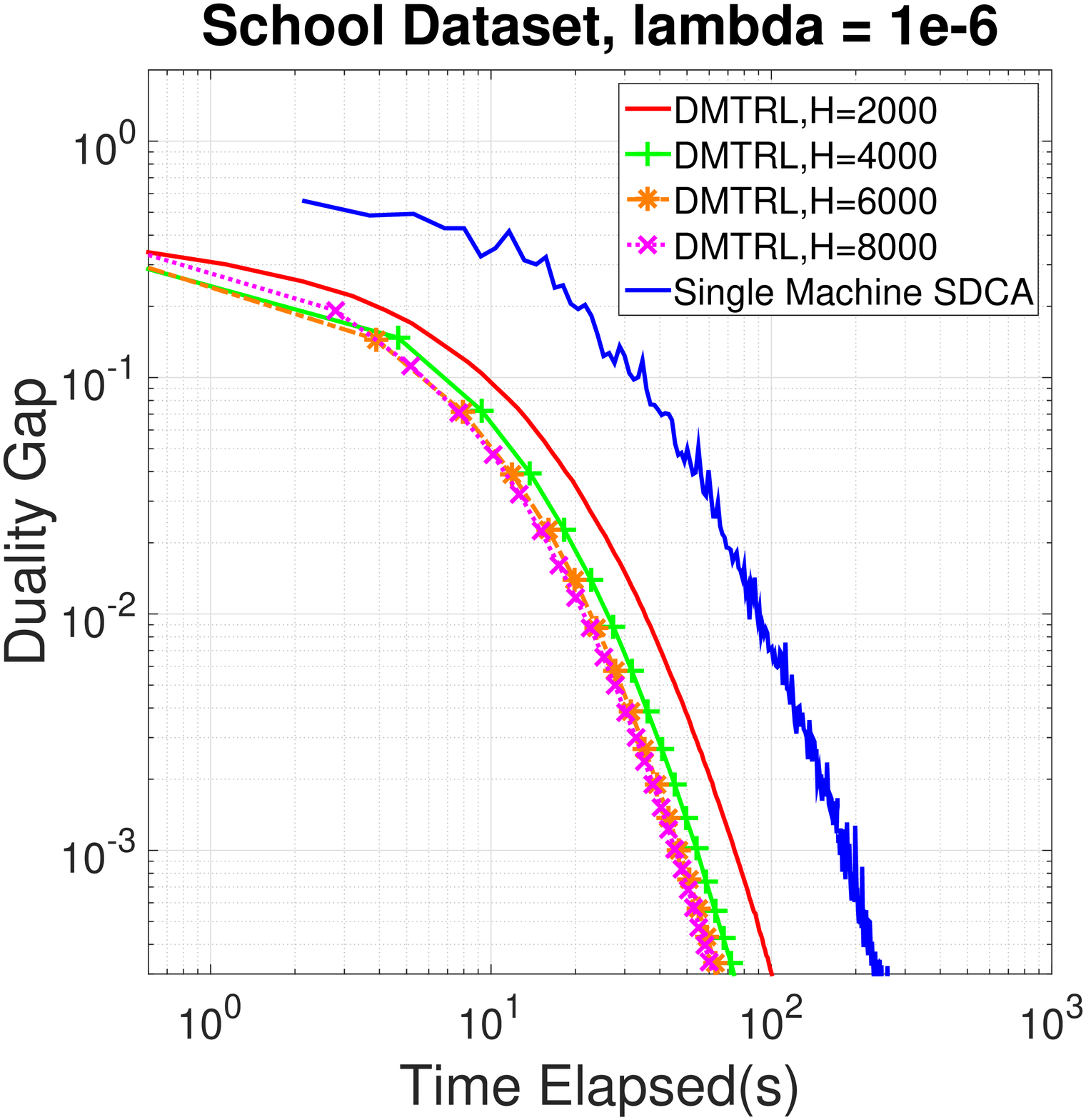}
			\includegraphics[trim = 10cm 0cm 12cm 0cm, clip, width=0.315\columnwidth]{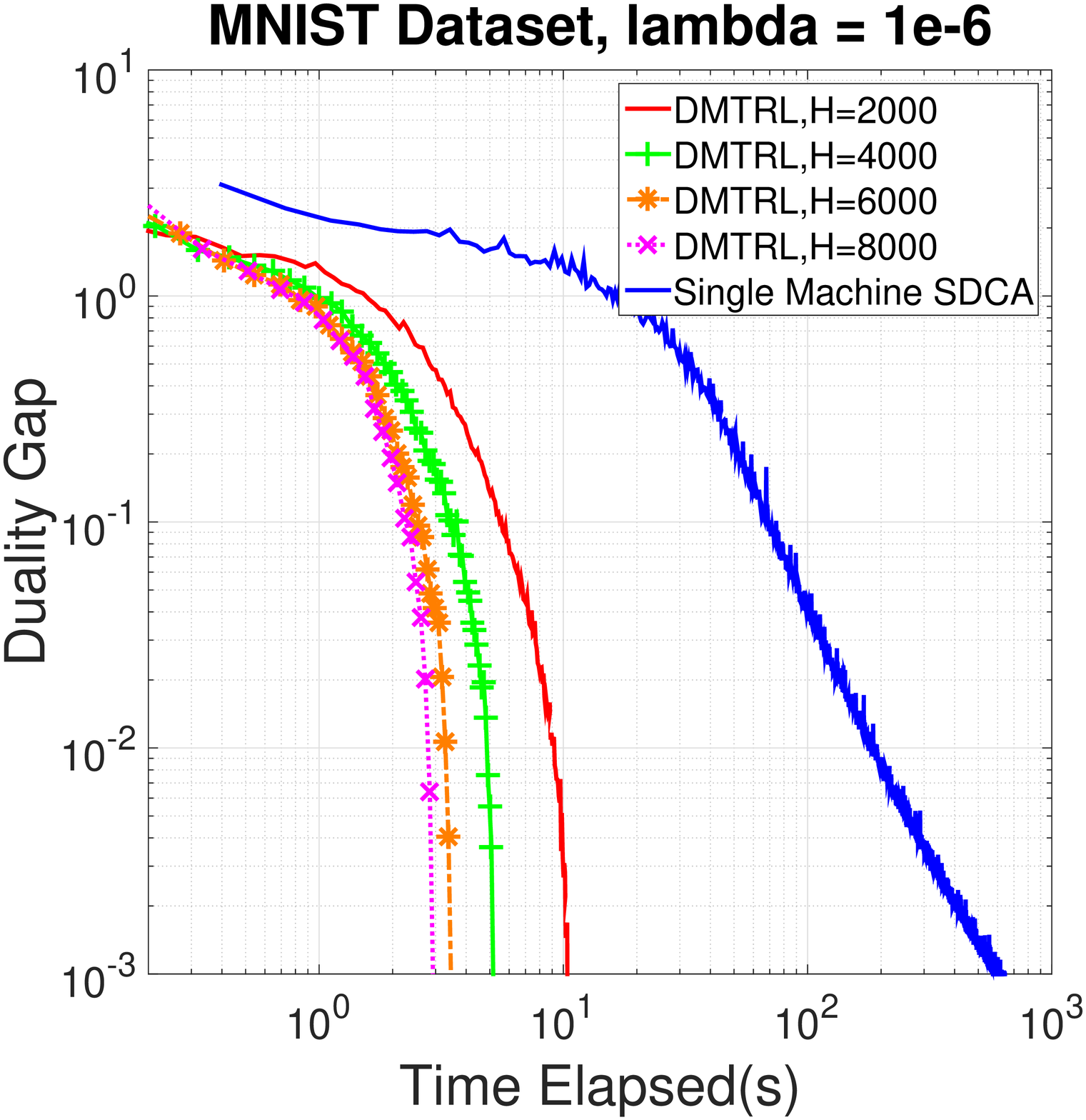}
			\includegraphics[trim = 10cm 0cm 12cm 0cm, clip,width=0.315\columnwidth]{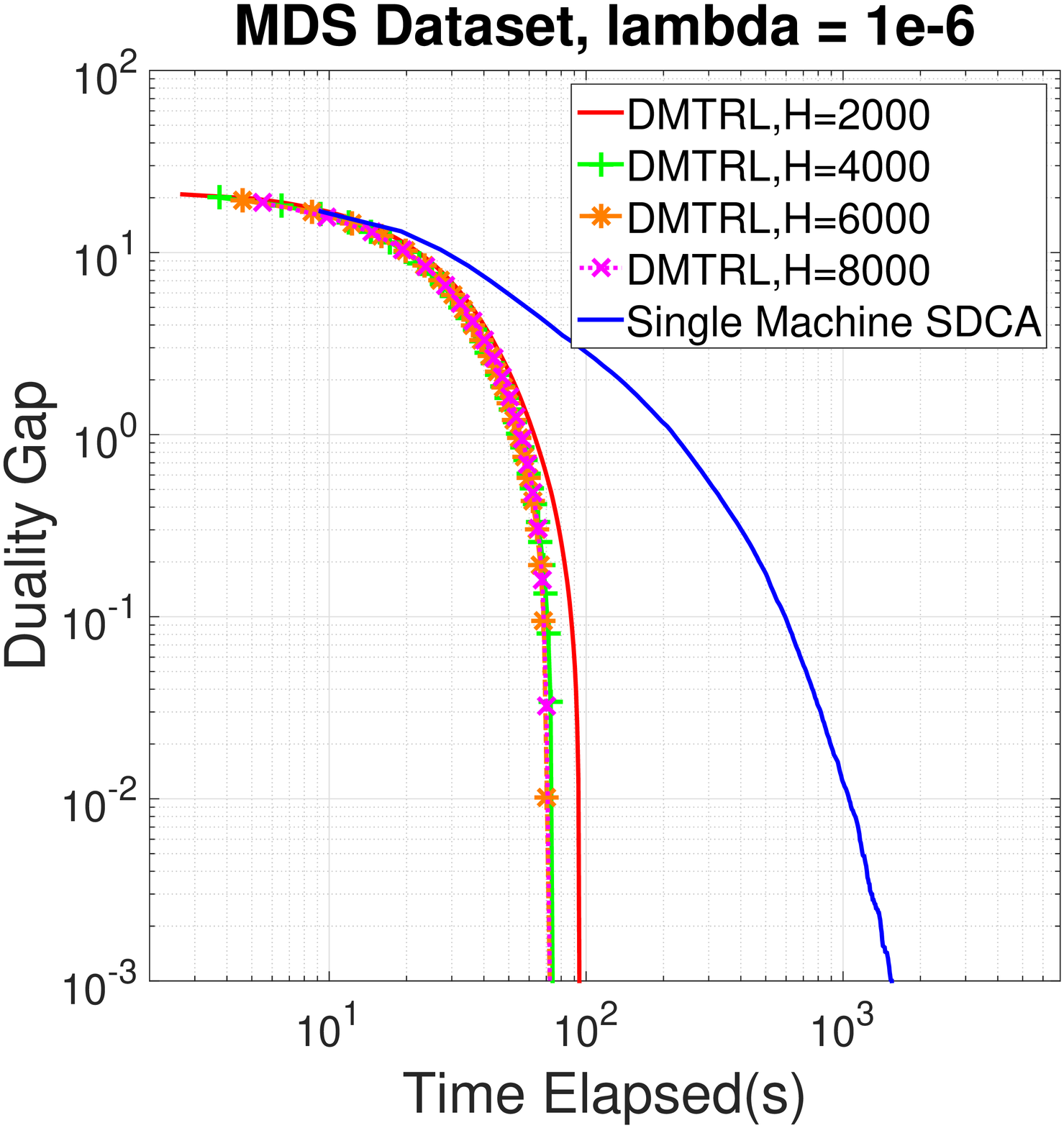}}
		\\
		\subfigure[Duality Gap v.s. Communication] {\label{fig:commu_real_world}\includegraphics[trim = 10cm 0cm 12cm 0cm,clip, width=0.315\columnwidth]{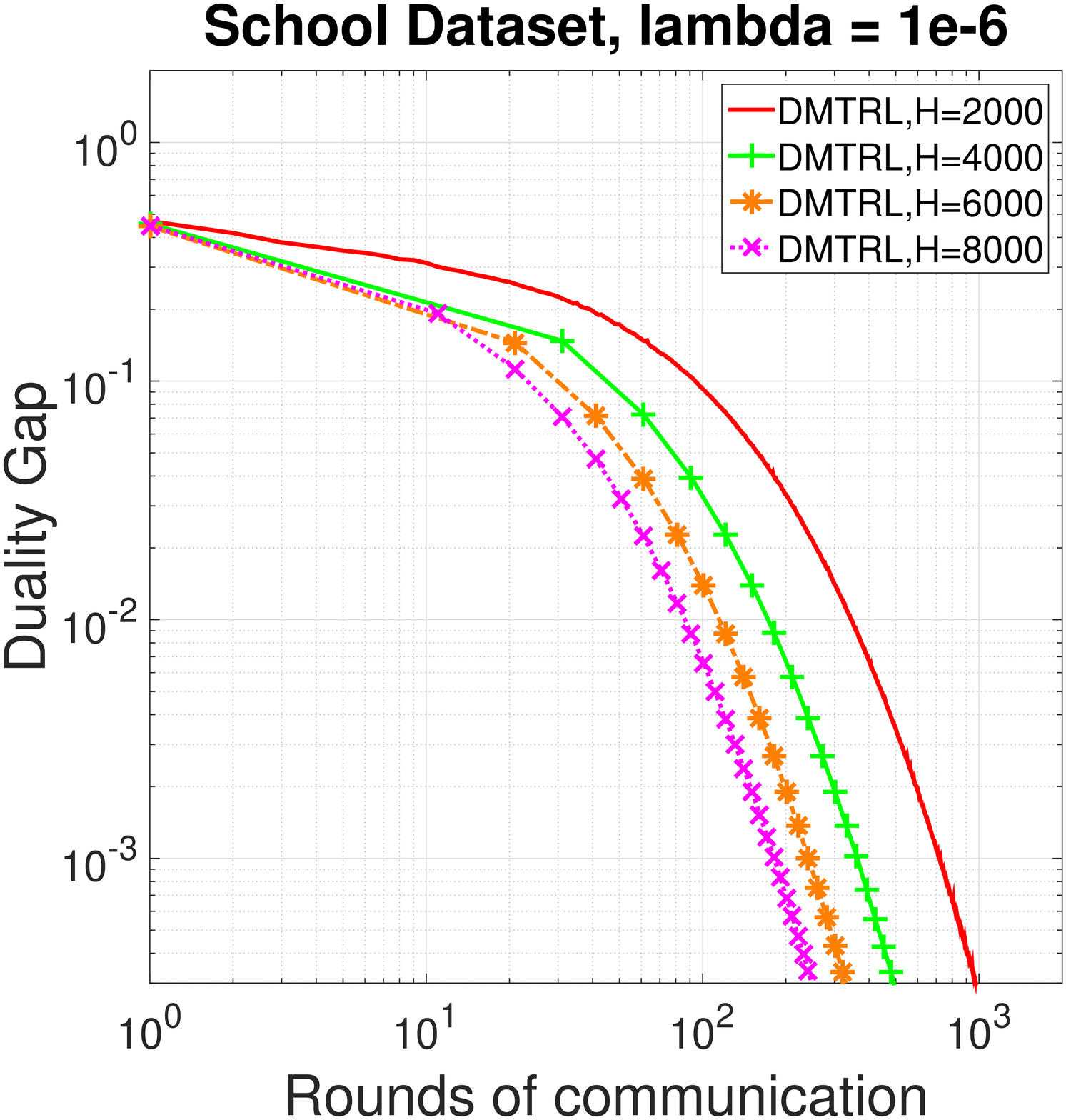}
			\includegraphics[trim = 10cm 0cm 12cm 0cm, clip, width=0.315\columnwidth]{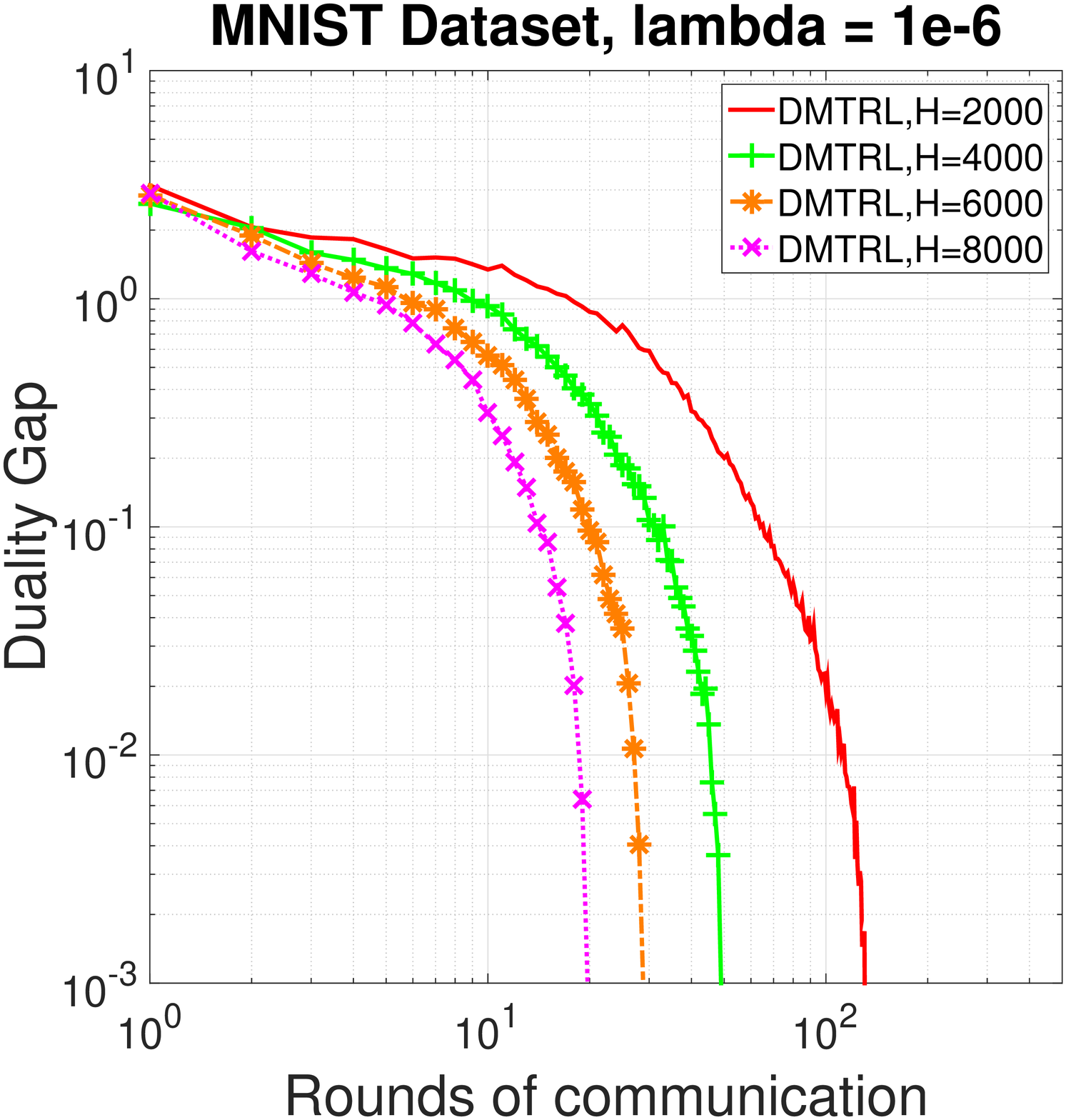}
			\includegraphics[trim = 10cm 0cm 12cm 0cm, clip, width=0.315\columnwidth]{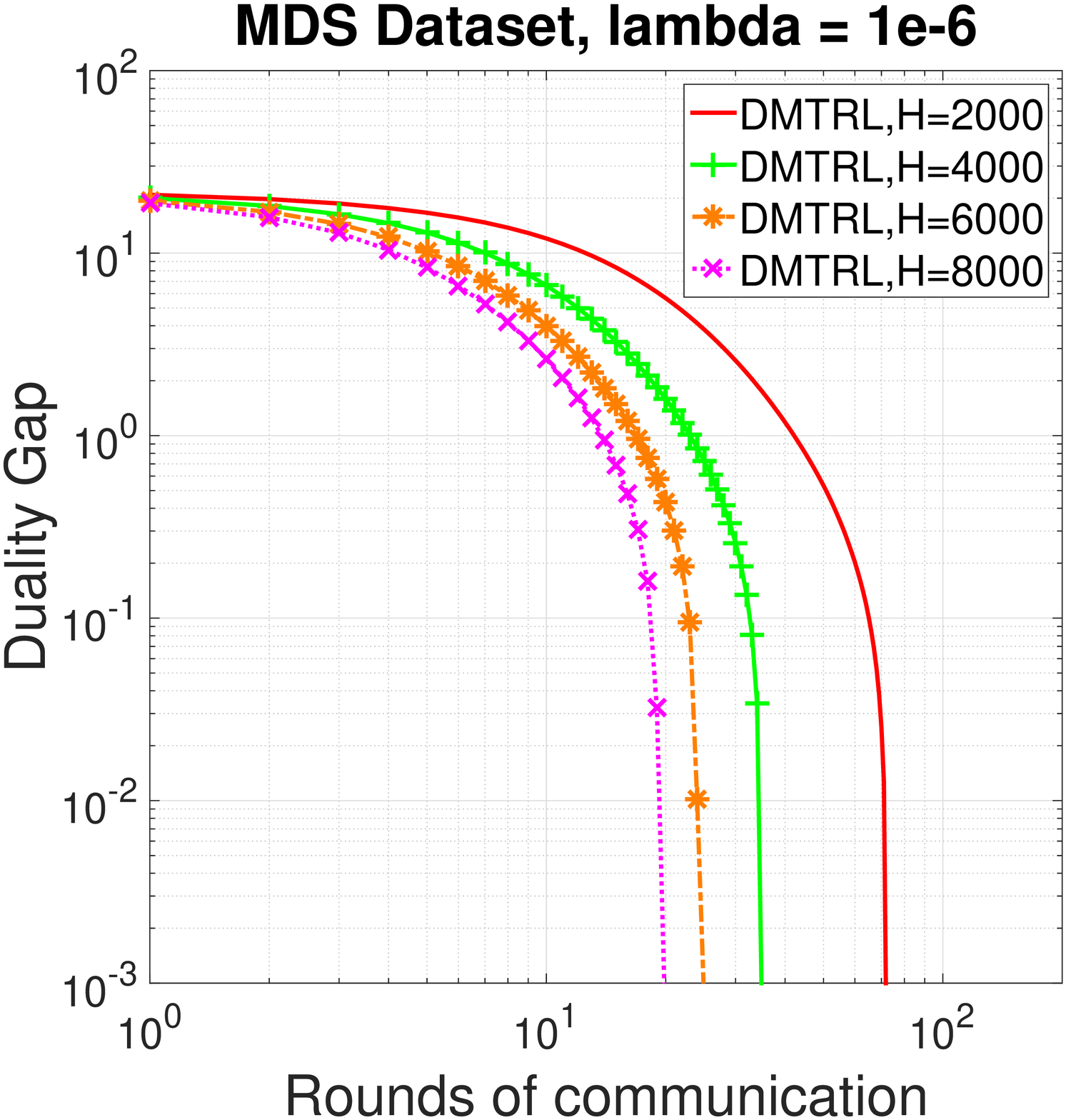}}
		\caption{Experimental results on real-world datasets}\label{fig:real_world}
	\end{center}
\end{figure}

The results from the tables and figures show that DMTRL converges to the same prediction error as Centralized MTRL. DMTRL outperforms STL significantly except on MNIST, which shows the advantage of DMTRL by leveraging related tasks to improve generalization performance. It is reasonable because in our experimental setting, each task in MNIST has around 12,000 instances for training, which is relatively large. As MNIST is a relatively easy learning task, such amount of training data is sufficient for STL to perform well. However this does not imply that DMTRL could not improve generalization performance when total amount of data is large. In the MDS case, there are in total 91,290 instances, with number of instances per task varying from 314 to 20,751. Experiment results show that performance of DMTRL outperforms STL by a significant amount. In this case, since some tasks do not have sufficient training data, DMTRL helps improve the prediction performance by leveraging task relationships in multi-task learning. We also note that for MDS, the number of instances of different tasks differs by a fairly large amount. This means with the same number of local SDCA iterations per task, the task with largest instance number will have a worse $\Theta$-approximate solution compared to others. Although convergence is still guaranteed, this hinders the overall primal-dual convergence performance. Thus it remains an open research issue on how to balance the data in each local worker to achieve better convergence for our future work.

Finally, we also report the experimental results of DMTRL about duality gap v.s. elapsed time and duality gap v.s. rounds of communication on the three real-world datasets in Figures~\ref{fig:dualGap_realworld}-\ref{fig:commu_real_world}. The observations from the figures are similar to those found on the synthetic datasets.

\section{Conclusion and Future Work}

In this paper, we present a novel distributed framework for multi-task relationship learning, denoted by DMTRL. With the proposed framework, data of different tasks can be geo-distributedly stored in local machines, and multiple tasks can be learned jointly without centralizing data of different tasks to a master machine. We provide theoretical convergence analysis for DMTRL with both smooth and non-smooth convex loss functions. To verify the effectiveness of DMTRL, we carefully design and conduct extensive experiments on both synthetic and real-world datasets to test the convergence and prediction accuracy of DMTRL in comparison with the baseline methods. In our future work, we aim to extend our framework to the setting that allows asynchronous communication. We also aim to conduct study on how to achieve better convergence when data are imbalanced over different tasks.

\section*{Acknowledgments}
This work is supported by NTU Singapore Nanyang Assistant Professorship (NAP) grant M4081532.020 and Singapore MOE AcRF Tier-2 grant MOE2016-T2-2-060.

\bibliography{arxivbib}
\bibliographystyle{plainnat}
\appendix
\onecolumn
\newpage
\Large
\textbf{Appendix}
\normalsize

\section{Technical Lemmas}\label{sec:lemmas}
\begin{lemma}\label{technical_L}
(Lemma 21 in~\citep{shalev2013stochastic}) Let $l: \mathbb{R} \rightarrow \mathbb{R} $ be an L-Lipschitz function. Then for any real value $\alpha$ s.t. $|\alpha| > L$ we have that ${l}^*(\alpha) = \infty$.\\
\end{lemma}

\begin{lemma}\label{technical_D_0}
For all $\boldsymbol{\alpha}$, $D(\boldsymbol{\alpha})\leq P(\mathbf{W}^*)\leq P(\mathbf{0}) \leq m$. And, $D(\mathbf{0}) \geq 0$.
\end{lemma}
\begin{proof}
The first inequality is due to weak duality theorem, the second inequality comes from the optimality of $\mathbf{W}^*$ and the third one is because of the assumption that $l_j^i(0) \leq 1$.

For the last inequality, we have $-{l_j^i}^*(0)= - \underset{z}{\text{max}}(0 - {l_j^i}(z)) = \underset{z}{\text{min}}\, {l_j^i}(z) \geq 0$.Therefore,
\begin{equation}
  D(\mathbf{0}) = - \sum_{i=1}^{m}\frac{1}{n_i}\sum_{j=1}^{n_i}{l_j^i}^*(0) \geq 0
\end{equation}
\end{proof}

\section{Technical Proofs}\label{sec:proofs}
Without loss of generality, to simplify the statements of the theorems used for convergence analysis, we have the following assumptions on loss functions in \eqref{formulation}:
\begin{center}
\begin{enumerate*}
  \item For all $(i,j)$,  $l_j^i(a) \geq 0$, $\forall a$, \;\; and \;\;
  \item For all $(i,j)$, $l_j^i(0) \leq 1$.
\end{enumerate*}
\end{center}

\textbf{Note}: The following proofs are for the \textbf{linear} kernel case.
Given an explicit feature mapping function $\phi()$, the proofs could be adapted to the kernelized version by just changing $\mathbf{x}_j^i$ to $\phi(\mathbf{x}_j^i)$.

\subsection{Proof of Theorem \ref{dual_form_theorem}}
\begin{proof}
Problem (\ref{formulation}) given fixed $\boldsymbol{\Omega}$ can be rewritten as:
\begin{equation*}
\begin{aligned}
& \underset{\mathbf{W,z}}{\text{min}}
& & \sum_{i=1}^{m} \frac{1}{n_i}\sum_{j=1}^{n_i}l_j^i(-z_j^i )+ \frac{\lambda}{2} tr(\mathbf{W}\boldsymbol{\Omega}\mathbf{W}^T)
\\
& \text{s.t.}
& & \mathbf{w}_i^T \mathbf{x}_j^i + z_j^i = 0\\
\label{eq:solvew2}
\end{aligned}
\end{equation*}
By introducing Lagrangian multipliers $-\frac{1}{n_i}\alpha_j^i $, we have the Lagrangian function defined as:
\begin{equation*}
\begin{aligned}
L(\mathbf{W},\mathbf{z},\boldsymbol{\alpha})
& = \sum_{i=1}^{m} \frac{1}{n_i}\sum_{j=1}^{n_i}l_j^i(-z_j^i )+ \frac{\lambda}{2} tr(\mathbf{W}\boldsymbol{\Omega}\mathbf{W}^T)+ \sum_{i=1}^{m}\sum_{j=1}^{n_i}-\frac{1}{n_i}\alpha_j^i(\mathbf{w}_i^T \mathbf{x}_j^i + z_j^i )
\label{eq:Lagrangian}
\end{aligned}
\end{equation*}

The Lagrangian dual function is defined to be:
\begin{equation*}
\begin{aligned}
g(\boldsymbol{\alpha}) =  \underset{\mathbf{W,z}}{\text{inf}}
L(\mathbf{W},\mathbf{z},\boldsymbol{\alpha})
& = \underset{\mathbf{W,z}}{\text{inf}}\ (\sum_{i=1}^{m} \frac{1}{n_i}\sum_{j=1}^{n_i}l_j^i(-z_j^i )+ \frac{\lambda}{2} tr(\mathbf{W}\boldsymbol{\Omega}\mathbf{W}^T)+ \sum_{i=1}^{m}\sum_{j=1}^{n_i}-\frac{1}{n_i}\alpha_j^i(\mathbf{w}_i^T \mathbf{x}_j^i + z_j^i ))
\\
& =
\sum_{i=1}^{m} \frac{1}{n_i}\sum_{j=1}^{n_i} \underset{z_j^i}{\text{inf}}\bigg(l_j^i(-z_j^i) - \alpha_j^i z_j^i) \bigg) +
\underset{\mathbf{W}}{\text{inf}} \bigg(\frac{\lambda}{2} tr(\mathbf{W}\boldsymbol{\Omega}\mathbf{W}^T) - \sum_{i=1}^{m}\sum_{j=1}^{n_i}\frac{1}{n_i}\alpha_j^i\mathbf{w}_i^T \mathbf{x}_j^i \bigg)
\\
& \stackrel{\eqref{solve_w_hat}}{=}\sum_{i=1}^{m} \frac{1}{n_i}\sum_{j=1}^{n_i} -{l_j^i}^*(-\alpha_j^i) -
\frac{\lambda}{2}tr(\hat{\mathbf{W}}\boldsymbol{\Omega}\hat{\mathbf{W}}^T)
\\
& \stackrel{\eqref{rewrite_regulizer}}{=} - \sum_{i=1}^{m} \frac{1}{n_i}\sum_{j=1}^{n_i} {l_j^i}^*(-\alpha_j^i) - \frac{1}{2\lambda} \boldsymbol{\alpha}^T \mathbf{K} \boldsymbol{\alpha}
\label{eq:Lagrangiandual}
\end{aligned}
\end{equation*}

where the second last equality comes from the fact that when the infinum takes at $\hat{\mathbf{W}}$, we have:
\begin{eqnarray}
\frac{d}{d\mathbf{w_i}}L
 &=& \lambda \sum_{k=1}^{m}\hat{\mathbf{w_k}}\boldsymbol{\Omega}_{ik}
- \sum_{j=1}^{n_i}\frac{1}{n_i}\alpha_j^i \mathbf{x}_j^i = 0 \label{solve_w_hat}\\
& \Leftrightarrow &
\lambda \sum_{k=1}^{m}{\hat{\mathbf{w}}}_i^T\hat{\mathbf{w_k}}\boldsymbol{\Omega}_{ik}
- \sum_{j=1}^{n_i}\frac{1}{n_i}\alpha_j^i {\hat{\mathbf{w}}}_i^T\mathbf{x}_j^i = 0  \label{w_hat}\\
& \Leftrightarrow &
\lambda \sum_{i=1}^{m}\sum_{k=1}^{m}{\hat{\mathbf{w}}}_i^T\hat{\mathbf{w_k}}\boldsymbol{\Omega}_{ik}
= \sum_{i=1}^{m}\sum_{j=1}^{n_i}\frac{1}{n_i}\alpha_j^i {\hat{\mathbf{w}}}_i^T\mathbf{x}_j^i \nonumber \\
& \Leftrightarrow &
\lambda tr(\hat{\mathbf{W}}^T\hat{\mathbf{W}}\boldsymbol{\Omega}) =
\lambda tr(\hat{\mathbf{W}}\boldsymbol{\Omega}\hat{\mathbf{W}}^T)
= \sum_{i=1}^{m}\sum_{j=1}^{n_i}\frac{1}{n_i}\alpha_j^i {\hat{\mathbf{w}}}_i^T\mathbf{x}_j^i \nonumber
\end{eqnarray}

Let's make $\mathbf{B}$ with its $i$th column being $\sum_{j=1}^{n_i}\frac{1}{n_i}\alpha_j^i \mathbf{x}_j^i$. From \eqref{solve_w_hat} we have:
\begin{equation}
\begin{aligned}
& &\lambda \hat{\mathbf{W}} \boldsymbol{\Omega} &= \mathbf{B} \\
& \Leftrightarrow & \hat{\mathbf{W}} &= \frac{1}{\lambda}\mathbf{B} \boldsymbol{\Omega}^{-1} \\
& \Leftrightarrow & \hat{\mathbf{W}}^T \hat{\mathbf{W}} \boldsymbol{\Omega} &= \frac{1}{\lambda^2} \boldsymbol{\Omega}^{-1}\mathbf{B}^T \mathbf{B} \\
& \Leftrightarrow & \lambda tr(\hat{\mathbf{W}}^T \hat{\mathbf{W}} \boldsymbol{\Omega}) &= \frac{1}{\lambda} tr(\boldsymbol{\Sigma}\mathbf{B}^T \mathbf{B})\\
& & &= \frac{1}{\lambda}\sum_{i=1}^{m} \sum_{i'=1}^{m} (\sum_{j=1}^{n_i} \frac{\alpha_j^i}{n_i} \mathbf{x}_j^i)^T(\sum_{j'=1}^{n_{I'}} \frac{\alpha_{j'}^{i'}}{n_{i'}} \mathbf{x}_{j'}^{i'}) \sigma_{ii'}\\
& & & = \frac{1}{\lambda} \boldsymbol{\alpha}^T \mathbf{K} \boldsymbol{\alpha}
\label{rewrite_regulizer}
\end{aligned}
\end{equation}

Next, we will show the primal-dual variable optimal point correspondence is given by
\begin{equation}
\mathbf{w}^*_i = \frac{1}{\lambda} \sum_{i'=1}^{m}\sum_{j'=1}^{n_{i'}}\frac{{\alpha_{j'}^{i'}}^*}{n_{i'}}{\mathbf{x}_{j'}^{i'}} \sigma_{ii'}
\end{equation}

Given optimal $\boldsymbol{\alpha}^*$ and $\mathbf{W}^*$ and strong duality that $P(\mathbf{W}^*) = D(\boldsymbol{\alpha}^*)$, we know that $\mathbf{W}^*$ minimizes $L(\mathbf{W},\boldsymbol{\alpha}^*)$. Then, from \eqref{w_hat} and \eqref{rewrite_regulizer}, we have that,
\begin{equation}
\begin{aligned}
\sum_{i=1}^{m}\sum_{j=1}^{n_i}\frac{{\alpha_j^i}^*}{n_i}{\mathbf{x}_j^i}^T {{\mathbf{w}^*}}_i & =
\lambda tr(\mathbf{W}^* \boldsymbol{\Omega}{\mathbf{W}^*}^T ) = \frac{1}{\lambda} \sum_{i=1}^{m}\sum_{j=1}^{n_i}\frac{{\alpha_j^i}^*}{n_i}{\mathbf{x}_j^i}^T \sum_{i'=1}^{m}\sum_{j'=1}^{n_{i'}}\frac{{\alpha_{j'}^{i'}}^*}{n_{i'}}{\mathbf{x}_{j'}^{i'}} \sigma_{ii'}
\label{w_hat_2}
\end{aligned}
\end{equation}
Note that when $\mathbf{w}^*_i = \frac{1}{\lambda} \sum_{i'=1}^{m}\sum_{j'=1}^{n_{i'}}\frac{{\alpha_{j'}^{i'}}^*}{n_{i'}}{\mathbf{x}_{j'}^{i'}} \sigma_{ii'}$,
equation \eqref{w_hat_2} holds. Since $\hat{\mathbf{W}}$ given $\mathbf{B}$ and $\mathbf{\Omega}$ is unique, we can get the coclusion that:
\begin{equation}
\begin{aligned}
\mathbf{w}^*_i = \frac{1}{\lambda} \sum_{i'=1}^{m}\sum_{j'=1}^{n_{i'}}\frac{{\alpha_{j'}^{i'}}^*}{n_{i'}}{\mathbf{x}_{j'}^{i'}} \sigma_{ii'}
\end{aligned}
\end{equation}
\end{proof}
\subsection{Proof of Lemma \ref{dual_lemma}}
\begin{proof}
We have that
\begin{equation*}
D(\boldsymbol{\alpha} + \eta \sum_{i=1}^{m} \Delta \boldsymbol{\alpha}_{[i]}) = \underbrace{ - \sum_{i=1}^{m} \frac{1}{n_i} \sum_{j=1}^{n_i} {l_j^i}^* (-\alpha_j^i - \eta \Delta \alpha_j^i )}_A \underbrace{-\frac{1}{2\lambda} (\boldsymbol{\alpha} + \eta \sum_{i=1}^{m} \Delta \boldsymbol{\alpha}_{[i]})^T \mathbf{K} (\boldsymbol{\alpha} + \eta \sum_{i=1}^{m} \Delta \boldsymbol{\alpha}_{[i]})}_B
\end{equation*}
Let us bound terms $A$ and $B$ seperately.
\begin{equation}\label{bound_A}
\begin{aligned}
A &= - \sum_{i=1}^{m} \frac{1}{n_i} \sum_{j=1}^{n_i} {l_j^i}^* (-(1-\eta)\alpha_j^i - \eta (\alpha_j^i +\Delta \alpha_j^i) )
& \geq - \sum_{i=1}^{m} \frac{1}{n_i} \sum_{j=1}^{n_i}  (1-\eta){l_j^i}^*(-\alpha_j^i) + \eta{l_j^i}^*(- (\alpha_j^i +\Delta \alpha_j^i) )
\end{aligned}
\end{equation}
The last inequality comes from Jensen's inequality.
Next we will bound $B$. Using the seperability measurement $\rho$ defined in \eqref{rhodef}.
\begin{equation}\label{bound_B}
\begin{aligned}
B &= - \frac{1}{2\lambda} \bigg( \boldsymbol{\alpha}^T\mathbf{K}\boldsymbol{\alpha} + 2\eta \sum_{i=1}^{m} \Delta\boldsymbol{\alpha}_{[i]}^T \mathbf{K} \boldsymbol{\alpha} + \eta^2( (\sum_{i=1}^{m}\Delta \boldsymbol{\alpha}_{[i]})^T \mathbf{K}(\sum_{i=1}^{m}\Delta \boldsymbol{\alpha}_{[i]})  ) \bigg) \\
& \geq - \frac{1}{2\lambda} \bigg( \boldsymbol{\alpha}^T\mathbf{K}\boldsymbol{\alpha} + 2\eta \sum_{i=1}^{m} \Delta\boldsymbol{\alpha}_{[i]}^T \mathbf{K} \boldsymbol{\alpha} + \eta \rho \sum_{i=1}^{m} \Delta\boldsymbol{\alpha}_{[i]}^T \mathbf{K} \Delta\boldsymbol{\alpha}_{[i]} \bigg)
\end{aligned}
\end{equation}
Combining \eqref{bound_A} and \eqref{bound_B}, we have
\begin{equation*}
\begin{aligned}
D(\boldsymbol{\alpha} + \eta \sum_{i=1}^{m} \Delta \boldsymbol{\alpha}_{[i]})  & \geq \underbrace{(1-\eta)\bigg( - \sum_{i=1}^{m} \frac{1}{n_i} \sum_{j=1}^{n_i}  {l_j^i}^*(-\alpha_j^i) - \frac{1}{2\lambda} \boldsymbol{\alpha}^T\mathbf{K}\boldsymbol{\alpha} \bigg)}_{(1-\eta)D(\boldsymbol{\alpha})} \\ & + \eta \sum_{i=1}^{m} \underbrace{ \bigg( \frac{1}{n_i} \sum_{j=1}^{n_i}{l_j^i}^*(- (\alpha_j^i +\Delta \alpha_j^i) ) - \frac{1}{2\lambda m} \boldsymbol{\alpha}^T\mathbf{K}\boldsymbol{\alpha} - \frac{1}{\lambda}\Delta\boldsymbol{\alpha}_{[i]}^T \mathbf{K} \boldsymbol{\alpha} - \frac{1}{2\lambda}\rho \Delta\boldsymbol{\alpha}_{[i]}^T \mathbf{K} \Delta\boldsymbol{\alpha}_{[i]} \bigg)}_{\mathcal{D}_i^\rho(\Delta\boldsymbol{\alpha}_{[i]};\mathbf{w}_i(\boldsymbol{\alpha}),{\boldsymbol{\alpha}}_{[i]})} \\
& \geq (1-\eta)D(\boldsymbol{\alpha}) + \eta \sum_{i=1}^{m} \mathcal{D}_i^\rho(\Delta\boldsymbol{\alpha}_{[i]};\mathbf{w}_i(\boldsymbol{\alpha}),{\boldsymbol{\alpha}}_{[i]})
\end{aligned}
\end{equation*}
\end{proof}

\subsection{Proof of Theorem \ref{local_convergence_smooth}}
\begin{proof}
When functions $l_j^i$ are $(1/\mu)$-smooth, it is well known that ${l_j^i}^*$ are $\mu$ strongly convex. Let us define function $F(\boldsymbol{\boldsymbol{\nu}}) = -\mathcal{D}_i^\rho(\boldsymbol{\nu}_{[i]};\mathbf{w}_i(\boldsymbol{\alpha}),{\boldsymbol{\alpha}}_{[i]}) $.
we have $F(\boldsymbol{\boldsymbol{\nu}}) = g(\boldsymbol{\nu})+f(\boldsymbol{\nu})$
 , with $g(\boldsymbol{\nu}) = \frac{1}{n_i}\sum_{j=1}^{n_i}{l_j^i}^*(-\alpha_j^i - \nu_j^i)$
and $f(\boldsymbol{\nu}) = \frac{1}{2\lambda m}{\boldsymbol{\alpha}}^T \mathbf{K} {\boldsymbol{\alpha}} + \frac{1}{n_i}\sum_{j=1}^{n_i} \nu_j^i {\mathbf{w}_i(\boldsymbol{\alpha})}^T \mathbf{x}_j^i + \frac{\rho \sigma_{ii}}{2\lambda {n_i}^2} {\|\mathbf{A} \boldsymbol{\nu}_{[i]} \|}^2$,where $\mathbf{A}$ is the data feature matrix that include datapoints from all tasks.

Since functions ${l_j^i}^*$ are $\mu$ strongly convex, $g(\boldsymbol{\nu})$ is strongly convex with convexity parameter $\frac{\mu}{n_i}$.
The gradient of function$f(\boldsymbol{\nu})$ is coordinate-wise Lipschitz continuous with Lipschitz constant $\frac{\rho \sigma_{ii}}{\lambda {n_i}^2}q_{\text{max}} $.
From the proof of Theorem 20 in~\citep{Richtarik2015}, we have that
\begin{eqnarray}
\mathbb{E}[\mathcal{D}_i^\rho(\Delta\boldsymbol{\alpha}_{[i]}^*;\mathbf{w}_i(\boldsymbol{\alpha}),{\boldsymbol{\alpha}}_{[i]}) \!\!\!\!&\!\! - \!\!&\!\!\!\!
   \mathcal{D}_i^\rho(\Delta\boldsymbol{\alpha}_{[i]}^{(h+1)};\mathbf{w}_i(\boldsymbol{\alpha}),{\boldsymbol{\alpha}}_{[i]})] \nonumber \\
  \!\!\!\!&\!\! \leq \!\!&\!\!\!\! \bigg( 1 - \frac{1}{n_i}\frac{\frac{\mu\lambda n_i}{\rho \sigma_{ii} q_{\text{max}}}}{ {1+\frac{\mu \lambda n_i}{\rho \sigma_{ii} q_{\text{max}}}} } \bigg) \big( \mathcal{D}_i^\rho(\Delta\boldsymbol{\alpha}_{[i]}^*;\mathbf{w}_i(\boldsymbol{\alpha}),{\boldsymbol{\alpha}}_{[i]})-
   \mathcal{D}_i^\rho(\Delta\boldsymbol{\alpha}_{[i]}^{(h)};\mathbf{w}_i(\boldsymbol{\alpha}),{\boldsymbol{\alpha}}_{[i]}) \big) \nonumber \\
  \!\!\!\!&\!\!   =  \!\!&\!\!\!\! \bigg( 1- \frac{1}{n_i}\frac{\mu\lambda n_i}{\rho\sigma_{ii} q_{\text{max}}+\mu\lambda n_i}\bigg)
   \big( \mathcal{D}_i^\rho(\Delta\boldsymbol{\alpha}_{[i]}^*;\mathbf{w}_i(\boldsymbol{\alpha}),{\boldsymbol{\alpha}}_{[i]})-
   \mathcal{D}_i^\rho(\Delta\boldsymbol{\alpha}_{[i]}^{(h)};\mathbf{w}_i(\boldsymbol{\alpha}),{\boldsymbol{\alpha}}_{[i]}) \big) \nonumber
\end{eqnarray}

Therefore, we obtain that
\begin{eqnarray}
  \mathbb{E}[\mathcal{D}_i^\rho(\Delta\boldsymbol{\alpha}_{[i]}^*;\mathbf{w}_i(\boldsymbol{\alpha}),{\boldsymbol{\alpha}}_{[i]}) \!\!\!\!&\!\! - \!\!&\!\!\!\!
   \mathcal{D}_i^\rho(\Delta\boldsymbol{\alpha}_{[i]}^{(h+1)};\mathbf{w}_i(\boldsymbol{\alpha}),{\boldsymbol{\alpha}}_{[i]})] \nonumber \\
   \!\!\!\!&\!\! \leq \!\!&\!\!\!\!  \bigg( 1- \frac{1}{n_i}\frac{\mu\lambda n_i}{\rho\sigma_{ii} q_{\text{max}}+\mu\lambda n_i}\bigg)^h
   \big( \mathcal{D}_i^\rho(\Delta\boldsymbol{\alpha}_{[i]}^*;\mathbf{w}_i(\boldsymbol{\alpha}),{\boldsymbol{\alpha}}_{[i]})-
   \mathcal{D}_i^\rho(\mathbf{0};\mathbf{w}_i(\boldsymbol{\alpha}),{\boldsymbol{\alpha}}_{[i]}) \big) \nonumber
\end{eqnarray}
If $H$ is chosen to be $H \geq \text{log}(\frac{1}{\Theta}) \frac{\rho\sigma_{ii} q_{\text{max}}+\mu\lambda n_i}{\mu\lambda }$, we will have $\bigg( 1- \frac{1}{n_i}\frac{\mu\lambda n_i}{\rho\sigma_{ii} q_{\text{max}}+\mu\lambda n_i}\bigg)^H \leq \Theta$ to complete the proof.
\end{proof}

\subsection{Proof of Theorem \ref{local_convergence_lipschitz}}
\begin{proof}
Similar to proof of Theorem \eqref{local_convergence_smooth}, $F(\boldsymbol{\boldsymbol{\nu}})= g(\boldsymbol{\nu})+f(\boldsymbol{\nu})$ is defined. In this case, functions ${l_j^i}^*$ are not guaranteed to be strongly convex. $f(\boldsymbol{\nu})$ is still coordinate-wise Lipschitz continuous with Lipschitz constant $\frac{\rho \sigma_{ii}}{\lambda {n_i}^2}q_{\text{max}} $. From Theorem 3 in~\citep{tappenden2015complexity}, we have
\begin{eqnarray}
   \mathbb{E}[\mathcal{D}_i^\rho(\Delta\boldsymbol{\alpha}_{[i]}^*;\mathbf{w}_i(\boldsymbol{\alpha}),{\boldsymbol{\alpha}}_{[i]}) \!\!\!\!&\!\! - \!\!&\!\!\!\!
   \mathcal{D}_i^\rho(\Delta\boldsymbol{\alpha}_{[i]}^{(h)};\mathbf{w}_i(\boldsymbol{\alpha}),{\boldsymbol{\alpha}}_{[i]})] \nonumber \\
   \!\!\!\!&\!\! \leq \!\!&\!\!\!\! \frac{n_i}{n_i+h} \bigg(\mathcal{D}_i^\rho(\Delta\boldsymbol{\alpha}_{[i]}^*;\mathbf{w}_i(\boldsymbol{\alpha}),{\boldsymbol{\alpha}}_{[i]})-
   \mathcal{D}_i^\rho(\mathbf{0};\mathbf{w}_i(\boldsymbol{\alpha}),{\boldsymbol{\alpha}}_{[i]})+ \frac{1}{2} \frac{\rho \sigma_{ii}}{\lambda {n_i}^2}q_{\text{max}} \| \Delta\boldsymbol{\alpha}_{[i]}^* \|^2 \bigg) \nonumber
\end{eqnarray}
If $H \geq n_i\bigg( \frac{1-\Theta}{\Theta} + \frac{\rho \sigma_{ii}q_{\text{max}} \| \Delta\boldsymbol{\alpha}_{[i]}^* \|^2}{2\Theta \lambda {n_i}^2 \big(\mathcal{D}_i^\rho(\Delta\boldsymbol{\alpha}_{[i]}^*;.)-\mathcal{D}_i^\rho(\mathbf{0};.)\big)} \bigg)$, we will have Assumption \eqref{local_approx_assumption} holds as desired.
\end{proof}

\subsection{Proof of Lemma \ref{conv_lemma}}
\begin{proof}
Following the derivation of proof of \textbf{Lemma 5} in~\citep{MaSJJRT15}, we arrive at the inequality
\begin{equation*}
\begin{aligned}
\mathbb{E}(D(\boldsymbol{\alpha})-D(\boldsymbol{\alpha} + \eta \sum_{i=1}^{m} \Delta \boldsymbol{\alpha}_{[i]}) ) \leq \eta(1-\Theta)\underbrace{\bigg( D(\boldsymbol{\alpha}) - \sum_{i=1}^{m} \mathcal{D}_i^\rho(\Delta\boldsymbol{\alpha}_{[i]}^*;\mathbf{w}_i(\boldsymbol{\alpha}),{\boldsymbol{\alpha}}_{[i]} \bigg)}_C
\end{aligned}
\end{equation*}
where
\begin{equation} \label{C_upp_bound}
\begin{aligned}
C =& - \sum_{i=1}^{m} \frac{1}{n_i} \sum_{j=1}^{n_i}  {l_j^i}^*(-\alpha_j^i) \cancel{ - \frac{1}{2\lambda} \boldsymbol{\alpha}^T\mathbf{K}\boldsymbol{\alpha}}\\
& +\sum_{i=1}^m\bigg( \frac{1}{n_i}\sum_{j=1}^{n_i}{l_j^i}^*(-\alpha_j^i - {\Delta\alpha_j^i}^*) \cancel{+ \frac{1}{2\lambda m}{\boldsymbol{\alpha}}^T \mathbf{K} {\boldsymbol{\alpha}}} + \frac{1}{n_i}\sum_{j=1}^{n_i} {\Delta\alpha_j^i}^* {\mathbf{w}_i(\boldsymbol{\alpha})}^T \mathbf{x}_j^i +  \frac{\rho}{2\lambda}  {\Delta\boldsymbol{\alpha}_{[i]}^*}^T \mathbf{K} \Delta\boldsymbol{\alpha}_{[i]}^* \bigg) \\
 \leq & \sum_{i=1}^m\bigg( \frac{1}{n_i}\sum_{j=1}^{n_i}{l_j^i}^*(-\alpha_j^i - s(u_j^i - \alpha_j^i)) - {l_j^i}^*(-\alpha_j^i) \bigg) \\
& + \sum_{i=1}^m \bigg(  \frac{1}{n_i}\sum_{j=1}^{n_i} s(u_j^i - \alpha_j^i) {\mathbf{w}_i(\boldsymbol{\alpha})}^T \mathbf{x}_j^i + \frac{\rho}{2\lambda}  {s^2(\mathbf{u} - \boldsymbol{\alpha})_{[i]}}^T \mathbf{K} {(\mathbf{u} - \boldsymbol{\alpha})_{[i]}} \bigg)\\
 \leq &\sum_{i=1}^m\bigg( \frac{1}{n_i}\sum_{j=1}^{n_i} s {l_j^i}^*(-u_j^i) +(1-s) {l_j^i}^*(-\alpha_j^i) -\frac{\mu}{2}(1-s)s(u_j^i-\alpha_j^i)^2 -{l_j^i}^*(-\alpha_j^i) + s(u_j^i - \alpha_j^i) {\mathbf{w}_i(\boldsymbol{\alpha})}^T \mathbf{x}_j^i \bigg)  \\&+  \sum_{i=1}^m \frac{\rho}{2\lambda}  {s^2(\mathbf{u} - \boldsymbol{\alpha})_{[i]}}^T \mathbf{K} {(\mathbf{u} - \boldsymbol{\alpha})_{[i]}}\\
 = &\sum_{i=1}^m \bigg( \frac{1}{n_i} \sum_{j=1}^{n_i}  s{l_j^i}^*(-u_j^i) +s u_j^i {\mathbf{w}_i (\boldsymbol{\alpha})}^T \mathbf{x}_j^i  - s {l_j^i}^*(-\alpha_j^i) - s \alpha_j^i {\mathbf{w}_i(\boldsymbol{\alpha})}^T \mathbf{x}_j^i -\frac{\mu}{2}(1-s)s(u_j^i-\alpha_j^i)^2 \bigg) \\
&+  \sum_{i=1}^m \frac{\rho}{2\lambda}  {s^2(\mathbf{u} - \boldsymbol{\alpha})_{[i]}}^T \mathbf{K} {(\mathbf{u} - \boldsymbol{\alpha})_{[i]}}
\end{aligned}
\end{equation}
From the definition of convex conjugate, we have
\begin{equation}\label{conv_conj}
  {l_j^i}^*(-u_j^i) = - u_j^i {\mathbf{w}_i(\boldsymbol{\alpha})}^T \mathbf{x}_j^i - {l_j^i}({\mathbf{w}_i(\boldsymbol{\alpha})}^T \mathbf{x}_j^i)
\end{equation}
And we could write the duality gap as
\begin{equation}\label{dual_gap}
\begin{aligned}
G(\boldsymbol{\alpha}) := P(\mathbf{W}(\boldsymbol{\alpha})) - D(\boldsymbol{\alpha}) &=
\sum_{i=1}^{m} \frac{1}{n_i}({l_j^i}({\mathbf{w}_i(\boldsymbol{\alpha})}^T \mathbf{x}_j^i)  + {l_j^i}^*(-\alpha_j^i) ) + \lambda tr(\mathbf{W}(\boldsymbol{\alpha}) \boldsymbol{\Omega} \mathbf{W}(\boldsymbol{\alpha})^T) \\
& = \sum_{i=1}^{m} \frac{1}{n_i}({l_j^i}({\mathbf{w}_i(\boldsymbol{\alpha})}^T \mathbf{x}_j^i)  + {l_j^i}^*(-\alpha_j^i) +\alpha_j^i {\mathbf{w}_i(\boldsymbol{\alpha})}^T \mathbf{x}_j^i )
\end{aligned}
\end{equation}
Plugging \eqref{conv_conj} and \eqref{dual_gap} into \eqref{C_upp_bound}, we have
\begin{equation}\label{C_upp_boundtwo}
\begin{aligned}
C  \leq &\sum_{i=1}^m \bigg( \frac{1}{n_i} \sum_{j=1}^{n_i} - s ( {l_j^i}({\mathbf{w}_i(\boldsymbol{\alpha})}^T \mathbf{x}_j^i) + {l_j^i}^*(-\alpha_j^i) + \alpha_j^i {\mathbf{w}_i(\boldsymbol{\alpha})}^T \mathbf{x}_j^i) -\frac{\mu}{2}(1-s)s(u_j^i-\alpha_j^i)^2 \bigg) \\ &+  \sum_{i=1}^m \frac{\rho \sigma_{ii}}{2\lambda {n_i}^2} {\|\mathbf{A} {s(\mathbf{u} - \boldsymbol{\alpha})_{[i]}} \|}^2 \\
= & -s G(\boldsymbol{\alpha}) -\sum_{i=1}^m  \frac{1}{n_i} \sum_{j=1}^{n_i} \frac{\mu}{2}(1-s)s(u_j^i-\alpha_j^i)^2  +  \sum_{i=1}^m \frac{\rho}{2\lambda}  {s^2(\mathbf{u} - \boldsymbol{\alpha})_{[i]}}^T \mathbf{K} {(\mathbf{u} - \boldsymbol{\alpha})_{[i]}}
\end{aligned}
\end{equation}
\end{proof}

\subsection{Proof of Lemma \ref{Q_upp_bound}}
\begin{proof}
The strong convexity parameter for general convex functions is $\mu = 0$. Therefore, the $Q^{(t)}$ becomes
\begin{equation*}
Q^{(t)} \stackrel{Q^{(t)} \textnormal{ in Lemma \ref{conv_lemma}}}{=\joinrel=\joinrel=}  \sum_{i=1}^{m}(\mathbf{u}_{[i]}^{(t)} - \boldsymbol{\alpha}_{[i]}^{(t)})^T \mathbf{K} (\mathbf{u}_{[i]}^{(t)} - \boldsymbol{\alpha}_{[i]}^{(t)}) \stackrel{\pi_i \textnormal{ in Lemma \ref{Q_upp_bound}}}{\leq} \sum_{i=1}^{m}\pi_i \|\mathbf{u}_{[i]}^{(t)} - \boldsymbol{\alpha}_{[i]}^{(t)} \|^2
\end{equation*}
Using \eqref{technical_L}, we know that $|{\alpha_j^i}^{(t)}| \leq L$ and we have $|{u_j^i}^{(t)}| \leq L$ because ${u_j^i}^{(t)}$ is an sub-gradient of the L-Lipschitz function. We could therefore bound $Q^{(t)} $
\begin{equation*}
Q^{(t)} \leq 4L^2\sum_{i=1}^{m} \pi_i n_i
\end{equation*}

When all data $\mathbf{x}_j^i$ are normalized such that $\|\mathbf{x}_j^i \|^2  \leq 1$, we have
\begin{equation*}
\pi_i := \underset{\boldsymbol{\alpha}_{[i]} \in \mathbb{R}^{n_i}}{\text{max}} \frac{\|\boldsymbol{\alpha}_{[i]}^T \mathbf{K} \boldsymbol{\alpha}_{[i]}\|^2}{\|\boldsymbol{\alpha}_{[i]}\|^2} = \underset{\boldsymbol{\alpha}_{[i]} \in \mathbb{R}^{n_i}}{\text{max}} \frac{\sigma_{ii}}{{n_i}^2} \frac{\|\mathbf{A}_{[i]}\boldsymbol{\alpha}_{[i]} \|^2}{\|\boldsymbol{\alpha}_{[i]} \|^2} = \frac{\sigma_{ii}}{{n_i}^2}\| \mathbf{A}_{[i]}\|_2^2 \leq \frac{\sigma_{ii}}{{n_i}^2} \| \mathbf{A}_{[i]}\|_F^2 \leq \frac{\sigma_{ii}}{{n_i}^2} n_i = \frac{\sigma_{ii}}{n_i}
\end{equation*}
where $\mathbf{A}_{[i]}$ is the data matrix for task $i$:
\begin{equation*}
  \mathbf{A}_{[i]} = \begin{bmatrix}
                       \mathbf{x}_1^i & \dots & \mathbf{x}_{n_i}^i
                     \end{bmatrix}
\end{equation*}
\end{proof}

\subsection{Proof of Lemma \ref{rho_upp_bound}}
\begin{proof}
\begin{equation*}\label{alpha_K_alpha}
\begin{aligned}
{\boldsymbol{\alpha}}^T \mathbf{K} \boldsymbol{\alpha} &= \sum_{i=1}^{m} {{\boldsymbol{\alpha}}_{[i]}^T \mathbf{K} \sum_{i'=1}^{m} \boldsymbol{\alpha}_{[i']}}\\
& = \sum_{i=1}^{m} \sum_{i'=1}^{m} \boldsymbol{\alpha}_{[i]}^T \mathbf{K}\boldsymbol{\alpha}_{[i']}\\
& = \sum_{i=1}^{m} \sum_{i'=1}^{m} \sigma_{ii'}  \langle \frac{1}{n_i} \sum_{i=1}^{n_i} \alpha_j^i \mathbf{x}_j^i, \frac{1}{n_{i'}}  \sum_{i'=1}^{n_{i'}} \alpha_j^{i'} \mathbf{x}_j^{i'} \rangle \\
& \leq \sum_{i=1}^{m} \sum_{i'=1}^{m} \frac{1}{2} |\sigma_{ii'}| \bigg( \frac{1}{{n_i}^2} \lVert \sum_{i=1}^{n_i} \alpha_j^i \mathbf{x}_j^i\rVert^2    +\frac{1}{{n_{i'}}^2} \lVert \sum_{i'=1}^{n_{i'}} \alpha_j^{i'} \mathbf{x}_j^{i'}\rVert^2      \bigg) \\
& = \sum_{i=1}^{m} \sum_{i'=1}^{m} \frac{1}{2} \bigg( \frac{|\sigma_{ii'}|}{\sigma_{ii}} \boldsymbol{\alpha}_{[i]}^T \mathbf{K}\boldsymbol{\alpha}_{[i]} + \frac{|\sigma_{ii'}|}{\sigma_{i'i'}} \boldsymbol{\alpha}_{[i']}^T \mathbf{K}\boldsymbol{\alpha}_{[i']} \bigg)\\
& = \sum_{i=1}^{m} \sum_{i'=1}^{m} \frac{|\sigma_{ii'}|}{\sigma_{ii}} \boldsymbol{\alpha}_{[i]}^T \mathbf{K}\boldsymbol{\alpha}_{[i]}
\end{aligned}
\end{equation*}
It follows that
\begin{equation*}
  {\text{max}} \frac{{\boldsymbol{\alpha}}^T \mathbf{K} \boldsymbol{\alpha}} {\sum_{i=1}^{m} {\boldsymbol{\alpha}}_{[i]}^T \mathbf{K} \boldsymbol{\alpha}_{[i]}} \leq \underset{i}{\text{max}} \sum_{i'=1}^{m} \frac{|\sigma_{ii'}|}{\sigma_{ii}}.
\end{equation*}
\end{proof}

\subsection{Proof of Theorem \ref{convergence_smooth}}
\begin{proof}
If function ${l_j^i}(\cdot)$ is $(\frac{1}{\mu})$-smooth then ${l_j^i}^*(\cdot)$ is $\mu$-strongly convex. From $Q^{(t)}$ in Lemma~\ref{conv_lemma},
\begin{equation}\label{Q_t_smooth}
\begin{aligned}
  Q^{(t)}:&= - \frac{\lambda\mu(1-s)}{\rho s}\sum_{i=1}^{m}\frac{1}{n_i}\|\mathbf{u}_{[i]}^{(t)} - \boldsymbol{\alpha}_{[i]}^{(t)} \|^2 + \sum_{i=1}^{m} (\mathbf{u}_{[i]}^{(t)} - \boldsymbol{\alpha}_{[i]}^{(t)})^T \mathbf{K} (\mathbf{u}_{[i]}^{(t)} - \boldsymbol{\alpha}_{[i]}^{(t)}) \\
  & \leq - \frac{\lambda\mu(1-s)}{\rho s}\sum_{i=1}^{m}\frac{1}{n_i}\|\mathbf{u}_{[i]}^{(t)} - \boldsymbol{\alpha}_{[i]}^{(t)} \|^2 + \sum_{i=1}^{m} \pi_i \|\mathbf{u}_{[i]}^{(t)} - \boldsymbol{\alpha}_{[i]}^{(t)}\|^2 \\
  & \leq \sum_{i=1}^{m} (- \frac{\lambda\mu(1-s)}{\rho s n_i} + \pi_i)\|\mathbf{u}_{[i]}^{(t)} - \boldsymbol{\alpha}_{[i]}^{(t)} \|^2 \\
  & \leq  (- \frac{\lambda\mu(1-s)}{\rho s n_{i^*}} + \pi_{i^*})\|\mathbf{u}^{(t)} - \boldsymbol{\alpha}^{(t)} \|^2
\end{aligned}
\end{equation}
where $i^* = \underset{i}{\text{argmax}} - \frac{\lambda\mu(1-s)}{\rho s n_i} + \pi_i$.
Let $s = \frac{\lambda\mu}{\lambda \mu + \rho n_{i^*} \pi_{i^*}} \in [0,1]$, we have that $Q^{(t)} \leq 0$, $\forall t$.
Now, \eqref{convergence_lemma_ineq} becomes
\begin{equation}\label{dual_improv_expect}
  \mathbb{E}(D(\boldsymbol{\alpha}^{(t+1)})-D(\boldsymbol{\alpha}^{(t)}) ) \geq \eta(1-\Theta) \frac{\lambda\mu}{\lambda \mu + \rho n_{i^*} \pi_{i^*}}G(\boldsymbol{\alpha}^{(t)}) \geq  \eta(1-\Theta) \frac{\lambda\mu}{\lambda \mu + \rho n_{i^*} \pi_{i^*}}(D(\boldsymbol{\alpha}^*) - D(\boldsymbol{\alpha}^{(t)}))
\end{equation}
\begin{eqnarray*}
   \mathbb{E}(D(\boldsymbol{\alpha}^{(t+1)})-D(\boldsymbol{\alpha}^{*})+ D(\boldsymbol{\alpha}^{*})-D(\boldsymbol{\alpha}^{(t)}) )  &\geq&  \eta(1-\Theta) \frac{\lambda\mu}{\lambda \mu + \rho n_{i^*} \pi_{i^*}}(D(\boldsymbol{\alpha}^*) - D(\boldsymbol{\alpha}^{(t)})) \\
   \mathbb{E}(D(\boldsymbol{\alpha}^{(t+1)})-D(\boldsymbol{\alpha}^{*})+ D(\boldsymbol{\alpha}^{*})-D(\boldsymbol{\alpha}^{(t)}))   &\geq&  \eta(1-\Theta) \frac{\lambda\mu}{\lambda \mu + \rho n_{i^*} \pi_{i^*}}(D(\boldsymbol{\alpha}^*) - D(\boldsymbol{\alpha}^{(t)})) \\
   \mathbb{E}(D(\boldsymbol{\alpha}^{*})-D(\boldsymbol{\alpha}^{(t+1)}))
   &\leq&  (1- \eta(1-\Theta) \frac{\lambda\mu}{\lambda \mu + \rho n_{i^*} \pi_{i^*}})(D(\boldsymbol{\alpha}^*) - D(\boldsymbol{\alpha}^{(t)}))\\
\end{eqnarray*}
\begin{equation*}
   \mathbb{E}(D(\boldsymbol{\alpha}^{*})-D(\boldsymbol{\alpha}^{(t+1)}))
    \leq (1- \eta(1-\Theta) \frac{\lambda\mu}{\lambda \mu + \rho n_{i^*} \pi_{i^*}})^t(D(\boldsymbol{\alpha}^*) - D(\boldsymbol{\alpha}^{(0)}))
    \leq m \text{exp}( -t\eta(1-\Theta) \frac{\lambda\mu}{\lambda \mu + \rho n_{i^*} \pi_{i^*}})
\end{equation*}
We have $ \mathbb{E}(D(\boldsymbol{\alpha}^{*})-D(\boldsymbol{\alpha}^{(t+1)})) \leq \epsilon_D$ when
\begin{equation*}
  t \geq \frac{1}{\eta(1-\Theta)}\frac{\lambda \mu + \rho n_{i^*} \pi_{i^*}}{\lambda\mu} \text{log}\frac{m}{\epsilon_D}
\end{equation*}
From \eqref{dual_improv_expect}, we have
\begin{eqnarray*}
    \eta(1-\Theta) \frac{\lambda\mu}{\lambda \mu + \rho n_{i^*} \pi_{i^*}}G(\boldsymbol{\alpha}^{(t)}) &\leq&\mathbb{E}(D(\boldsymbol{\alpha}^{(t+1)})-D(\boldsymbol{\alpha}^{(t)}) ) \leq \mathbb{E}(D(\boldsymbol{\alpha}^{*})-D(\boldsymbol{\alpha}^{(t)}) )  \\
  G(\boldsymbol{\alpha}^{(t)}) &\leq& \frac{1 }{\eta(1-\Theta)} \frac{\lambda \mu + \rho n_{i^*} \pi_{i^*}}{\lambda\mu}\mathbb{E}(D(\boldsymbol{\alpha}^{*})-D(\boldsymbol{\alpha}^{(t)}) )\\
\end{eqnarray*}
For $ G(\boldsymbol{\alpha}^{(t)}) \leq \epsilon_G$, we will need $\epsilon_D \leq \eta(1-\Theta) \frac{\lambda\mu}{\lambda \mu + \rho n_{i^*} \pi_{i^*}} \epsilon_G $. Hence, after
\begin{equation*}
  t \geq \frac{1}{\eta(1-\Theta)}\frac{\lambda \mu + \rho n_{i^*} \pi_{i^*}}{\lambda\mu} \text{log}\bigg( \frac{m }{\eta(1-\Theta)} \frac{\lambda \mu + \rho n_{i^*} \pi_{i^*}}{\lambda\mu} \frac{1}{\epsilon_G} \bigg)
\end{equation*}
iterations, duality gap will be obtained to be less than $\epsilon_G$.
\end{proof}

\subsection{Proof of Theorem \ref{convergence_general}}
\begin{proof}
Most part of the proof is adapted from proof in~\citep{MaSJJRT15}.
\textit{Proof:}
Let $Q = \text{max}_t Q^{(t)}$.
From Lemma \ref{conv_lemma}, we have
\begin{equation}
\begin{aligned}\label{inequality_main_theorem3}
  \mathbb{E}[D(\boldsymbol{\alpha}^*) -D(\boldsymbol{\alpha}^{(t+1)}]=& \mathbb{E}[D(\boldsymbol{\alpha}^*) -D(\boldsymbol{\alpha}^{(t+1)})+D(\boldsymbol{\alpha}^{(t)})-D(\boldsymbol{\alpha}^{(t)})] \\
  \leq& D(\boldsymbol{\alpha}^*) -D(\boldsymbol{\alpha}^{(t)}) - \eta(1-\Theta) (sG(\boldsymbol{\alpha}^{(t)})-\frac{\rho}{2\lambda}s^2 Q^{(t)} ) \\
  =&  D(\boldsymbol{\alpha}^*) -D(\boldsymbol{\alpha}^{(t)}) - \eta(1-\Theta)s(P(\mathbf{W}(\boldsymbol{\alpha}^{t}))-D(\boldsymbol{\alpha}^{(t)}))+\eta(1-\Theta)\frac{\rho}{2\lambda}s^2 Q^{(t)}\\
   \leq&  D(\boldsymbol{\alpha}^*) -D(\boldsymbol{\alpha}^{(t)}) - \eta(1-\Theta)s(D(\boldsymbol{\alpha}^{(*)})-D(\boldsymbol{\alpha}^{(t)})) +\eta(1-\Theta)\frac{\rho}{2\lambda}s^2 Q^{(t)}\\
   \leq& (1- \eta(1-\Theta)s)(D(\boldsymbol{\alpha}^{(*)})-D(\boldsymbol{\alpha}^{(t)})) +\eta(1-\Theta)\frac{\rho}{2\lambda}s^2 Q
\end{aligned}
\end{equation}
From the above inequality, we have
\begin{equation*}
\begin{aligned}
\mathbb{E}[D(\boldsymbol{\alpha}^*) -D(\boldsymbol{\alpha}^{(t)}] \leq & (1- \eta(1-\Theta)s)^t(D(\boldsymbol{\alpha}^{(*)})-D(\boldsymbol{\alpha}^{(0)}))+\eta(1-\Theta)\frac{\rho}{2\lambda}s^2 Q \sum_{m=0}^{t-1} (1-\Theta)s)^m\\
 = &(1- \eta(1-\Theta)s)^t(D(\boldsymbol{\alpha}^{(*)})-D(\boldsymbol{\alpha}^{(0)}))+\eta(1-\Theta)\frac{\rho}{2\lambda}s^2 Q \frac{1-(1- \eta(1-\Theta)s)^t}{\eta(1-\Theta)s}\\
 \leq &(1- \eta(1-\Theta)s)^t(D(\boldsymbol{\alpha}^{(*)})-D(\boldsymbol{\alpha}^{(0)})) + \frac{\rho}{2\lambda}s Q
\end{aligned}
\end{equation*}
By choosing $s = 1$ and $t = t_0 := \text{max}\bigg(0,\ceil[\Big] {\frac{1}{\eta(1-\Theta)}\text{log} \bigg( \frac{2\lambda (D(\boldsymbol{\alpha}^*)-D(\boldsymbol{\alpha}^{(0)}))}{4L^2\pi\rho} \bigg)} \bigg)$, we will have
\begin{equation}\label{inequality_main_theorem1}
\begin{aligned}
\mathbb{E}[D(\boldsymbol{\alpha}^*) -D(\boldsymbol{\alpha}^{(t)}] \leq \frac{\rho}{2\lambda} Q + \frac{\rho}{2\lambda}Q  = \frac{Q\rho}{\lambda}
\end{aligned}
\end{equation}
Next, we will show by induction that
\begin{equation}\label{inequality_main_theorem2}
\forall t \geq t_0, \mathbb{E}[D(\boldsymbol{\alpha}^*) -D(\boldsymbol{\alpha}^{(t)}] \leq \frac{Q\rho}{\lambda(1+ \frac{\eta}{2}(1-\Theta)(t-t_0))}
\end{equation}
Clearly \eqref{inequality_main_theorem1} shows \eqref{inequality_main_theorem2} holds when $t = t_0$.

Now assume \eqref{inequality_main_theorem2} holds for $t = t'$, then for $t = t'+1$, from \eqref{inequality_main_theorem3} using $s = \frac{1}{1+\frac{\eta}{2}(1-\Theta)(t-t_0)} \in [0,1]$, we have
\begin{equation*}
\begin{aligned}
 \mathbb{E}[D(\boldsymbol{\alpha}^*) -D(\boldsymbol{\alpha}^{(t+1)}]
 \leq& (1- \eta(1-\Theta)s)(D(\boldsymbol{\alpha}^{(*)})-D(\boldsymbol{\alpha}^{(t)})) +\eta(1-\Theta)\frac{\rho}{2\lambda}s^2 Q\\
 \stackrel{\eqref{inequality_main_theorem2}}{\leq} & (1- \eta(1-\Theta)s)\frac{4L^2 \pi\rho}{\lambda(1+ \frac{\eta}{2}(1-\Theta)(t-t_0))} +\eta(1-\Theta)\frac{\rho}{2\lambda}s^2 Q\\
 \stackrel{\eqref{inequality_main_theorem3}}{\leq} & \frac{Q\rho}{\lambda} \frac{1+\frac{\eta}{2}(1-\Theta)(t-t_0-1)}{(1+\frac{\eta}{2}(1-\Theta)(t-t_0))^2}\\
 = & \frac{Q\rho}{\lambda} \frac{1}{1+\frac{\eta}{2}(1-\Theta)(t-t_0+1)} \frac{(1+\frac{\eta}{2}(1-\Theta)(t-t_0-1))(1+\frac{\eta}{2}(1-\Theta)(t-t_0+1))}{(1+\frac{\eta}{2}(1-\Theta)(t-t_0))^2}\\
 \leq &\frac{Q\rho}{\lambda} \frac{1}{1+\frac{\eta}{2}(1-\Theta)(t-t_0+1)}
\end{aligned}
\end{equation*}

Next, given $\bar{\boldsymbol{\alpha}} = \frac{1}{T-T_0}\sum_{t=T_0}^{T-1} \boldsymbol{\alpha}^{(t)}$, we have that
\begin{equation*}
\begin{aligned}
\mathbb{E}[G(\bar{\boldsymbol{\alpha}})] =& \mathbb{E}[G(\frac{1}{T-T_0}\sum_{t=T_0}^{T-1} \boldsymbol{\alpha}^{(t)}) ] \leq \frac{1}{T-T_0} \mathbb{E}[\sum_{t=T_0}^{T-1}G(\boldsymbol{\alpha}^{(t)})]\\
\stackrel{\text{Lemma }\ref{conv_lemma},\ref{Q_upp_bound}}{\leq} & \frac{1}{T-T_0} \mathbb{E}\bigg[\sum_{t=T_0}^{T-1}\bigg( \frac{1}{\eta(1-\Theta)s}(D(\boldsymbol{\alpha}^{(t+1)})-D(\boldsymbol{\alpha}^{(t)})) + \frac{Q\rho s}{2\lambda}\bigg) \bigg]\\
= & \frac{1}{\eta(1-\Theta)s} \frac{1}{T-T_0} \mathbb{E} [D(\boldsymbol{\alpha}^{(T)})-D(\boldsymbol{\alpha}^{(T_0)})] + \frac{Q\rho s}{2\lambda}\\
\leq& \frac{1}{\eta(1-\Theta)s} \frac{1}{T-T_0} \mathbb{E} [D(\boldsymbol{\alpha}^{*})-D(\boldsymbol{\alpha}^{(T_0)})] + \frac{Q\rho s}{2\lambda}
\end{aligned}
\end{equation*}

Let's assume $T \geq T_0 + \ceil[\Big]{\frac{1}{\eta(1-\Theta)}}$ and $T_0 \geq t_0$, we get that
\begin{equation*}
\begin{aligned}
\mathbb{E}[G(\bar{\boldsymbol{\alpha}})] \leq& \frac{1}{\eta(1-\Theta)s} \frac{1}{T-T_0} \bigg(\frac{Q\rho}{\lambda(1+ \frac{\eta}{2}(1-\Theta)(T_0-t_0))}\bigg) + \frac{Q\rho s}{2\lambda}\\
= & \frac{Q\rho}{\lambda}\bigg(\frac{1}{\eta(1-\Theta)s(T-T_0)(1+ \frac{\eta}{2}(1-\Theta)(T_0-t_0)) }+ \frac{s}{2} \bigg)
\end{aligned}
\end{equation*}
By setting $s = \frac{1}{(T-T_0)\eta(1-\Theta)} \in [0,1]$,
we will have
\begin{equation*}
\begin{aligned}
\mathbb{E}[G(\bar{\boldsymbol{\alpha}})] \leq& \frac{Q\rho}{\lambda}\bigg(\frac{1}{1+ \frac{\eta}{2}(1-\Theta)(T_0-t_0) }+ \frac{1}{2(T-T_0)\eta(1-\Theta)} \bigg)
\end{aligned}
\end{equation*}
To have the right hand side smaller than $\epsilon_G$, we just need the following two inequalities to hold:
\begin{eqnarray}
\label{eq:gap_ineq1}
  \frac{Q\rho}{\lambda}\bigg( \frac{1}{1+ \frac{\eta}{2}(1-\Theta)(T_0-t_0) } \bigg)&\leq& \frac{\epsilon_G}{2} \\
  \label{gap_ineq2}
  \frac{Q\rho}{\lambda}\bigg( \frac{1}{2(T-T_0)\eta(1-\Theta)} \bigg)&\leq& \frac{\epsilon_G}{2}
\end{eqnarray}
If we have
\begin{eqnarray*}
   T_0 &\geq& t_0 + \bigg(\frac{2}{\eta(1-\Theta)}(\frac{2Q\rho}{\lambda \epsilon_G} -1) \bigg), \\
   T &\geq& T_0 + \frac{Q \rho}{\lambda\epsilon_G\eta(1-\Theta) },
\end{eqnarray*}
\eqref{eq:gap_ineq1} and \eqref{gap_ineq2} will hold and we will have $\mathbb{E}[G(\bar{\boldsymbol{\alpha}})] \leq \epsilon_G$. \
From Lemma \ref{technical_D_0} we know that $D(\boldsymbol{\alpha}^{(*)})-D(\boldsymbol{\alpha}^{(0)}) \leq m$, therefore, we conclude the proof.
\end{proof}

\section{Additional Experimental Results on Real-world Datasets}\label{sec:experiment_appendix}
We conduct more experiments on the three real-world dataset by setting a different value to $\lambda$ in \eqref{formulation} ($\lambda=10^{-5}$), and report the results in Figures~\ref{fig:pred_error_realworld_add}-\ref{fig:commu_real_world_add}. The results are similar to those with $\lambda=10^{-6}$.

\begin{figure}[h!]
	\begin{center}
		\subfigure[Prediction Error v.s. Communication] {\label{fig:pred_error_realworld_add}
			\includegraphics[trim = 10cm 0cm 12cm 0cm, clip, width=0.315\columnwidth]{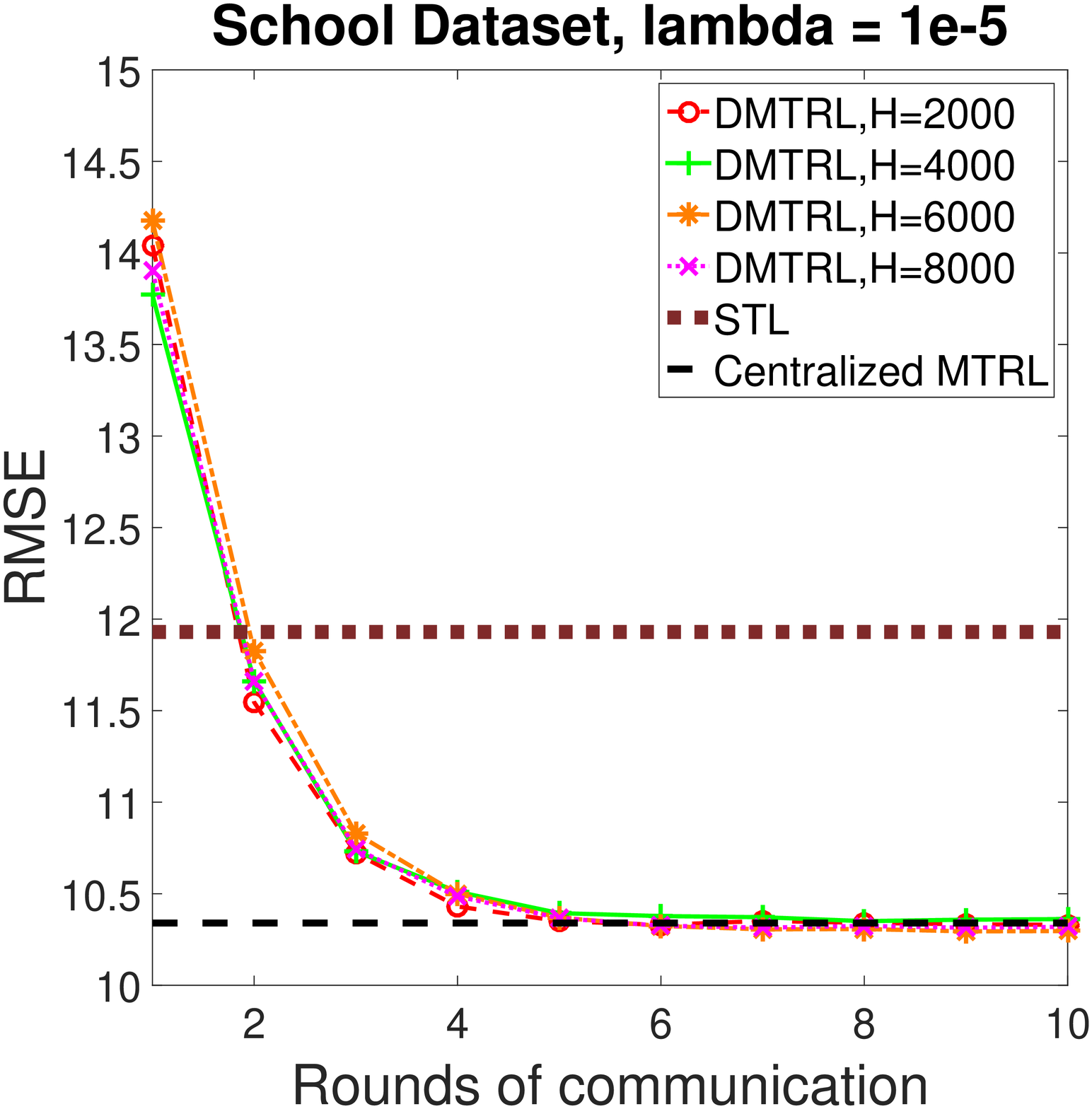}
			\includegraphics[trim = 10cm 0cm 12cm 0cm, clip, width=0.315\columnwidth]{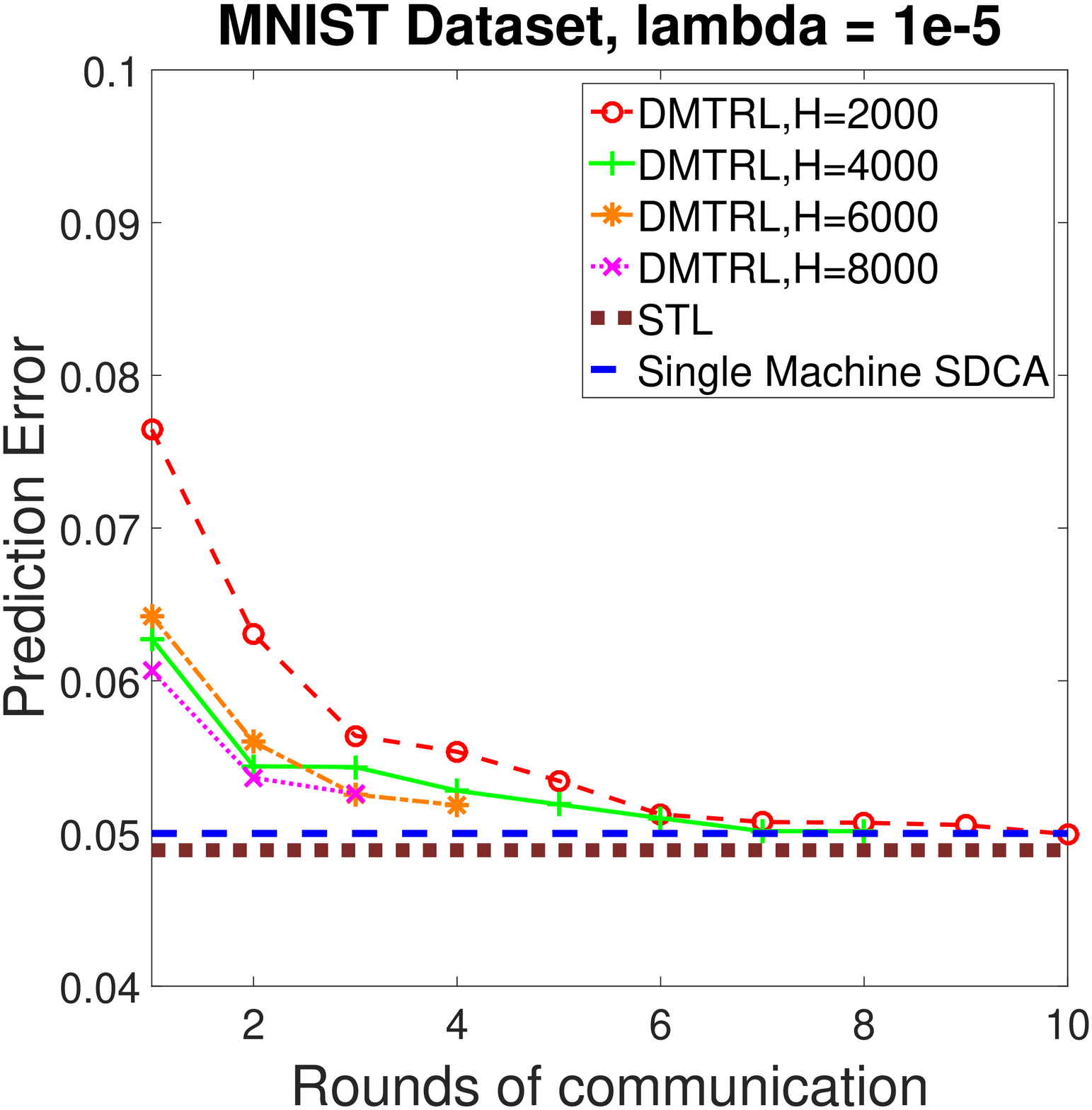}
			\includegraphics[trim = 10cm 0cm 12cm 0cm, clip, width=0.315\columnwidth]{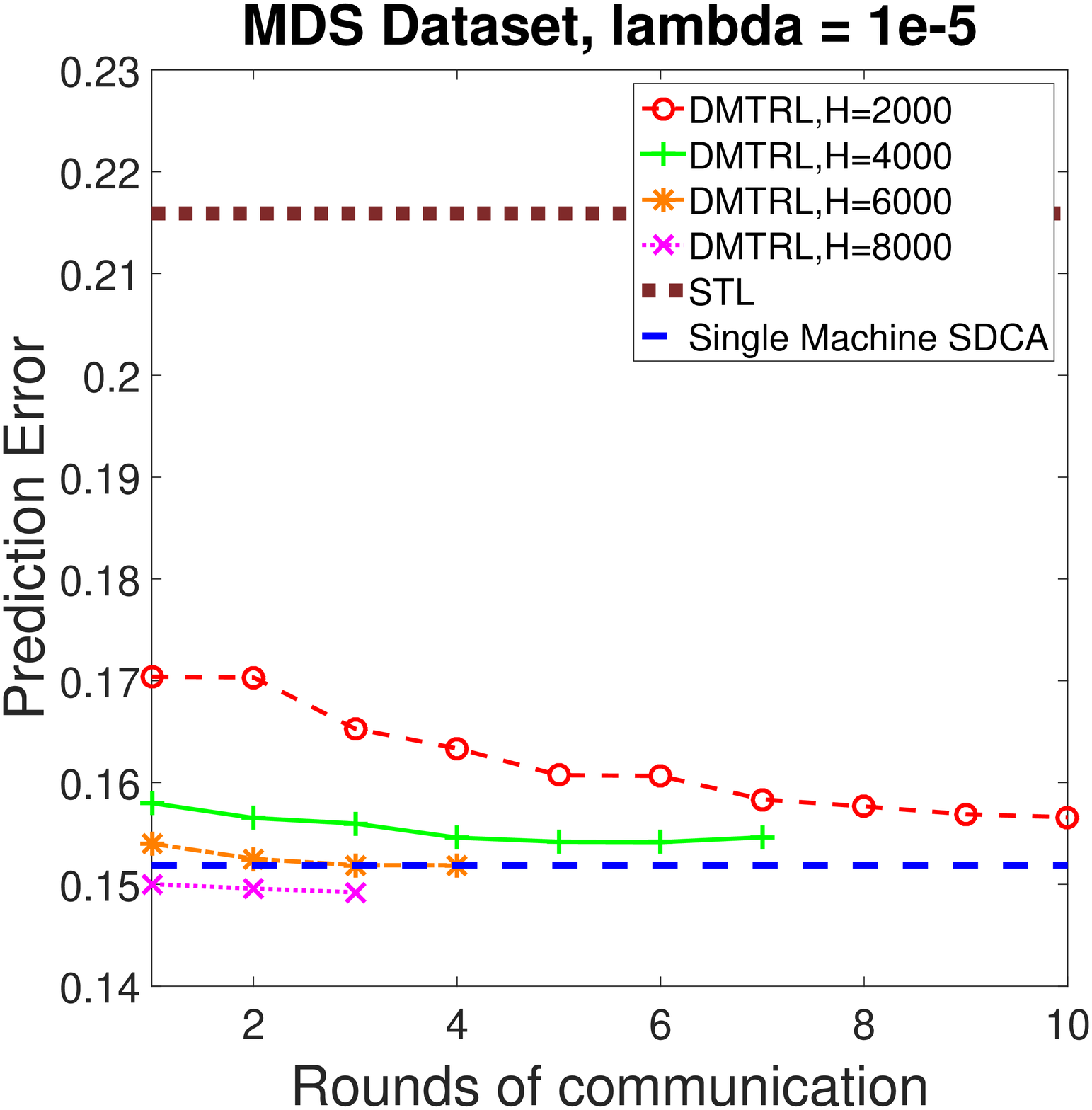}}
		\\
		\subfigure[Duality Gap v.s. Time] {\label{fig:dualGap_realworld_add}
			\includegraphics[trim = 10cm 0cm 12cm 0cm, clip, width=0.315\columnwidth]{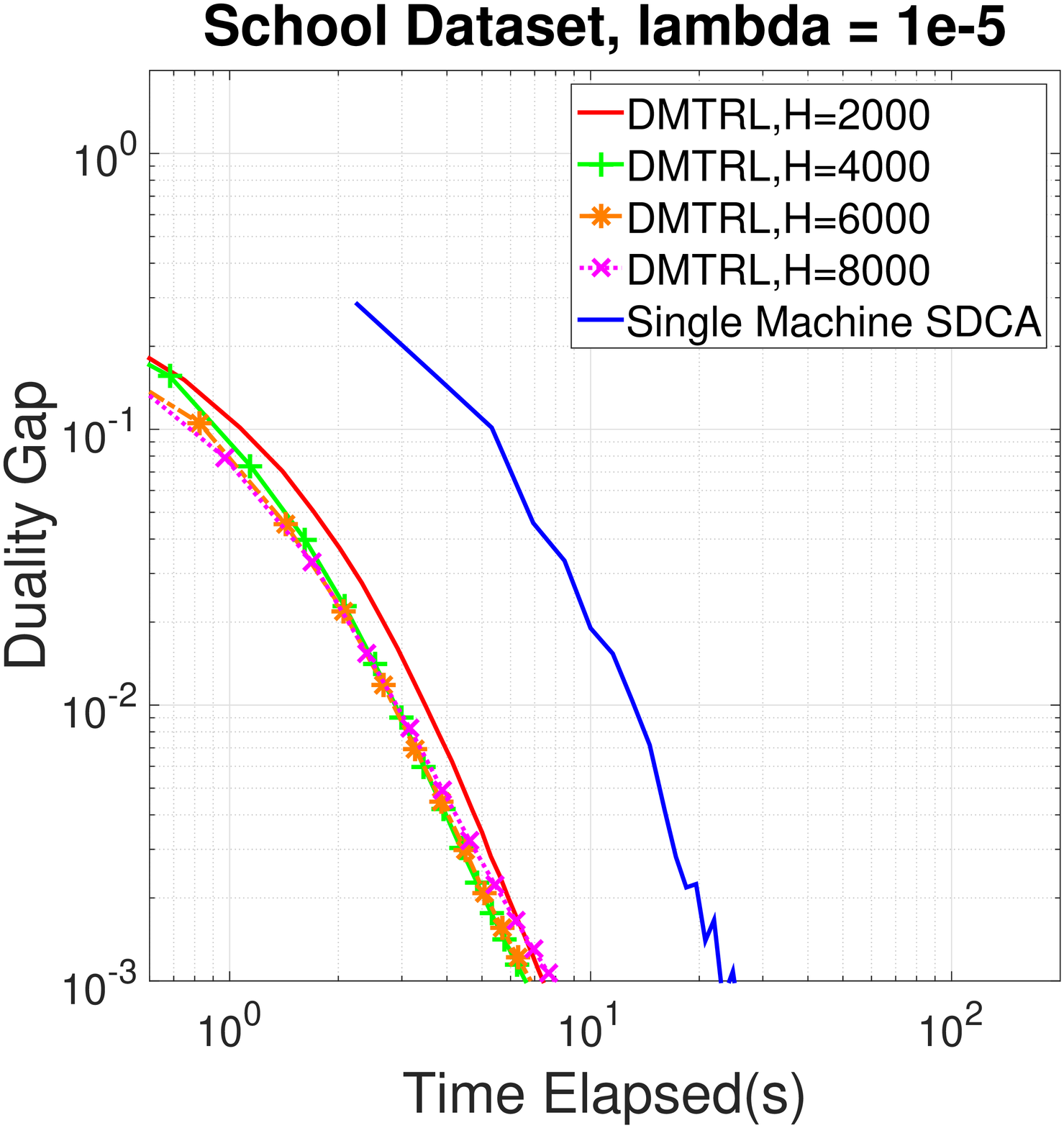}
			\includegraphics[trim = 10cm 0cm 12cm 0cm, clip, width=0.315\columnwidth]{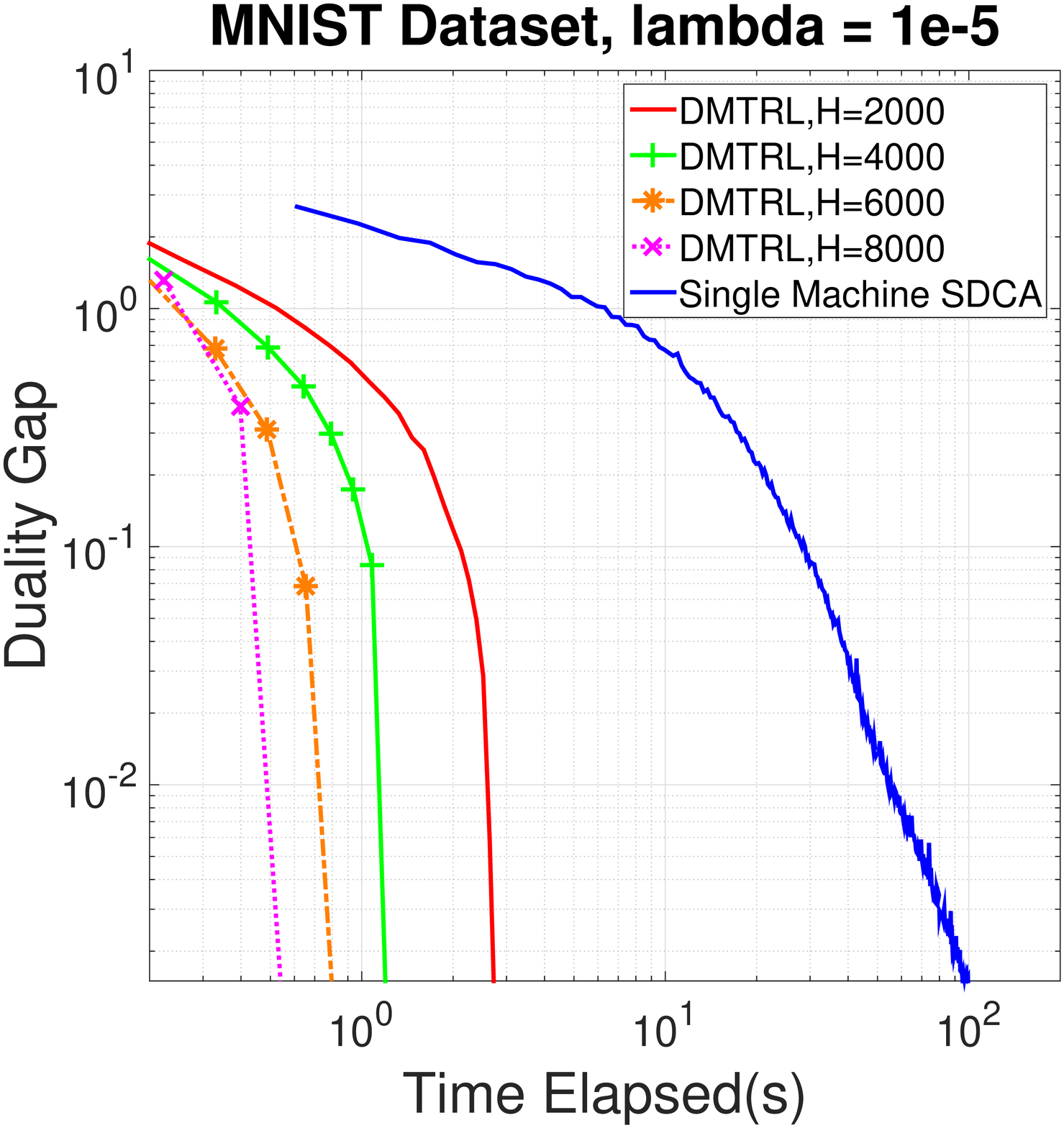}
			\includegraphics[trim = 10cm 0cm 12cm 0cm, clip, width=0.315\columnwidth]{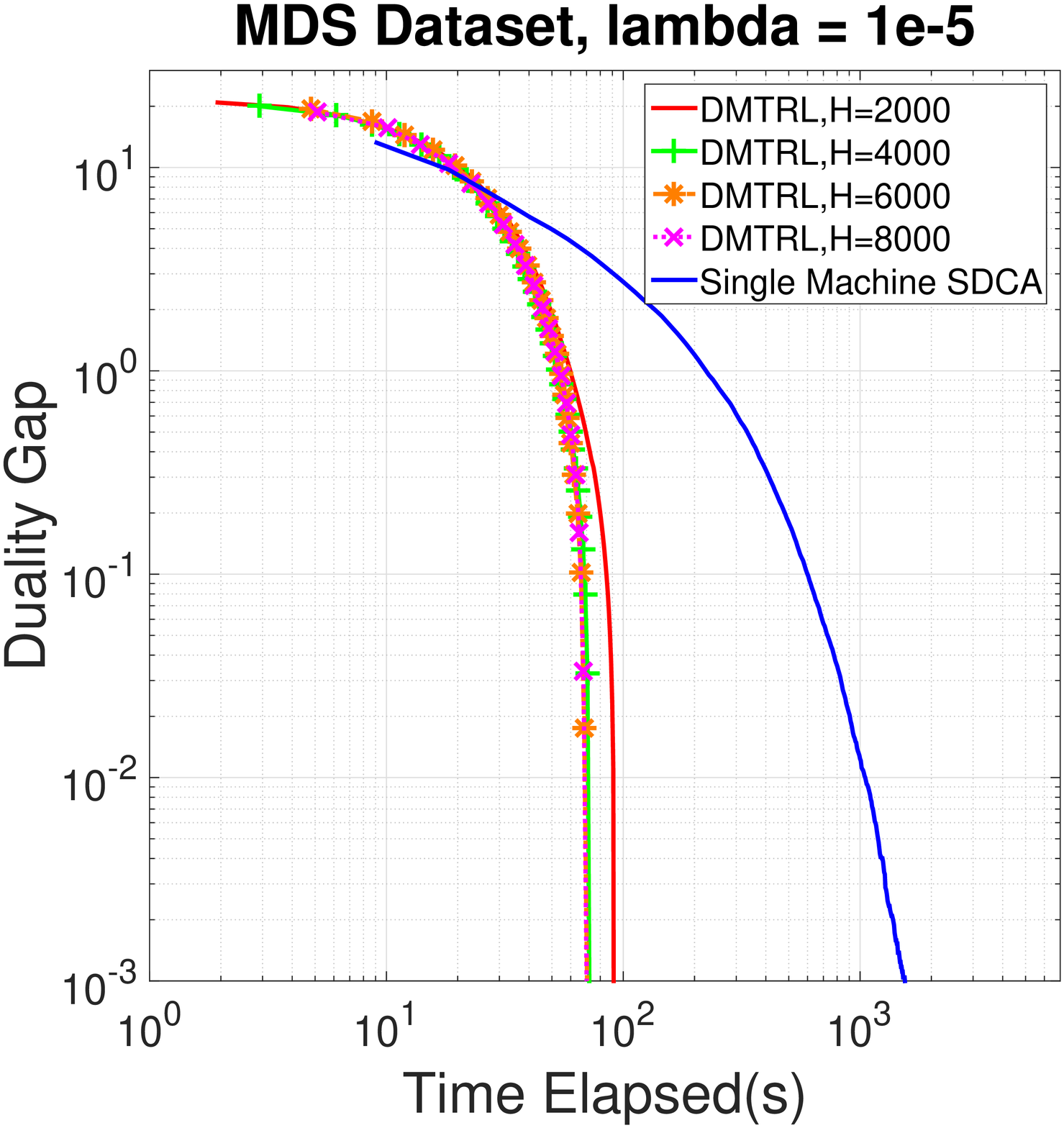}}
		\\
		\subfigure[Duality Gap v.s. Communication] {\label{fig:commu_real_world_add}\
			\includegraphics[trim = 10cm 0cm 12cm 0cm, clip, width=0.315\columnwidth]{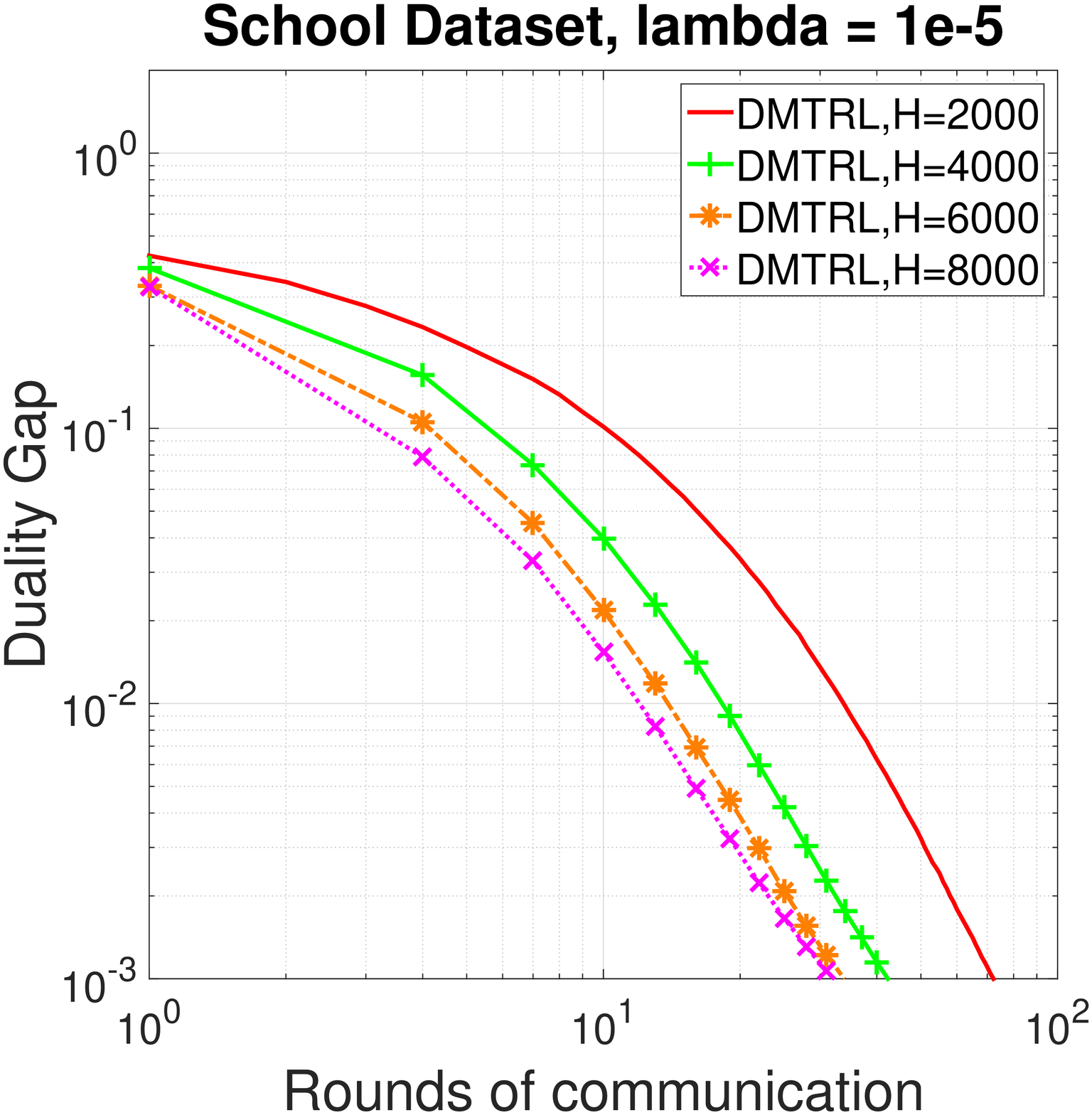}
			\includegraphics[trim = 10cm 0cm 12cm 0cm, clip, width=0.315\columnwidth]{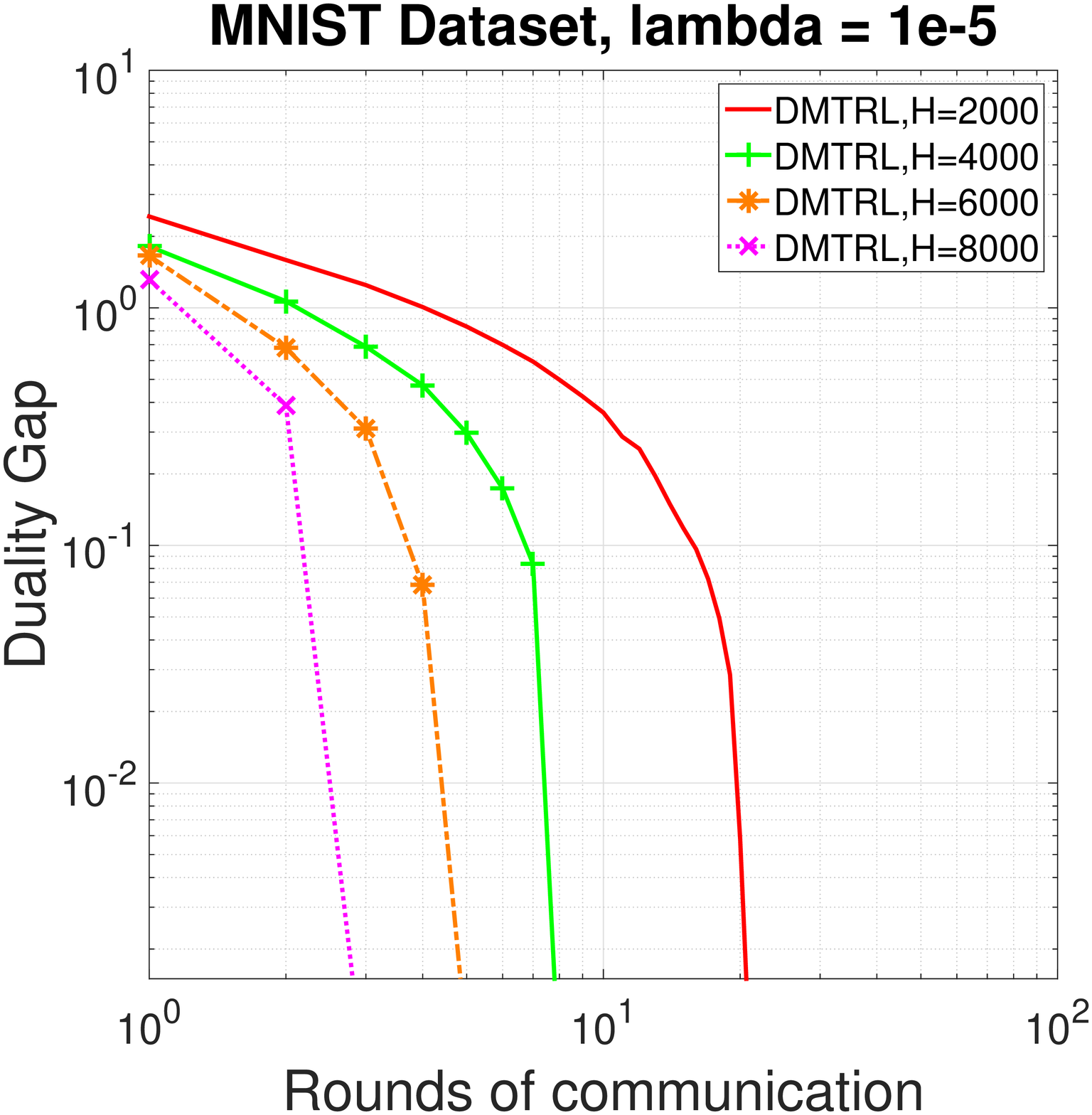}
			\includegraphics[trim = 10cm 0cm 12cm 0cm, clip, width=0.315\columnwidth]{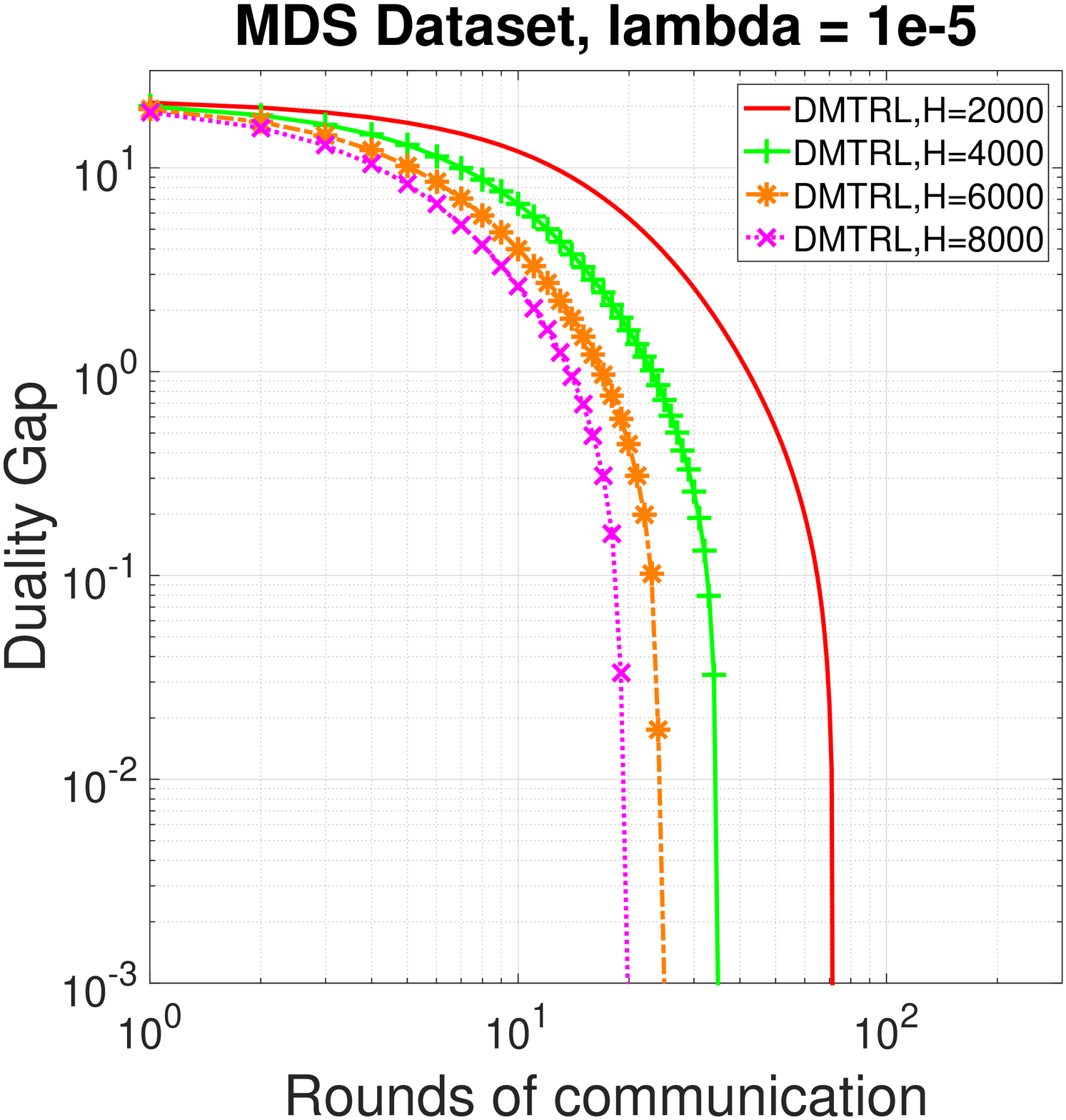}}
		\caption{Experimental results on real-world datasets}\label{fig:real_world_add}
	\end{center}
\end{figure}

\end{document}